\documentclass[10pt,journal,compsoc]{IEEEtran}

\usepackage[utf8]{inputenc}
\usepackage[T1]{fontenc}
\usepackage{hyperref}
\usepackage{url}
\usepackage{booktabs}
\usepackage{amsfonts}
\usepackage{amsmath}
\usepackage{amssymb}
\usepackage{amsthm}
\usepackage{nicefrac}
\usepackage{microtype}
\usepackage{graphicx}
\usepackage{xcolor}
\usepackage{subcaption}
\usepackage{multirow}
\usepackage{mathtools}
\usepackage{cite}
\usepackage{enumitem}
\usepackage{float}
\usepackage{rotating}

\newcommand{\X}{\mathcal{X}}
\newcommand{\Y}{\mathcal{Y}}

\newcommand{\phat}{\hat{p}}
\newcommand{\pihat}{\hat{\pi}}
\newcommand{\tpihat}{\tilde{\pi}}
\newcommand{\bE}{\mathbb{E}}
\newcommand{\bP}{\mathbb{P}}
\newcommand{\bR}{\mathbb{R}}
\newcommand{\ECE}{\mathrm{ECE}}
\newcommand{\KL}{\mathrm{KL}}
\newcommand{\ECEvoted}{\ECE_{\text{voted}}}
\newcommand{\ECEtrue}{\ECE_{\text{true}}}

\DeclareMathOperator*{\argmax}{arg\,max}

\newtheorem{definition}{Definition}
\newtheorem{proposition}{Proposition}

\newtheorem{remark}{Remark}

\definecolor{darkred}{RGB}{180,0,0}
\definecolor{darkgreen}{RGB}{0,120,0}
\definecolor{darkblue}{RGB}{0,0,160}

\title{Confidence Calibration under Ambiguous Ground Truth}

\author{Linwei~Tao,
        Haoyang~Luo,
        Minjing~Dong,
        and~Chang~Xu,~\IEEEmembership{Senior~Member,~IEEE}%
  \IEEEcompsocitemizethanks{
    \IEEEcompsocthanksitem Manuscript received XX XXX, XXXX.
    \protect\IEEEcompsocthanksitem L.~Tao and C.~Xu are with the School of Computer Science,
    University of Sydney, Sydney, NSW 2006, Australia.
    E-mail: linwei.tao@sydney.edu.au; c.xu@sydney.edu.au.
    \protect\IEEEcompsocthanksitem H.~Luo and M.~Dong are with the Department of Computer Science,
    City University of Hong Kong, Hong Kong SAR, China.
    E-mail: luohaoyang.lalutte@gmail.com; minjdong@cityu.edu.hk.
    \protect\IEEEcompsocthanksitem Corresponding author: Chang Xu (e-mail: c.xu@sydney.edu.au).
  }%
}

\begin{document}

\maketitle

% ─────────────────────────────────────────────────────────────────────────────
\begin{abstract}
% ─────────────────────────────────────────────────────────────────────────────
Confidence calibration assumes a unique ground-truth label per input, yet this assumption fails wherever annotators genuinely disagree.
Post-hoc calibrators fitted on majority-voted labels, the standard single-label targets used in practice, can appear well-calibrated under conventional evaluation yet remain substantially miscalibrated against the underlying annotator distribution.
We show that this failure is structural: under simplifying assumptions, Temperature Scaling is biased toward temperatures that underestimate annotator uncertainty, with true-label miscalibration increasing monotonically with annotation entropy.
To address this, we develop a family of ambiguity-aware post-hoc calibrators that optimise proper scoring rules against the full label distribution and require no model retraining.
Our methods span progressively weaker annotation requirements: Dirichlet-Soft leverages the full annotator distribution and achieves the best overall calibration quality across settings; Monte Carlo Temperature Scaling with a single annotation per example (MCTS $S{=}1$) matches full-distribution calibration across all benchmarks, demonstrating that pre-aggregated label distributions are unnecessary; and Label-Smooth Temperature Scaling (LS-TS) operates with voted labels alone by constructing data-driven pseudo-soft targets from the model's own confidence.
Experiments on four benchmarks with real multi-annotator distributions (CIFAR-10H, ChaosNLI) and clinically-informed synthetic annotations (ISIC~2019, DermaMNIST) show that Dirichlet-Soft reduces true-label ECE by 55--87\% relative to Temperature Scaling, while LS-TS reduces ECE by 9--77\% without any annotator data.
\end{abstract}


% ─────────────────────────────────────────────────────────────────────────────
\section{Introduction}
\label{sec:intro}
% ─────────────────────────────────────────────────────────────────────────────

\IEEEPARstart{D}{eep} neural networks are increasingly deployed as decision-support tools in high-stakes domains.  In medical imaging, classifiers assist dermatologists in diagnosing skin lesions and pathologists in identifying malignant tissue~\cite{haenssle2018man,kompa2021second}; in natural language processing, models flag toxic content, assess clinical notes, and support legal document review; in autonomous driving, perception systems make real-time safety decisions under uncertainty~\cite{feng2022review}.  In each of these settings the model outputs not merely a predicted class but an implicit \emph{reliability claim}: the confidence score attached to a prediction is understood by downstream systems and human operators as a direct probability that the prediction is correct.  A radiologist who sees 90\% confidence on a malignant-lesion classification may reduce the level of follow-up scrutiny; a content moderator acting on a 99\% toxicity score may take irreversible action; an autonomous vehicle assigns braking decisions based on its detection confidence.  The faithfulness of stated confidence to empirical correctness, i.e.\ calibration, is therefore as important to deployment safety as accuracy itself~\cite{ovadia2019trust,abdar2021review}.

This property is formalised as follows.  A well-calibrated classifier $f:\X\to\Delta^{K}$ produces probability vectors $\phat(x)=f(x)$ such that
\begin{equation}
  \bP\!\bigl(\hat{Y}=Y \mid \phat_{\hat{Y}}(X)=p\bigr) = p, \quad \forall\, p\in[0,1],
  \label{eq:hard-cal}
\end{equation}
where $\hat{Y}=\argmax_k\phat_k(X)$ and $\Delta^K$ denotes the $K$-dimensional probability simplex.  Guo et al.~\cite{guo2017calibration} showed that modern high-accuracy networks are systematically overconfident, and proposed Temperature Scaling (TS), which divides logits by a scalar $T>1$ before softmax, as an effective post-hoc remedy.  Subsequent work extended post-hoc calibration to Platt scaling~\cite{platt1999probabilistic}, isotonic regression~\cite{zadrozny2002transforming}, and Dirichlet calibration~\cite{kull2019beyond}, and studied its behaviour across architectures~\cite{minderer2021revisiting} and tasks~\cite{bai2021don}.  These methods share a common design principle: they fit a correction to the model's logits using a held-out calibration set annotated with \emph{one-hot labels}.

The calibration condition~\eqref{eq:hard-cal} implicitly assumes that every input $x$ has a unique ground-truth label $Y$.  This assumption is violated whenever the task is inherently ambiguous or subjective.  A skin lesion may be legitimately classified as malignant or benign by different dermatologists, depending on subtle visual features that experts genuinely disagree about~\cite{haenssle2018man}; a natural-language premise may or may not entail a hypothesis depending on pragmatic interpretation~\cite{nie2020what}; a low-resolution object may be categorically ambiguous even to careful human observers~\cite{peterson2019human}.  Such disagreement is not annotation \emph{noise} to be corrected but irreducible \emph{aleatoric} uncertainty inherent to the task.  The correct calibration target is therefore not a single label $Y$ but a \emph{distribution over labels} $\pi(\cdot\mid x)\in\Delta^K$, which we call the \emph{annotator distribution}, representing the probability that a randomly selected expert would assign each label to input $x$.  Figure~\ref{fig:intro-ambiguity} illustrates three concrete instances of this phenomenon across image classification, natural language inference, and medical imaging.

\begin{figure*}[t]
  \centering
  \includegraphics[width=\linewidth]{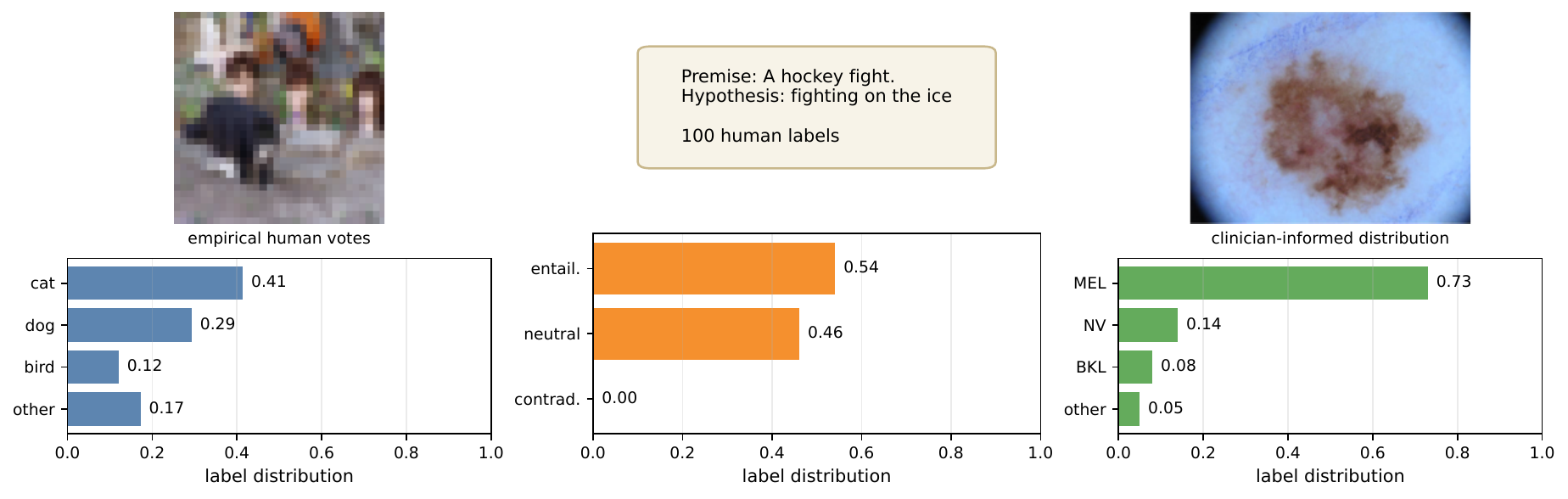}
  \caption{
    \textbf{Left:} a CIFAR-10H image with dispersed human votes (cat/dog/bird), illustrating perceptual ambiguity in low-resolution vision.
    \textbf{Middle:} a ChaosNLI premise--hypothesis pair with split entailment/neutral judgments, illustrating semantic ambiguity.
    \textbf{Right:} an ISIC~2019 melanoma image paired with the clinician-informed label distribution used in our medical experiments, showing clinically plausible MEL/NV confusion.
    CIFAR-10H and ChaosNLI distributions are empirical human label distributions; the ISIC panel uses the dermatologist confusion model described in Appendix~\ref{app:isic-confusion}.
  }
  \label{fig:intro-ambiguity}
\end{figure*}

Standard practice aggregates multiple annotations by majority vote to obtain a single \emph{voted label} $y^*$, the one-hot label used in virtually all existing calibration work, and calibrates against it.  This paper demonstrates that this is not merely a rough approximation but a \emph{systematic failure}: calibrating to the voted label optimises the wrong objective, and the resulting miscalibration worsens as genuine annotator disagreement increases.  A particularly revealing observation is that the most expressive voted-label calibrator (Dirichlet calibration with a full $K\!\times\!K$ affine transform) consistently \emph{underperforms} simple Temperature Scaling in true-label ECE; additional flexibility with the wrong target exacerbates rather than alleviates the problem.  In a controlled experiment (Section~\ref{sec:toy}), every standard voted-label calibrator fails to close the gap; the root cause is the calibration \emph{target}, not the method's capacity.  This observation parallels the finding of Stutz et al.~\cite{stutz2023conformal} for conformal prediction, where sets calibrated on voted labels systematically undercover the true annotator distribution.

\begin{enumerate}[leftmargin=*]
  \item \textbf{Problem and theory} (Sections~\ref{sec:problem}--\ref{sec:theory}): we formalise true-label calibration under ambiguous ground truth and prove that Temperature Scaling is biased toward unduly low temperatures, with the miscalibration gap growing monotonically with annotation entropy.
  \item \textbf{Methods} (Section~\ref{sec:methods}): we propose a family of ambiguity-aware post-hoc calibrators that use the annotator distribution as the calibration target and require no model retraining, including: \textbf{Dirichlet-Soft} (full distribution available); \textbf{MCTS} (individual annotations available, where a single annotation per example, $S{=}1$, already matches full-distribution calibration); and \textbf{LS-TS}, an annotation-free variant that uses only voted labels yet substantially closes the calibration gap.
  \item \textbf{Evidence} (Sections~\ref{sec:toy} and~\ref{sec:experiments}): experiments across four benchmarks spanning vision (CIFAR-10H, ISIC~2019, DermaMNIST) and language (ChaosNLI) show that Dirichlet-Soft reduces $\ECEtrue$ by 55--87\% and LS-TS by 9--77\% relative to TS, while standard calibrators can actively worsen true-label ECE.
\end{enumerate}

% ─────────────────────────────────────────────────────────────────────────────
\section{Related Work}
\label{sec:related}
% ─────────────────────────────────────────────────────────────────────────────

\subsection{Confidence Calibration}
Guo et al.~\cite{guo2017calibration} established that modern high-accuracy networks are systematically overconfident and that Temperature Scaling (TS) is a simple and effective post-hoc remedy.  The multiclass calibration literature has since proposed isotonic regression~\cite{zadrozny2002transforming}, Platt scaling~\cite{platt1999probabilistic}, Dirichlet calibration~\cite{kull2019beyond}, and adaptive per-instance temperature variants~\cite{joy2023sample}.  Minderer et al.~\cite{minderer2021revisiting} showed that Vision Transformers are inherently better calibrated than CNNs, and Bai et al.~\cite{bai2021don} revisit when TS is and is not sufficient.  Kumar et al.~\cite{kumar2019verified} propose verified uncertainty calibration through non-parametric testing.  All of these methods share a common assumption: the calibration set contains one-hot labels, and calibration means aligning confidence with correctness relative to those labels.  We show that this assumption, not the method's architectural capacity, is the limiting factor when labels are genuinely distributional.

\subsection{Soft Labels, Label Smoothing, and Knowledge Distillation}
Training-time approaches have explored the benefits of soft targets.  Label smoothing~\cite{szegedy2016rethinking,muller2019does} adds a fixed uniform component to one-hot targets and has been observed to improve both accuracy and calibration; however, it applies a dataset-global constant regardless of per-instance annotator disagreement.  Knowledge distillation~\cite{hinton2015distilling} uses teacher model outputs as soft training targets, improving student generalisation and calibration.  Thulasidasan et al.~\cite{thulasidasan2019mixup} show that Mixup training substantially improves calibration by encouraging confident predictions only when training examples are well-separated.  Collins et al.~\cite{collins2022eliciting} collect CIFAR-10S, a companion to CIFAR-10H with individually elicited soft labels from every annotator, and demonstrate that training with per-annotator soft labels improves calibration over aggregated hard labels.  Our work is the post-hoc calibration counterpart: we do not retrain the model, but instead re-target the calibration objective to the annotator distribution using only cached logits.

\subsection{Annotation Disagreement and Multi-Annotator Learning}
Dense annotation datasets~\cite{peterson2019human,nie2020what} and multi-annotator learning methods~\cite{uma2021learning,rodrigues2018deep} have documented the ubiquity of label disagreement across domains.  Aroyo and Welty~\cite{aroyo2015truth} argue that disagreement is not noise but signal, and that discarding it by majority aggregation loses information that is critical for understanding task difficulty.  Gordon et al.~\cite{gordon2022jury} and Plank~\cite{plank2022problem} highlight systematic problems with annotation aggregation, showing that models trained on majority labels inherit structural biases against minority annotator perspectives.  Baan et al.~\cite{baan2022stop} show that ECE is theoretically ill-defined when annotators genuinely disagree, because no unique ``correct'' label exists to serve as the reference; they propose instance-level uncertainty measures.  We complement this critique by retaining $\ECEtrue$ as a practical metric: it quantifies the overconfidence a user experiences when drawing a single label $\tilde{y}\sim\pi(\cdot|x)$ at deployment time, mirroring realistic usage, while supplementing it with Brier score and NLL, which are strictly proper scoring rules that directly penalise divergence from the full annotator distribution $\pihat$.  Khurana et al.~\cite{crowd_calibrator2024} study how annotator disagreement can inform selective abstention in NLP; unlike their approach, our framework targets all predictions and is applicable across vision and language without task-specific thresholding.

\subsection{Proper Scoring Rules and Distributional Evaluation}
Gneiting and Raftery~\cite{gneiting2007strictly} provide the foundational theory of proper scoring rules: a loss is \emph{strictly proper} if and only if it is uniquely minimised when the predicted distribution equals the true distribution.  Brier score and NLL are canonical instances.  This perspective motivates our choice of calibration objective: minimising cross-entropy against $\pihat(x)$ is strictly proper over the predicted distribution $\hat{p}(x)$, while minimising against a one-hot label is only proper with respect to a degenerate Dirac distribution and misidentifies the optimal calibrated output whenever $\pihat$ is non-degenerate.  The proper scoring framework also clarifies why Dirichlet calibration with voted targets ($\mathrm{Dirichlet\text{-}Hard}$) is not merely suboptimal but is consistent for the \emph{wrong} estimand, explaining the empirical finding that it underperforms TS.

\subsection{Conformal Prediction under Ambiguous Ground Truth}
Stutz et al.~\cite{stutz2023conformal} demonstrate the set-prediction analogue of our finding: conformal prediction sets calibrated on voted labels systematically undercover the true annotator distribution, with coverage gaps of up to 10\% on dermatology data.  Their Monte Carlo conformal approach samples from the annotator distribution to compute coverage, closely paralleling our Monte Carlo ECE evaluation ($S=100$ draws per example).  Our work addresses point-estimate calibration rather than set coverage: we produce well-calibrated scalar confidence scores, which conformal prediction does not.  Calibration under covariate shift~\cite{park2020calibrated,wang2020transferable} is orthogonal to our label-ambiguity setting and assumes a unique ground truth throughout.

\subsection{Uncertainty Quantification in High-Stakes Applications}
The deployment literature motivates the practical urgency of this problem.  Kompa et al.~\cite{kompa2021second} argue that medical image classifiers are consistently overconfident, even after calibration, and that this overconfidence degrades clinical decision quality.  Ovadia et al.~\cite{ovadia2019trust} demonstrate that popular uncertainty estimation methods (dropout, deep ensembles, variational inference) exhibit significantly degraded calibration under distribution shift, highlighting that single-point calibration benchmarks may be overly optimistic.  Feng et al.~\cite{feng2022review} survey uncertainty quantification in autonomous driving, where overconfident detections directly contribute to unsafe manoeuvres.  These findings establish the applied context for our work: in each of these domains, label ambiguity is simultaneously a property of the task and an unaddressed source of calibration error.  The present paper provides the first systematic post-hoc calibration framework that directly targets this source.

% ─────────────────────────────────────────────────────────────────────────────
\section{Problem Formulation}
\label{sec:problem}
% ─────────────────────────────────────────────────────────────────────────────

\subsection{Setup and Notation}

Let $\Y=\{1,\ldots,K\}$ be the label space.  A classifier $f:\X\to\Delta^K$ outputs $\phat(x)=\operatorname{softmax}(z(x))$ from pre-softmax logits $z(x)\in\bR^K$.  The \textbf{annotator distribution} $\pi(\cdot\mid x)\in\Delta^K$ denotes the underlying ambiguous label distribution for input $x$ and is estimated from $m$ annotations as $\pihat_k(x)=\frac{1}{m}\sum_j\mathbf{1}[a_j=k]$.  The \textbf{voted label} is $y^*=\argmax_k\pihat_k(x)$.  Standard calibration uses one-hot voted labels.  For example, Temperature Scaling \cite{guo2017calibration} minimises
\begin{equation}
  \mathcal{L}_{\text{TS}}(T) = -\frac{1}{n}\sum_{i=1}^n \log \operatorname{softmax}(z_i/T)_{y_i^*}.
  \label{eq:ts-loss}
\end{equation}

\subsection{True-Label Calibration}
\label{sec:true-cal-def}

Following Stutz et al.~\cite{stutz2023conformal}, we distinguish the voted label $y^*$ from the \textbf{true label} $\tilde{y}\sim\pi(\cdot\mid x)$, drawn from the underlying ambiguous label distribution.

\begin{definition}[True-Label Calibration]
\label{def:true-cal}
$f$ is \textbf{true-label calibrated} if for all $p\in[0,1]$:
$\bP(\tilde{Y}=\hat{c}(X)\mid\phat_{\hat{c}}(X)=p)=p$, where $\hat{c}(x)=\argmax_k\phat_k(x)$ and $\tilde{Y}\sim\pi(\cdot\mid X)$.
\end{definition}

The \textbf{voted-label ECE} $\ECEvoted$ is the standard ECE against $y^*$.  The \textbf{true-label ECE} $\ECEtrue$ averages ECE over 100 draws $\tilde{y}_i\sim\pi(\cdot\mid x_i)$ ($B=15$ equal-width bins).  We also define the \textbf{pointwise true-label calibration error} of a prediction at input $x$ as $|\phat_{\hat{c}}(x) - \pi_{\hat{c}}(x)|$, i.e.\ the absolute gap between the model's top-class confidence and the annotator-distribution probability of that class.  This per-example quantity is used in the theoretical analysis (Section~\ref{sec:theory}).

% ─────────────────────────────────────────────────────────────────────────────
\section{Motivating Example}
\label{sec:toy}
% ─────────────────────────────────────────────────────────────────────────────

\noindent\textbf{Setup.}
We construct a 3-class 2D Gaussian dataset with three clusters:
$x\sim\mathcal{N}((-3.2,1.1), \operatorname{diag}(0.60,0.45))$ for class~0 with $\pi=[1,0,0]$;
$x\sim\mathcal{N}((0,0), \operatorname{diag}(1.15,0.75))$ for the ambiguous cluster with $\pi=[0,0.70,0.30]$;
and $x\sim\mathcal{N}((3.2,-1.1), \operatorname{diag}(0.60,0.45))$ for class~2 with $\pi=[0,0,1]$.
The voted label for the middle cluster is always class~1 because $0.70>0.30$.
A 2-layer MLP (64 hidden units per layer) is trained on voted labels, and all calibrators are fitted on a held-out calibration split using the same voted labels.

\begin{figure*}[t]
  \centering
  \includegraphics[width=\linewidth]{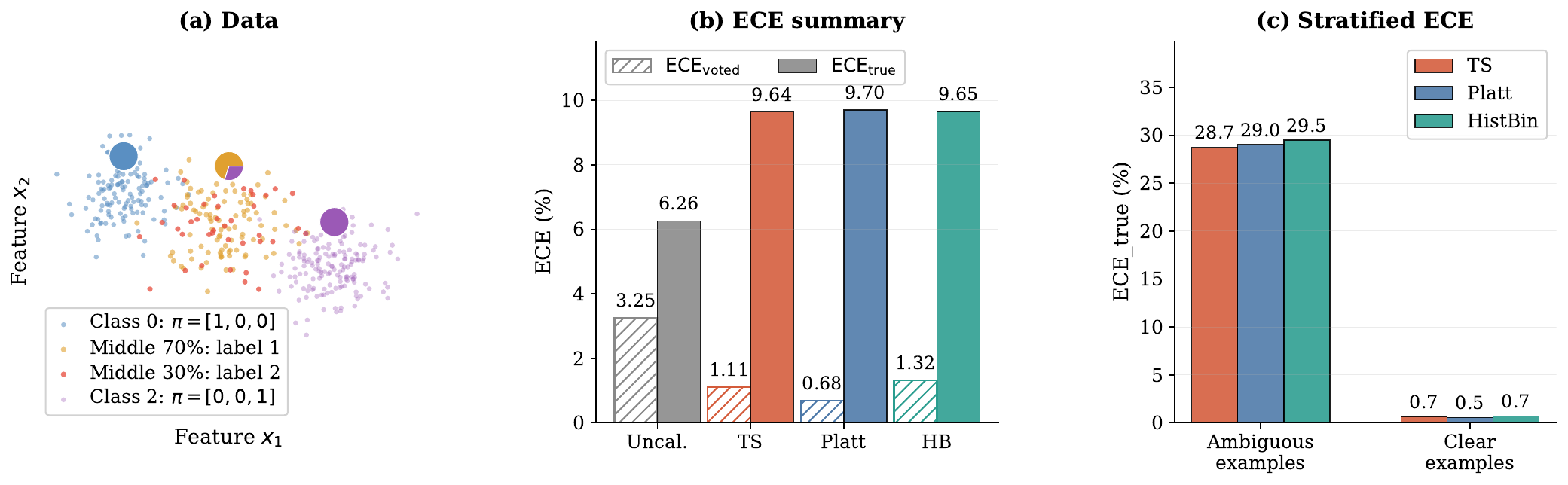}
  \caption{
    \textbf{Motivating example: all standard calibration methods fail.}
    \textbf{(a)} Toy dataset generation: only the middle Gaussian cluster is ambiguous, with $\pi=[0,0.70,0.30]$; orange points are drawn as label~1 (70\%) and red points as label~2 (30\%), while all receive the same voted label~1.
    \textbf{(b)} Summary of the toy results: all three voted-label calibrators (TS, Platt, Histogram Binning) lower $\ECEvoted$ but \emph{increase} $\ECEtrue$, so voted-label evaluation masks the failure.
    \textbf{(c)} Stratified $\ECEtrue$ for TS, Platt, and Histogram Binning: ambiguous examples are those from the middle Gaussian cluster (where annotators disagree); clear examples are those from the two unambiguous clusters (class 0 and class 2).  The residual error is concentrated in ambiguous examples for all three methods.
  }
  \label{fig:toy}
\end{figure*}

\noindent\textbf{Results.}
Figure~\ref{fig:toy} shows that this mismatch is a structural consequence of the voted-label target.  TS reduces $\ECEvoted$ from $3.25\%$ to $1.11\%$, but \emph{increases} $\ECEtrue$ from $6.26\%$ to $9.64\%$.  Platt scaling behaves similarly ($\ECEvoted=0.68\%$, $\ECEtrue=9.70\%$), and Histogram Binning likewise ($\ECEvoted=1.32\%$, $\ECEtrue=9.65\%$).  Figure~\ref{fig:toy}\textbf{(c)} further shows that this residual error is concentrated in the ambiguous cluster for all three methods.  TS sets $T=0.62<1$, boosting confidence toward the voted target and thereby pushing the model \emph{further} from the true ambiguous label distribution.  The failure is therefore not specific to TS: standard post-hoc methods are useful for $\ECEvoted$, but the voted-label target makes them ineffective for $\ECEtrue$.

% ─────────────────────────────────────────────────────────────────────────────
\section{Theory: Why Standard Calibration Fails}
\label{sec:theory}
% ─────────────────────────────────────────────────────────────────────────────

We now formalise, under simplifying assumptions, two complementary aspects of the failure of voted-label calibration under ambiguity: the \emph{direction} of the TS bias (Proposition~\ref{prop:ts-bias}) and its \emph{scaling} with label ambiguity (Proposition~\ref{prop:entropy-gap}).  Both propositions characterise \emph{pointwise} miscalibration on individual examples; they describe when and why a voted-label calibrator misfits individual ambiguous inputs, rather than providing a theorem about population-level ECE.  Figure~\ref{fig:entropy-validation} provides empirical support.

\begin{proposition}[Direction of TS Bias]
\label{prop:ts-bias}
Consider a calibration set containing an ambiguous cluster in which the voted label $y^*$ equals the majority class for every example, but the true annotator probability of $y^*$ is $\pi_{y^*}(x)=q < 1$.  Under the assumption that the model is already reasonably accurate on this cluster ($\operatorname{softmax}(z_i)_{y^*}>q$ for all $i$ in the cluster), the voted-label loss exerts a downward pull on $T$ for these examples, whereas a soft-label loss with target $q$ exerts an upward pull.  Consequently, $T^*_{\text{TS}} < T^*_{\text{soft}}$, where $T^*_{\text{soft}}$ denotes the optimal temperature when minimising the cross-entropy against the annotator distribution $\pihat$ (Eq.~\ref{eq:slts-loss}): TS selects a lower temperature than ambiguity-aware calibration requires.  In a setting where the model is already overconfident ($\operatorname{softmax}(z_i)_{y^*} \gg q$), this manifests as $T^*_{\text{TS}}<1$; more generally, both optima may exceed 1, but the gap $T^*_{\text{soft}} - T^*_{\text{TS}} > 0$ persists, consistent with our experimental findings (e.g., $T_{\text{TS}}=2.03$ vs.\ $T_{\text{soft}}=3.18$ on CIFAR-10H ResNet-50).
\end{proposition}

\begin{proof}[Proof sketch]
On the ambiguous cluster, all voted labels are $y^*$, so the voted-label loss~\eqref{eq:ts-loss} is minimised by pushing $\operatorname{softmax}(z_i/T)_{y^*}\to 1$, i.e.\ the loss decreases as $T$ decreases.  Hence the ambiguous examples exert a downward gradient on $T$.  The cross-entropy against the annotator distribution with target $q<1$ is instead minimised at $\operatorname{softmax}(z_i/T^*)_{y^*}\approx q$, requiring a higher $T^*$ whenever $\operatorname{softmax}(z_i)_{y^*}>q$.  The global TS optimum balances this downward pull against the upward pull from unambiguous examples, but remains lower than the ambiguity-aware optimum.  Full proof in Appendix~\ref{app:proof-ts}.
\end{proof}

\begin{proposition}[True-Label Miscalibration and Annotation Entropy]
\label{prop:entropy-gap}
Let $H(x) = -\sum_k \pi_k(x)\log\pi_k(x)$ denote the annotation entropy of $x$.  Assume \emph{(i)} the model's predicted confidence $\hat{p}_{\hat{c}}(x)$ is approximately independent of $H(x)$ (similar accuracy across ambiguity levels), and \emph{(ii)} $\pi_{y^*}(x)$ is a decreasing function of $H(x)$ (greater entropy implies lower majority-class probability).  Under assumptions \emph{(i)} and \emph{(ii)}, the expected per-example true-label calibration error of voted-label-calibrated predictions is non-decreasing in $H(x)$.
\end{proposition}

\begin{proof}[Proof sketch]
For unambiguous examples ($H(x)=0$), $\pi(\cdot\mid x)$ is one-hot, so true-label and voted-label evaluation coincide and the per-example error is $|{\hat{p}_{\hat{c}}(x) - 1}|$.  As $H(x)$ increases, assumption~(ii) gives $\pi_{y^*}(x) < 1$, introducing a gap between the voted-label target and the true annotator probability.  Because TS is optimised against one-hot voted labels, it targets $\hat{p}_{y^*}(x)\approx 1$ for all examples regardless of their annotation entropy; hence $\hat{p}_{\hat{c}}(x) > \pi_{y^*}(x)$ for ambiguous examples.  The true-label calibration error is then $\hat{p}_{\hat{c}}(x) - \pi_{y^*}(x)$, which by assumption~(i) (fixed $\hat{p}_{\hat{c}}$) increases as $\pi_{y^*}(x)$ falls with $H(x)$.
\end{proof}

\begin{figure}[t]
  \centering
  \includegraphics[width=\linewidth]{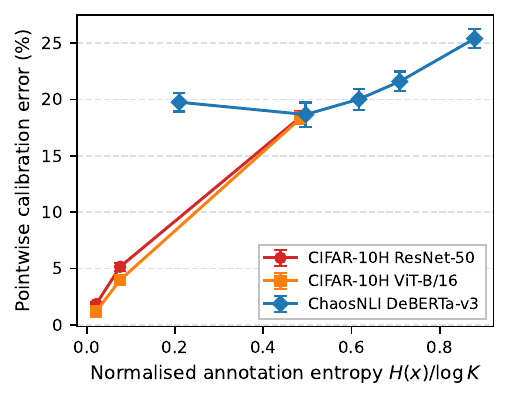}
  \caption{\textbf{Empirical validation of Proposition~\ref{prop:entropy-gap}.}  Test examples are grouped into equal-frequency bins by normalised annotation entropy $H(x)/\log K$.  Each point shows the mean pointwise true-label calibration error $|\hat{p}_{\hat{c}}(x)-\pi_{\hat{c}}(x)|$ for Temperature Scaling (error bars: $\pm 1$ s.e.).  Across all three dataset--architecture combinations, TS error increases monotonically with annotation entropy, as predicted by Proposition~\ref{prop:entropy-gap}.  ChaosNLI operates in a higher-entropy regime and exhibits uniformly higher error ($19$--$25\%$).}
  \label{fig:entropy-validation}
\end{figure}

% ─────────────────────────────────────────────────────────────────────────────
\section{Methods}
\label{sec:methods}
% ─────────────────────────────────────────────────────────────────────────────

All methods in this paper are \emph{post-hoc}: they operate on cached logits from a fixed pre-trained model and require only a calibration set equipped with annotator distributions or individual annotations.  Training-time approaches such as soft-label training \cite{szegedy2016rethinking,muller2019does}, mixup \cite{thulasidasan2019mixup}, and multi-annotator learning \cite{uma2021learning,rodrigues2018deep} are orthogonal and not compared here.

\subsection{Ambiguity-Aware Calibration Methods}

\noindent\textit{Soft-label calibration objective.}
The natural replacement for the voted-label loss~\eqref{eq:ts-loss} is the cross-entropy against the full annotator distribution:
\begin{align}
  \mathcal{L}_{\text{soft}}(T)
  &= -\frac{1}{n}\sum_{i=1}^n \sum_{k=1}^K
    \pihat_k(x_i)\,\log \operatorname{softmax}(z_i/T)_k \nonumber\\
  &= \frac{1}{n}\sum_{i=1}^n \KL\!\bigl(\pihat(x_i)\,\|\,\operatorname{softmax}(z_i/T)\bigr) + \text{const}.
  \label{eq:slts-loss}
\end{align}
This retains the same single-parameter family as TS but replaces the one-hot voted label with a target that correctly reflects the annotator distribution.  When $\pihat$ is available in closed form, minimising~\eqref{eq:slts-loss} directly yields the \textbf{Soft-Label Temperature Scaling (SLTS)} baseline reported in our tables.

\begin{proposition}[Correctness of distributional target]
\label{prop:proper}
The per-example loss $\ell(q;\pihat) = -\sum_k \pihat_k \log q_k$ is a strictly proper scoring rule over $q\in\Delta^K$: it is uniquely minimised at $q=\pihat(x)$ for each $x$.  Consequently, within the constrained family $\{\operatorname{softmax}(z/T): T>0\}$, $\mathcal{L}_{\text{soft}}$ identifies the temperature $T^*$ that minimises $\mathrm{KL}(\pihat\|\operatorname{softmax}(z/T^*))$, targeting the correct distributional objective, unlike $\mathcal{L}_{\text{TS}}$, which targets the one-hot voted label.  Note that $T^*$ need not achieve $\operatorname{softmax}(z/T^*)=\pihat$ exactly within this one-parameter family; the proper scoring rule property guarantees correctness of the \emph{target}, not achievability within the constrained family.
\end{proposition}

\noindent\textit{Monte Carlo Temperature Scaling (MCTS).}
In practice, pre-aggregated label distributions $\pihat(x_i)$ may not be available; instead, individual annotation records $\{a_{i1},\ldots,a_{iS}\}$ are stored per example.  \textbf{MCTS} draws $S$ such samples $a_{is}\sim\pihat(x_i)$ and minimises
$\mathcal{L}_{\text{MCTS}}(T) = -\frac{1}{nS}\sum_{i,s} \log\operatorname{softmax}(z_i/T)_{a_{is}}$.
Since $\bE_{\hat{y}\sim\pihat}[\mathrm{CE}(\operatorname{softmax}(z/T),\hat{y})] = \KL(\pihat\|\operatorname{softmax}(z/T)) + H(\pihat)$, MCTS converges to the deterministic objective~\eqref{eq:slts-loss} as $S\to\infty$.
The key practical finding is that \emph{even $S=1$} suffices: on CIFAR-10H, a single annotation per calibration example achieves ECE$_\text{true}$ of $1.53\%$, matching the $S\!\to\!\infty$ limit within rounding error (Table~\ref{tab:mcts}).
This efficiency generalises across all eight dataset--architecture settings in Tables~\ref{tab:cifar} and~\ref{tab:isic-main}, where MCTS $S{=}1$ matches SLTS within 0.6~pp ECE (see Appendix~\ref{app:mcts-convergence} for the full convergence analysis).

\begin{table}[t]
  \centering
  \caption{
    \textbf{MCTS convergence on CIFAR-10H (ResNet-50).}  Mean $\pm$ std over 5 seeds; $S$ = MC annotation samples per calibration example.
  }
  \label{tab:mcts}
  \small
  \setlength{\tabcolsep}{6pt}
  \begin{tabular}{lcc}
    \toprule
    Method & ECE (\%) & $T^*$ \\
    \midrule
    TS (voted labels)   & 4.29          & $2.030$            \\
    \midrule
    MCTS $S=1$          & $1.53\pm0.01$ & $3.174\pm0.041$ \\
    MCTS $S=5$          & $1.53\pm0.02$ & $3.166\pm0.014$ \\
    MCTS $S=20$         & $1.52\pm0.00$ & $3.187\pm0.013$ \\
    MCTS $S=50$         & $1.52\pm0.01$ & $3.185\pm0.007$ \\
    MCTS $S=200$        & $1.52\pm0.00$ & $3.178\pm0.003$ \\
    SLTS ($S\to\infty$) & 1.51          & $3.180$            \\
    \bottomrule
  \end{tabular}
\end{table}

\noindent\textit{Vector Scaling (VS).}
VS extends the soft-label objective~\eqref{eq:slts-loss} by replacing the global temperature with a per-class temperature vector $\mathbf{T}=(T_1,\ldots,T_K)$, so that the calibrated logit for class $k$ becomes $z_k/T_k$.  This accommodates datasets where some classes are systematically more ambiguous than others, at the cost of $K$ parameters rather than one.  The loss remains the KL divergence against $\pihat$ (Eq.~\ref{eq:slts-loss}).

\noindent\textit{Distributional Isotonic Regression (IR-Soft).}
IR-Soft fits a monotone step function from the model's predicted confidence $\hat{p}_{\hat{c}}(x)$ to the mean annotator probability for the top class, $\pihat_{\hat{c}}(x)$, using the Pool Adjacent Violators Algorithm (PAVA)~\cite{zadrozny2002transforming}.  As a non-parametric method, it makes no assumptions about the functional form of the calibration map and is capable of correcting arbitrary monotone distortions.

\noindent\textit{SoftPlatt.}
SoftPlatt applies a diagonal affine transformation $\hat{q}_k=\operatorname{softmax}(w_k z_k + b_k)$, the same parametric family as Platt scaling~\cite{platt1999probabilistic}, but fits the parameters by minimising $\KL(\pihat(x)\|\hat{q}(x))$ against the annotator distribution.  This design isolates the effect of the distributional target from any architectural difference relative to standard Platt scaling.

\noindent\textit{Dirichlet-Soft.}
Dirichlet-Soft applies the full $K\!\times\!K$ affine transformation $\hat{q}(x)=\operatorname{softmax}(Wz(x)+b)$ and fits $W\in\bR^{K\times K}$, $b\in\bR^K$ by minimising $\KL(\pihat(x)\|\hat{q}(x))$ with ODIR regularisation ($\lambda=10^{-3}$)~\cite{kull2019beyond}.  This constitutes the most expressive parametric calibrator in our framework, with $K^2+K$ free parameters, and represents the natural ambiguity-aware counterpart of the voted-label Dirichlet calibration of Kull et al.~\cite{kull2019beyond}.

% ─────────────────────────────────────────────────────────────────────────────
\subsection{Annotation-Free Calibration}
\label{sec:annotation-free}
% ─────────────────────────────────────────────────────────────────────────────

The methods above require annotator distributions $\pihat(x_i)$ at calibration time.
When only voted labels $y^*_i$ are available, we propose \textbf{Label-Smooth Temperature Scaling (LS-TS)}, which constructs a pseudo-target distribution from the model's own predictions without any additional annotations.

Let $\bar\varepsilon = \frac{1}{n}\sum_{i=1}^n \bigl(1 - \hat p_{y^*_i}(x_i)\bigr)$ be the mean complement of the model's voted-class confidence on the calibration set.
Define the pseudo-target distribution
\begin{equation}
  \tpihat^{\mathrm{LS}}_i
  = (1-\bar\varepsilon)\,e_{y^*_i} + \tfrac{\bar\varepsilon}{K}\,\mathbf{1},
  \label{eq:lsts}
\end{equation}
and minimise $\mathcal{L}_{\mathrm{soft}}$ (Eq.~\ref{eq:slts-loss}) with these targets.
The global smoothing weight $\bar\varepsilon$ is a data-driven estimate of average annotator disagreement: if the model is 80\% confident on the voted class, the remaining 20\% is spread uniformly across all $K$ classes.  The result is $T^*_{\mathrm{LS}} > T^*_{\mathrm{TS}}$ whenever $\bar\varepsilon > 0$, moving the calibrator in the correct direction even without annotator data.

\noindent\textit{Theoretical interpretation.}
The smoothing weight $\bar\varepsilon$ satisfies a moment-matching fixed point: the expected pseudo-label mass on the voted class equals the average model confidence, $\mathbb{E}[\tpihat^{\mathrm{LS}}_{y^*}] = \mathbb{E}[\hat p_{y^*}]$.  Algorithmically, LS-TS can be viewed as a single E-step of an EM algorithm in which the latent annotator distribution is estimated from the current (pre-calibration) model, followed by one M-step that optimises $T$ given those pseudo-labels.  This connection to knowledge distillation \cite{hinton2015distilling} explains the self-referential nature of the approach: the uncalibrated model both supplies the pseudo-targets and is then calibrated against them.  Section~\ref{sec:ablation-lsts} compares LS-TS against three simpler smoothing strategies and confirms that the data-driven global $\bar\varepsilon$ is the key design choice.  Appendix~\ref{app:ats} further shows that even an \emph{adaptive} voted-label calibrator (one that learns a per-instance temperature from logit-derived features) fails to improve over TS, reinforcing that the target distribution, not the flexibility of the calibration map, is the limiting factor.

% ─────────────────────────────────────────────────────────────────────────────
\section{Experiments}
\label{sec:experiments}
% ─────────────────────────────────────────────────────────────────────────────

The experimental evaluation addresses three questions: (1)~Does the voted-label target, rather than model capacity, limit calibration quality under ambiguity? (2)~How much annotation information is needed to enter the ambiguity-aware regime? (3)~Are the improvements robust to dataset characteristics, model architectures, and random variation?  We evaluate on four benchmarks with multi-annotator data, each with two backbone architectures: CIFAR-10H (ResNet-50, ViT-B/16), ChaosNLI (RoBERTa-Large, DeBERTa-v3), ISIC~2019 (EfficientNet-B4, ViT-S/16), and DermaMNIST (ResNet-18, ViT-S/16).

\noindent\textbf{Baselines.}
We compare against (i)~\textbf{Uncalibrated}: raw softmax; (ii)~\textbf{TS}: Temperature Scaling on voted labels \cite{guo2017calibration}; (iii)~\textbf{ATS} (Adaptive Temperature Scaling \cite{joy2023sample}): per-instance temperature predicted from four logit-derived features ($\max_k z_k$, prediction entropy, top-2 margin, top-1 confidence) via a linear model, trained with the same voted-label NLL as TS (the most direct adaptive extension of TS without annotator data); (iv)~\textbf{Platt (PS)}: per-class weight and bias on logits, calibrated against voted labels \cite{platt1999probabilistic,kull2019beyond}; (v)~\textbf{Dirichlet-Hard}: Dirichlet calibration \cite{kull2019beyond} with the full $K\times K$ weight matrix and bias, calibrated against voted labels.  Dirichlet-Hard uses the \emph{same architecture} as our Dirichlet-Soft but with voted-label targets; the contrast isolates the effect of the calibration target from model capacity.  ATS similarly tests whether richer architecture, without changing the target, is sufficient to fix the calibration gap.

\noindent\textbf{Metrics.}
We report two families of metrics, which differ in how the ground-truth label is treated:

Binning-based ECE metrics ($\ECEtrue$, aECE, cwECE) are estimated by drawing $S=100$ labels $\tilde{y}^{(s)}_i \sim \pihat(x_i)$ per example and averaging calibration error over draws.
$\ECEtrue$ uses $B=15$ equal-width confidence bins; aECE (adaptive ECE \cite{nixon2019measuring}) uses equal-mass bins that adapt to the confidence distribution; cwECE \cite{kull2019beyond} averages per-class ECE over $K$ classes.
These mirror the evaluation protocol faced by a practitioner who observes one label per prediction at test time.

Soft-target metrics (Brier score $\mathrm{Br}$, NLL) are computed directly against the soft target $\pihat$:
\begin{align}
  \mathrm{Br}_i &= \|\phat(x_i) - \pihat(x_i)\|^2, \label{eq:soft-metrics}\\
  \mathrm{NLL}_i &= -\sum_k \pihat_k(x_i)\log\phat_k(x_i)
    = H\!\bigl(\pihat(x_i),\,\phat(x_i)\bigr), \nonumber
\end{align}
averaged over the test set.
Brier and NLL are strictly proper scoring rules that directly penalise divergence from the full annotator distribution, complementing the binning-based ECE metrics.

We abbreviate $\ECEtrue$, aECE, cwECE, Br, NLL throughout.

% ─────────────────────────────────────────────────────────────────────────────
\subsection{CIFAR-10H: Image Classification with Human Label Noise}
\label{sec:cifar10h}
% ─────────────────────────────────────────────────────────────────────────────

\noindent\textbf{Dataset and models.}
CIFAR-10H \cite{peterson2019human} provides ${\approx}51$ human annotations per image for the 10\,000-image CIFAR-10 test set, directly exposing perceptual ambiguity in low-resolution natural images through repeated human labeling.  The test set is split into calibration ($n=5{,}000$) and evaluation ($5{,}000$) by stratified sampling (seed 42).  We evaluate two architectures:
\textbf{ResNet-50} \cite{he2016deep} pretrained on ImageNet-1k, fine-tuned for 30 epochs (AdamW, cosine annealing), achieving 97.3\% top-1 accuracy.
\textbf{ViT-B/16} \cite{dosovitskiy2021image} pretrained on ImageNet-21k, fine-tuned with the same protocol, achieving 98.1\% top-1 accuracy.

\begin{table*}[t]
  \centering
  \caption{
    \textbf{CIFAR-10H and ChaosNLI results.}
    ECE/aECE/cwECE averaged over 100 sampled labels $\tilde{y}\sim\pihat$; Brier/NLL computed against soft target $\pihat$ (Eq.~\ref{eq:soft-metrics}).
    Best per architecture in \textbf{bold} (Oracle TS excluded).
    \textit{Oracle TS} fits TS on the test set with soft labels (cheating upper bound on soft NLL; not a strict ECE upper bound).
  }
  \label{tab:cifar}\label{tab:chaosnli}
  \footnotesize
  \setlength{\tabcolsep}{3.5pt}
  \renewcommand{\arraystretch}{0.82}
  \resizebox{\linewidth}{!}{%
  \begin{tabular}{l cccccc cccccc}
    \toprule
    & \multicolumn{6}{c}{\textbf{ResNet-50}} & \multicolumn{6}{c}{\textbf{ViT-B/16}} \\
    \cmidrule(lr){2-7} \cmidrule(lr){8-13}
    Method & $T$ & ECE$\downarrow$ & aECE$\downarrow$ & cwECE$\downarrow$ & Br$\downarrow$ & NLL$\downarrow$
           & $T$ & ECE$\downarrow$ & aECE$\downarrow$ & cwECE$\downarrow$ & Br$\downarrow$ & NLL$\downarrow$ \\
    \midrule
    \multicolumn{13}{l}{\textbf{CIFAR-10H} \textit{($m{\approx}51$ annotations per image, $K{=}10$)}} \\
    \midrule
    \multicolumn{13}{l}{\textit{Voted-label baselines}} \\
    Uncalibrated   & ---  &  4.97 &  5.71 &  1.09 & 0.120 & 0.692  & ---  &  4.99 &  5.36 &  1.06 & 0.111 & 0.678 \\
    TS             & 2.03 &  4.29 &  4.25 &  0.91 & 0.112 & 0.363  & 2.04 &  4.48 &  4.40 &  0.93 & 0.106 & 0.346 \\
    ATS            & --- &  4.40 &  4.37 &  0.94 & 0.113 & 0.350  & --- &  4.54 &  4.47 &  0.94 & 0.106 & 0.345 \\
    Platt          & ---  &  4.29 &  4.23 &  0.91 & 0.112 & 0.372  & ---  &  4.54 &  4.38 &  0.93 & 0.105 & 0.363 \\
    Dirichlet-Hard & ---  &  4.46 &  4.44 &  0.95 & 0.114 & 0.395  & ---  &  4.70 &  4.62 &  0.97 & 0.107 & 0.394 \\
    \cmidrule(l){1-13}
    \multicolumn{13}{l}{\textit{Ambiguity-aware (ours)}} \\
    SLTS           & 3.18 &  1.51 &  1.39 &  0.45 & 0.110 & 0.293  & 3.07 &  0.85 &  0.56 &  0.39 & 0.102 & 0.278 \\
    MCTS $S{=}1$   & 3.14 &  1.45 &  1.44 &  0.45 & 0.111 & 0.296  & 3.07 &  0.81 &  0.65 &  0.42 & 0.101 & 0.277 \\
    SoftPlatt      & ---  &  1.52 &  1.31 &  \textbf{0.35} & 0.110 & 0.288  & ---  &  0.88 &  0.47 &  \textbf{0.24} & 0.101 & 0.272 \\
    VS             & ---  &  1.35 &  1.26 &  0.39 & \textbf{0.109} & 0.289  & ---  &  0.92 &  0.50 &  0.32 & 0.101 & 0.274 \\
    IR-Soft        & ---  &  \textbf{0.72} & \textbf{0.83} &  0.45 & 0.115 & 0.340  & ---  &  0.91 &  0.61 &  0.43 & 0.105 & 0.321 \\
    Dirichlet-Soft & ---  &  1.25 &  1.06 &  \textbf{0.35} & \textbf{0.109} & \textbf{0.271}  & ---  &  \textbf{0.72} & \textbf{0.41} &  0.25 & \textbf{0.100} & \textbf{0.255} \\
    \cmidrule(l){1-13}
    \multicolumn{13}{l}{\textit{Annotation-free (ours)}} \\
    LS-TS          & 3.08 &  1.57 &  1.31 &  0.46 & 0.109 & 0.293  & 2.76 &  2.37 &  2.21 &  0.53 & 0.103 & 0.285 \\
    \cmidrule(l){1-13}
    \multicolumn{13}{l}{\textit{Oracle (soft labels on test set)}} \\
    Oracle TS      & 3.17 &  1.50 &  1.52 &  0.45 & 0.111 & 0.296  & 3.09 &  0.70 &  0.63 &  0.42 & 0.101 & 0.277 \\
    \midrule
    & \multicolumn{6}{c}{\textbf{RoBERTa-Large}} & \multicolumn{6}{c}{\textbf{DeBERTa-v3}} \\
    \cmidrule(lr){2-7} \cmidrule(lr){8-13}
    Method & $T$ & ECE$\downarrow$ & aECE$\downarrow$ & cwECE$\downarrow$ & Br$\downarrow$ & NLL$\downarrow$
           & $T$ & ECE$\downarrow$ & aECE$\downarrow$ & cwECE$\downarrow$ & Br$\downarrow$ & NLL$\downarrow$ \\
    \midrule
    \multicolumn{13}{l}{\textbf{ChaosNLI} \textit{(SNLI+MNLI, $m{=}100$ annotations per example, $K{=}3$)}} \\
    \midrule
    \multicolumn{13}{l}{\textit{Voted-label baselines}} \\
    Uncalibrated   & ---  & 27.79 & 27.77 & 18.61 & 0.650 & 1.384  & ---  & 35.97 & 35.92 & 24.02 & 0.743 & 2.148 \\
    TS             & 2.42 & 10.55 & 10.38 &  7.30 & 0.537 & 0.886  & 3.81 & 11.63 & 11.57 &  8.48 & 0.549 & 0.900 \\
    ATS            & 2.28 & 11.36 & 11.29 &  7.73 & 0.538 & 0.883  & 3.57 & 13.09 & 13.07 &  9.06 & 0.552 & 0.900 \\
    Platt          & ---  & 11.16 & 10.92 &  7.28 & 0.530 & 0.875  & ---  & 12.43 & 12.36 &  8.43 & 0.539 & 0.896 \\
    Dirichlet-Hard & ---  & 11.55 & 11.40 &  7.64 & 0.535 & 0.889  & ---  & 12.69 & 12.63 &  8.84 & 0.539 & 0.895 \\
    \cmidrule(l){1-13}
    \multicolumn{13}{l}{\textit{Ambiguity-aware (ours)}} \\
    SLTS           & 3.41 &  3.22 &  2.59 &  4.62 & 0.520 & 0.862  & 5.28 &  3.45 &  3.36 &  6.53 & 0.529 & 0.877 \\
    MCTS $S{=}1$   & 3.49 &  2.82 &  2.35 &  4.60 & 0.519 & 0.862  & 5.45 &  3.20 &  2.98 &  6.30 & 0.529 & 0.877 \\
    SoftPlatt      & ---  &  2.66 &  2.32 &  2.43 & \textbf{0.509} & \textbf{0.839}  & ---  &  3.17 &  2.92 &  3.18 & 0.511 & 0.841 \\
    VS             & ---  &  3.46 &  3.07 &  5.13 & 0.518 & 0.857  & ---  &  3.91 &  3.35 &  6.46 & 0.525 & 0.869 \\
    IR-Soft        & ---  &  2.65 &  2.59 &  6.24 & 0.549 & 0.934  & ---  &  \textbf{2.15} &  3.96 &  7.20 & 0.554 & 0.942 \\
    Dirichlet-Soft & ---  &  \textbf{2.57} & \textbf{1.71} &  \textbf{2.03} & 0.510 & \textbf{0.839}  & ---  &  3.19 &  \textbf{2.71} &  \textbf{2.95} & \textbf{0.509} & \textbf{0.836} \\
    \cmidrule(l){1-13}
    \multicolumn{13}{l}{\textit{Annotation-free (ours)}} \\
    LS-TS          & 4.40 &  4.16 &  3.72 &  5.61 & 0.522 & 0.873  & 6.15 &  2.65 &  2.99 &  6.32 & 0.529 & 0.881 \\
    \cmidrule(l){1-13}
    \multicolumn{13}{l}{\textit{Oracle (soft labels on test set)}} \\
    Oracle TS      & 3.34 &  3.48 &  2.83 &  4.64 & 0.520 & 0.862  & 5.21 &  3.66 &  3.54 &  6.66 & 0.529 & 0.876 \\
    \bottomrule
  \end{tabular}}
\end{table*}

\begin{figure*}[t]
  \centering
  \includegraphics[width=\linewidth]{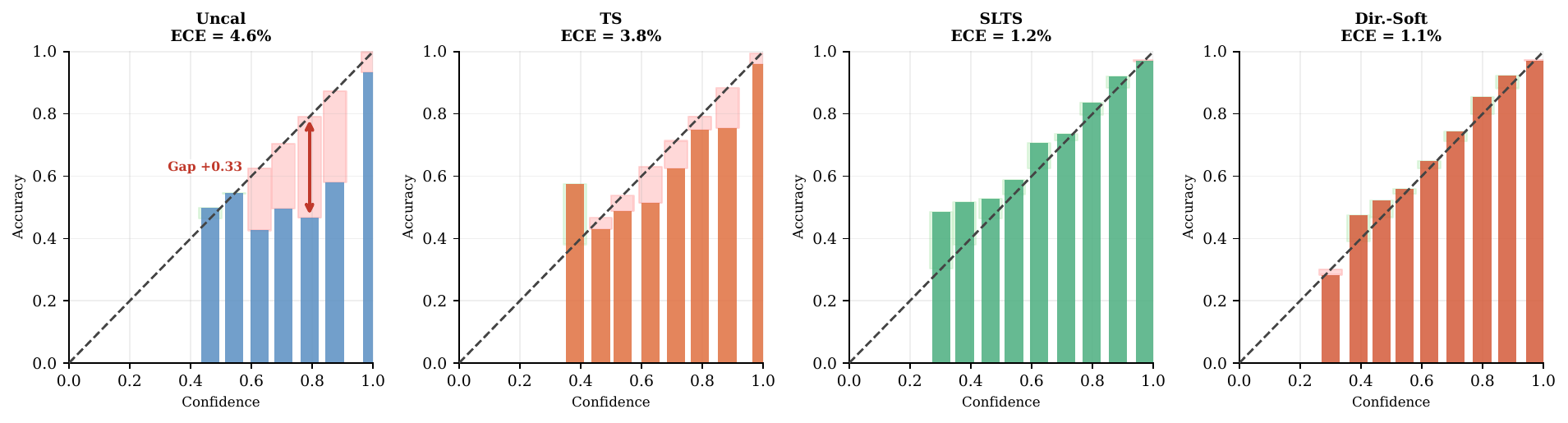}
  \caption{
    \textbf{Reliability diagrams for four representative methods: CIFAR-10H ResNet-50 (ECE$_\text{true}$).}
    Red shading indicates overconfidence; green indicates underconfidence.
    Uncal and TS remain overconfident; soft-label temperature methods (MCTS/SLTS) substantially correct this; Dirichlet-Soft reduces the residual gap further.
    Reliability diagrams for all remaining methods are provided in Appendix~\ref{app:reliability}.
  }
  \label{fig:reliability-main}
\end{figure*}

\begin{figure*}[t]
  \centering
  \includegraphics[width=\linewidth]{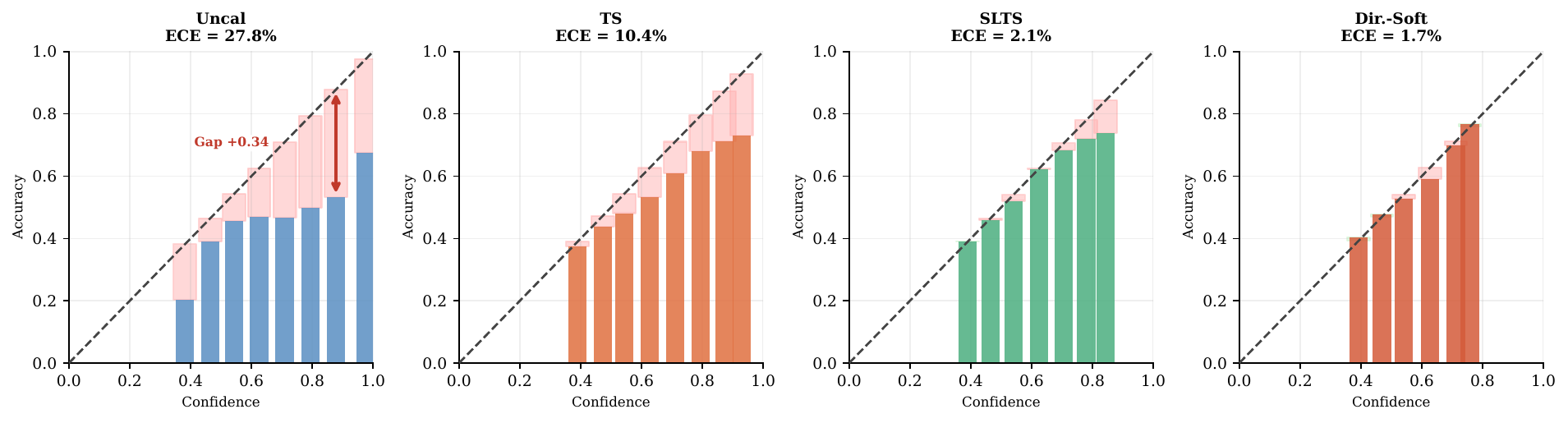}
  \caption{
    \textbf{Reliability diagrams for four representative methods: ChaosNLI RoBERTa-Large (ECE$_\text{true}$).}
    Same method selection as Figure~\ref{fig:reliability-main}.
    Red shading indicates overconfidence; green indicates underconfidence.
    NLI models are severely overconfident before calibration; ambiguity-aware methods substantially correct this.
  }
  \label{fig:reliability-chaosnli-main}
\end{figure*}

\noindent\textbf{Results.}
Table~\ref{tab:cifar} (upper half) reports results for both CIFAR-10H backbones.  All voted-label baselines leave 4--5\% ECE$_\text{true}$, whereas ambiguity-aware methods reduce it to below 2\%.  Two comparisons are particularly informative.  First, ATS (adaptive per-instance temperature) achieves $4.40\%$ ECE on ResNet-50 and $4.54\%$ on ViT-B/16, no better than global TS ($4.29\%$ and $4.48\%$, respectively) and far behind even the annotation-free LS-TS ($1.57\%$ and $2.37\%$), confirming that per-instance temperature adaptation cannot compensate for the wrong calibration target.  Second, MCTS $S{=}1$ matches SLTS on both architectures ($1.45\%$ vs.\ $1.51\%$ on ResNet-50; $0.81\%$ vs.\ $0.85\%$ on ViT-B/16), demonstrating that individual annotation records are as informative as pre-aggregated distributions for the purpose of temperature calibration.  Among the more expressive calibrators, Dirichlet-Soft achieves the best Brier score and NLL on both backbones (0.109/0.271 on ResNet-50; 0.100/0.255 on ViT-B/16).  LS-TS, operating with voted labels alone, already reduces ECE from 4.29\% to 1.57\% (ResNet-50) and from 4.48\% to 2.37\% (ViT-B/16), demonstrating that even approximate soft targets yield substantial improvements.

% ─────────────────────────────────────────────────────────────────────────────
\subsection{ChaosNLI: Natural Language Inference}
\label{sec:chaosnli}
% ─────────────────────────────────────────────────────────────────────────────

We next evaluate on \textbf{ChaosNLI} \cite{nie2020what}: 100 human annotations per example for the combined SNLI+MNLI subset ($n=3{,}113$), where disagreement reflects genuine semantic ambiguity rather than annotation noise.  We split the data into calibration and test sets (50/50, stratified by majority-vote label to preserve class balance).  We evaluate \textbf{RoBERTa-Large} \cite{liu2019roberta} and \textbf{DeBERTa-v3-base} \cite{he2021debertav3}, both fine-tuned on MNLI training data; ChaosNLI draws from the SNLI and MNLI dev/test splits, so no training examples are re-used at evaluation.

\noindent\textbf{Results.}
Table~\ref{tab:cifar} (lower half) reports ChaosNLI results.  The NLI setting amplifies the calibration gap observed on CIFAR-10H: voted-label baselines leave 10--12\% ECE, and ATS ($11.36\%$ on RoBERTa-L, $13.09\%$ on DeBERTa-v3) performs \emph{worse} than global TS on DeBERTa-v3, further corroborating that adaptive capacity without the correct target is ineffective.  MCTS $S{=}1$ again closely matches SLTS ($2.82\%$ vs.\ $3.22\%$ on RoBERTa-L; $3.20\%$ vs.\ $3.45\%$ on DeBERTa-v3).  On this benchmark MCTS $S{=}1$ in fact slightly outperforms SLTS, which may be attributable to the stochastic single-annotation objective acting as an implicit regulariser.  The best distributional fit is achieved by Dirichlet-Soft and SoftPlatt: on RoBERTa-L, SoftPlatt achieves the best Brier score (0.509 vs.\ 0.510 for Dirichlet-Soft) and both share the same NLL (0.839); on DeBERTa-v3, Dirichlet-Soft achieves the best Brier (0.509) and NLL (0.836).  LS-TS reduces ECE from 10.55\% to 4.16\% (RoBERTa-L) and from 11.63\% to 2.65\% (DeBERTa-v3) using only voted labels, substantially outperforming all hard-label baselines.

\noindent\textbf{ATS vs.\ LS-TS: the calibration target is the decisive factor.}
A controlled comparison between ATS (increased architectural capacity, voted-label target) and LS-TS (global temperature, soft target) shows that ATS fails to improve over TS in all eight settings and in several cases degrades performance (ECE increases from $11.63\%$ to $13.09\%$ on DeBERTa-v3; from $17.04\%$ to $17.53\%$ on ISIC ViT-S/16).  The per-instance temperature learned from logit features adapts to model uncertainty, which is only weakly correlated with annotator disagreement.  In contrast, LS-TS reduces ECE by 9--77\% across all eight settings without any annotator data (see Appendix~\ref{app:ats} for the full table and analysis).

% ─────────────────────────────────────────────────────────────────────────────
\subsection{ISIC 2019: Skin Disease Diagnosis}
\label{sec:isic2019}
% ─────────────────────────────────────────────────────────────────────────────

We evaluate on \textbf{ISIC~2019} \cite{tschandl2018ham10000,combalia2019bcn} (8 skin conditions, ${\approx}25{,}000$ images), a clinically realistic differential-diagnosis task in which several lesion categories have overlapping visual morphology.  We use $m=9$ synthetic dermatologist annotations per image, generated by sampling from a clinically calibrated $8\times8$ confusion matrix following Dawid and Skene~\cite{dawid1979maximum} (Appendix~\ref{app:isic-confusion}), with diagonal agreement matched to Liu et al.~\cite{liu2020deep}: mean $75\%$.  Because no per-image reader study is publicly available for this dataset, annotation ambiguity is modelled \emph{class-conditionally}: each image's confusion depends on its consensus class, not on its individual visual content, which is a deliberate simplification of true instance-level annotator variation.  We evaluate \textbf{EfficientNet-B4} \cite{tan2019efficientnet} and \textbf{ViT-S/16} \cite{dosovitskiy2021image} on a stratified 70/15/15 split.

\begin{table*}[t]
  \centering
  \caption{
    \textbf{ISIC 2019 and DermaMNIST results.}
    ECE/aECE/cwECE averaged over 100 sampled labels $\tilde{y}\sim\pihat$; Brier/NLL computed against soft target $\pihat$ (Eq.~\ref{eq:soft-metrics}).
    Best per architecture in \textbf{bold} (Oracle TS excluded).
    \textit{Oracle TS} fits TS on the test set with soft labels (cheating upper bound on soft NLL; not a strict ECE upper bound).
  }
  \label{tab:isic-main}\label{tab:derm}
  \footnotesize
  \setlength{\tabcolsep}{3.5pt}
  \renewcommand{\arraystretch}{0.82}
  \resizebox{\linewidth}{!}{%
  \begin{tabular}{l cccccc cccccc}
    \toprule
    & \multicolumn{6}{c}{\textbf{EfficientNet-B4}} & \multicolumn{6}{c}{\textbf{ViT-S/16}} \\
    \cmidrule(lr){2-7} \cmidrule(lr){8-13}
    Method & $T$ & ECE$\downarrow$ & aECE$\downarrow$ & cwECE$\downarrow$ & Br$\downarrow$ & NLL$\downarrow$
           & $T$ & ECE$\downarrow$ & aECE$\downarrow$ & cwECE$\downarrow$ & Br$\downarrow$ & NLL$\downarrow$ \\
    \midrule
    \multicolumn{13}{l}{\textbf{ISIC 2019} \textit{($m{=}9$ synthetic annotators, $K{=}8$, mean agreement $75\%$)}} \\
    \midrule
    \multicolumn{13}{l}{\textit{Voted-label baselines}} \\
    Uncalibrated   & ---  & 25.87 & 25.97 &  6.64 & 0.600 & 2.686  & ---  & 21.28 & 21.17 &  6.21 & 0.650 & 1.803 \\
    TS             & 1.77 & 18.72 & 18.70 &  5.02 & 0.559 & 1.661  & 1.23 & 17.04 & 16.95 &  5.49 & 0.630 & 1.578 \\
    ATS            & 1.70 & 18.83 & 18.80 &  5.03 & 0.560 & 1.733  & 1.12 & 17.53 & 17.48 &  5.46 & 0.631 & 1.730 \\
    Platt          & ---  & 19.16 & 19.08 &  5.20 & 0.559 & 1.658  & ---  & 18.17 & 18.14 &  4.94 & 0.610 & 1.595 \\
    Dirichlet-Hard & ---  & 20.03 & 19.96 &  5.25 & 0.560 & 1.677  & ---  & 18.51 & 18.49 &  5.02 & 0.613 & 1.617 \\
    \cmidrule(l){1-13}
    \multicolumn{13}{l}{\textit{Ambiguity-aware (ours)}} \\
    SLTS           & 4.42 &  9.76 &  9.91 &  3.03 & 0.529 & 1.149  & 2.68 &  7.10 &  6.77 &  3.35 & 0.594 & 1.248 \\
    MCTS $S{=}1$   & 4.51 & 10.09 & 10.15 &  3.08 & 0.530 & 1.149  & 2.77 &  7.47 &  7.15 &  3.35 & 0.595 & 1.248 \\
    SoftPlatt      & ---  &  9.98 &  9.99 &  2.33 & 0.523 & 1.111  & ---  &  7.82 &  7.71 &  1.90 & 0.572 & 1.201 \\
    VS             & ---  &  9.68 &  9.64 &  2.48 & 0.524 & 1.123  & ---  &  7.74 &  7.59 &  2.14 & 0.574 & 1.211 \\
    IR-Soft        & ---  &  \textbf{1.77} & \textbf{2.21} &  4.36 & 0.539 & 1.288  & ---  &  \textbf{2.05} & \textbf{2.13} &  5.65 & 0.627 & 1.479 \\
    Dirichlet-Soft & ---  &  8.35 &  8.39 &  \textbf{2.08} & \textbf{0.518} & \textbf{1.086}  & ---  &  6.32 &  6.28 &  \textbf{1.64} & \textbf{0.568} & \textbf{1.184} \\
    \cmidrule(l){1-13}
    \multicolumn{13}{l}{\textit{Annotation-free (ours)}} \\
    LS-TS          & 4.30 &  9.34 &  9.46 &  2.98 & 0.528 & 1.149  & 3.98 & 15.49 & 15.47 &  4.52 & 0.619 & 1.303 \\
    \cmidrule(l){1-13}
    \multicolumn{13}{l}{\textit{Oracle (soft labels on test set)}} \\
    Oracle TS      & 4.38 &  9.64 &  9.77 &  3.00 & 0.528 & 1.149  & 2.70 &  7.23 &  6.86 &  3.34 & 0.594 & 1.248 \\
    \midrule
    & \multicolumn{6}{c}{\textbf{ResNet-18}} & \multicolumn{6}{c}{\textbf{ViT-S/16}} \\
    \cmidrule(lr){2-7} \cmidrule(lr){8-13}
    Method & $T$ & ECE$\downarrow$ & aECE$\downarrow$ & cwECE$\downarrow$ & Br$\downarrow$ & NLL$\downarrow$
           & $T$ & ECE$\downarrow$ & aECE$\downarrow$ & cwECE$\downarrow$ & Br$\downarrow$ & NLL$\downarrow$ \\
    \midrule
    \multicolumn{13}{l}{\textbf{DermaMNIST} \textit{($m{=}5$ synthetic annotators, $K{=}7$, mean agreement $64.7\%$)}} \\
    \midrule
    \multicolumn{13}{l}{\textit{Voted-label baselines}} \\
    Uncalibrated   & ---  & 34.56 & 34.51 & 10.11 & 0.771 & 2.617  & ---  & 37.55 & 37.48 & 10.88 & 0.791 & 3.700 \\
    TS             & 1.84 & 22.25 & 22.20 &  6.91 & 0.687 & 1.637  & 2.55 & 24.65 & 24.57 &  7.30 & 0.683 & 1.664 \\
    ATS            & 1.86 & 22.82 & 22.75 &  6.97 & 0.687 & 1.606  & 2.53 & 25.35 & 25.24 &  7.53 & 0.686 & 1.650 \\
    Platt          & ---  & 23.10 & 23.03 &  6.93 & 0.682 & 1.623  & ---  & 24.18 & 24.02 &  7.13 & 0.675 & 1.670 \\
    Dirichlet-Hard & ---  & 23.02 & 22.79 &  7.06 & 0.683 & 1.704  & ---  & 24.55 & 24.46 &  7.33 & 0.680 & 1.877 \\
    \cmidrule(l){1-13}
    \multicolumn{13}{l}{\textit{Ambiguity-aware (ours)}} \\
    SLTS           & 3.43 &  5.13 &  4.76 &  4.02 & 0.629 & 1.388  & 5.03 &  3.29 &  2.74 &  3.15 & 0.609 & 1.346 \\
    MCTS $S{=}1$   & 3.42 &  5.07 &  4.74 &  4.02 & 0.629 & 1.388  & 5.04 &  3.28 &  2.71 &  3.16 & 0.609 & 1.346 \\
    SoftPlatt      & ---  &  4.05 &  3.80 &  1.51 & 0.614 & 1.313  & ---  &  3.41 &  3.05 &  1.36 & 0.599 & 1.287 \\
    VS             & ---  &  4.58 &  4.48 &  3.38 & 0.623 & 1.365  & ---  &  3.81 &  3.35 &  2.73 & 0.609 & 1.329 \\
    IR-Soft        & ---  &  \textbf{2.20} & \textbf{1.94} &  4.96 & 0.640 & 1.468  & ---  &  \textbf{2.06} & \textbf{1.87} &  4.57 & 0.626 & 1.423 \\
    Dirichlet-Soft & ---  &  3.94 &  3.93 &  \textbf{1.26} & \textbf{0.612} & \textbf{1.305}  & ---  &  3.18 &  3.00 &  \textbf{1.16} & \textbf{0.598} & \textbf{1.279} \\
    \cmidrule(l){1-13}
    \multicolumn{13}{l}{\textit{Annotation-free (ours)}} \\
    LS-TS          & 3.03 &  5.05 &  4.82 &  3.79 & 0.630 & 1.394  & 4.19 &  7.36 &  7.06 &  3.68 & 0.615 & 1.362 \\
    \cmidrule(l){1-13}
    \multicolumn{13}{l}{\textit{Oracle (soft labels on test set)}} \\
    Oracle TS      & 3.38 &  4.87 &  4.49 &  3.97 & 0.628 & 1.388  & 4.94 &  3.51 &  2.97 &  3.13 & 0.609 & 1.346 \\
    \bottomrule
  \end{tabular}}
\end{table*}

\noindent\textbf{ISIC 2019 results.}
Table~\ref{tab:isic-main} (upper half) shows that the benefits of ambiguity-aware calibration extend to clinically realistic settings with synthetic annotations.  Dirichlet-Soft achieves the best Brier and NLL on both backbones (0.518/1.086 on ENet-B4; 0.568/1.184 on ViT-S/16), while IR-Soft achieves the lowest ECE, a pattern consistent with its non-parametric flexibility.  MCTS $S{=}1$ closely tracks SLTS on ENet-B4 ($10.09\%$ vs.\ $9.76\%$) and ViT-S/16 ($7.47\%$ vs.\ $7.10\%$), replicating the pattern observed on natural-annotation benchmarks.  LS-TS approximately halves the ECE of TS on ENet-B4 ($18.72\%\to9.34\%$) using only voted labels; the improvement is smaller on ViT-S/16 ($17.04\%\to15.49\%$), where model confidence is a weaker proxy for annotation ambiguity.

\noindent\textbf{Annotator confusion model.}
Figure~\ref{fig:isic-confusion} shows the clinically calibrated $8\times8$ confusion matrix used to generate synthetic annotations for ISIC~2019.  The MEL/NV pair (melanoma/nevi, orange dashed) is the most consequential and most confused: diagonal rates of 73\%/76\% with 14\%/15\% cross-confusion, reflecting the well-documented difficulty of distinguishing early melanoma from benign nevi \cite{haenssle2018man}.  The AK/BKL/SCC cluster (purple dotted) forms a second high-confusion group (diagonal 65--67\%), capturing clinically significant overlap between actinic keratosis and squamous cell carcinoma.  The mean diagonal agreement $\bar{C}_{ii}=75\%$ matches the inter-reader agreement reported by Liu et al.~\cite{liu2020deep}.  Full matrix values and construction details are given in Appendix~\ref{app:isic-confusion}.

\begin{figure}[t]
  \centering
  \includegraphics[width=0.92\linewidth]{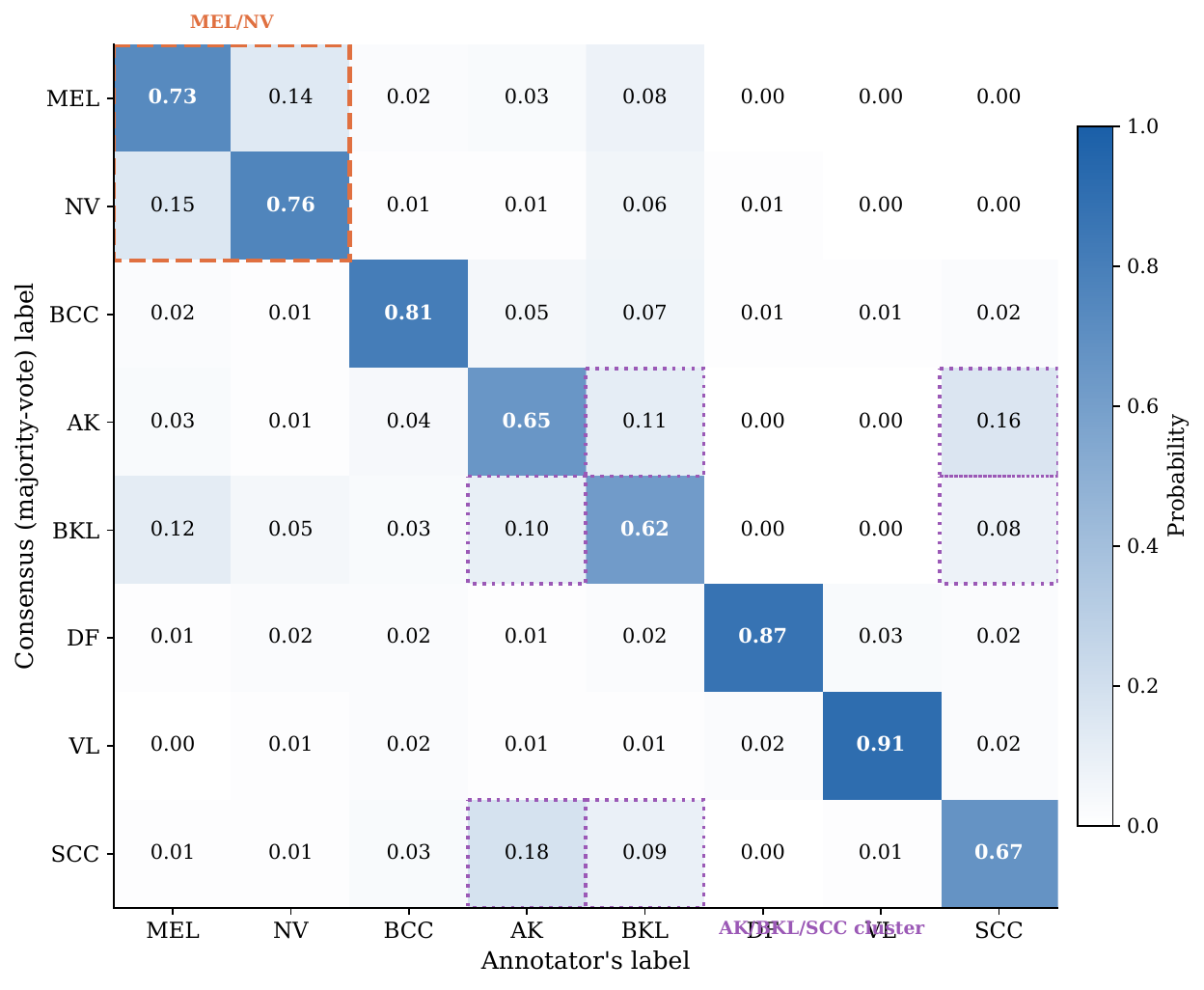}
  \caption{
    \textbf{Inter-reader confusion matrix for ISIC 2019.}
    Entry $C_{ij}$ is the probability that a dermatologist labels an image as
    class~$j$ given consensus label~$i$.
    Orange dashed box: MEL/NV most-confused pair ($73$--$76\%$ diagonal).
    Purple dotted: AK/BKL/SCC high-confusion cluster.
    Mean diagonal agreement $\bar{C}_{ii}=75\%$.
  }
  \label{fig:isic-confusion}
\end{figure}

% ─────────────────────────────────────────────────────────────────────────────
\subsection{DermaMNIST: Supplementary Skin Disease Experiment}
\label{sec:dermamnist}
% ─────────────────────────────────────────────────────────────────────────────

Table~\ref{tab:isic-main} (lower half) reports results on \textbf{DermaMNIST}~\cite{yang2023medmnist} (7-class, derived from HAM10000~\cite{tschandl2018ham10000}; $m=5$ synthetic annotators; overall agreement ${\approx}64.7\%$) with ResNet-18 and ViT-S/16.  TS leaves ECE $=22.3\%$ on ResNet-18; ambiguity-aware methods substantially close this gap: SLTS achieves $5.1\%$, IR-Soft $2.2\%$, and Dirichlet-Soft $3.9\%$, with Dirichlet-Soft also achieving the best cwECE ($1.3\%$), Brier ($0.612$), and NLL ($1.305$).  ViT-S/16 exhibits the same hierarchy (Dirichlet-Soft: $3.2\%$/$0.598$/$1.279$).  MCTS $S{=}1$ again closely approximates SLTS ($5.07\%$ vs.\ $5.13\%$ on ResNet-18; $3.28\%$ vs.\ $3.29\%$ on ViT-S/16).  Both backbones show $T_{\text{soft}}>T_{\text{TS}}>1$, consistent with Proposition~\ref{prop:ts-bias}.

DermaMNIST is the most ambiguous benchmark in our study, with mean annotator agreement of $64.7\%$ compared to $75\%$ for ISIC~2019, ${\approx}80\%$ for CIFAR-10H, and ${\approx}67\%$ for ChaosNLI.  Correspondingly, TS ECE$_\text{true}$ ($22.3\%$) is the highest across all benchmarks, reflecting the widest gap between the voted-label target and the true annotator distribution.  The Dirichlet-Hard calibrator, despite being the most expressive voted-label method (full $K\!\times\!K$ affine), worsens ECE relative to TS on ResNet-18 ($23.0\%$ vs.\ $22.3\%$) and provides no meaningful improvement on ViT-S/16 ($24.6\%$ vs.\ $24.7\%$), again confirming that expressive architecture cannot compensate for a miscalibrated target.  LS-TS reduces ECE by $77\%$ relative to TS on ResNet-18 ($22.3\%\to 5.1\%$), the largest relative improvement of the annotation-free method across all settings, because the high level of ambiguity means model confidence (complement 35.3\%) correlates well with the fraction of annotators choosing the minority class.

% ─────────────────────────────────────────────────────────────────────────────
\subsection{Cross-Benchmark Analysis}
\label{sec:cross-benchmark}
% ─────────────────────────────────────────────────────────────────────────────

Aggregating the results across all four benchmarks and eight architecture configurations, four consistent findings emerge that together characterise the role of the calibration target under label ambiguity.

\noindent\textbf{(1) Target, not capacity, is the bottleneck.}
Dirichlet-Hard, the most expressive voted-label calibrator, consistently \emph{underperforms} simple TS in both ECE and Brier (e.g., 4.46\% vs.\ 4.29\% on CIFAR-10H R50; 11.55\% vs.\ 10.55\% on ChaosNLI RoBERTa-L; see Tables~\ref{tab:cifar} and~\ref{tab:isic-main}).  This rules out architectural capacity as the bottleneck and directly supports the claim that calibration target is limiting.
\noindent\textbf{(2) Ambiguity-aware methods consistently dominate.}
Dirichlet-Soft improves ECE by 71--84\% relative to TS on the two natural-annotation benchmarks (CIFAR-10H and ChaosNLI) and by 55--87\% on the medical benchmarks (ISIC~2019 and DermaMNIST).  Dirichlet-Soft consistently achieves the best overall calibration quality, making it the recommended default whenever annotator distributions are available.
\noindent\textbf{(3) A single annotation per example suffices.}
MCTS $S{=}1$ matches SLTS within 0.6~pp ECE in all 8 settings.  The aggregated annotation distribution provides no measurable additional benefit once individual annotator labels are available: any dataset annotated once per example by a randomly selected annotator is immediately amenable to the same calibration quality as full multi-annotator annotation.
\noindent\textbf{(4) LS-TS is effective but dataset-dependent.}
LS-TS substantially reduces ECE without any annotator data in all settings except ISIC ViT-S/16 (where the LS-TS temperature exceeds the oracle soft-label temperature, $T_\text{LS}=3.98 > T_\text{SLTS}=2.68$, indicating that model confidence over-estimates annotation ambiguity for this architecture--dataset pair).
This divergence highlights the assumption underlying LS-TS: that model uncertainty is a reliable proxy for annotation disagreement, an assumption that holds broadly but can fail when the pre-softmax logit magnitude is not well correlated with the fraction of annotators that choose the minority class.

% ─────────────────────────────────────────────────────────────────────────────
\subsection{Reliability Diagrams}
\label{sec:reliability}
% ─────────────────────────────────────────────────────────────────────────────

Figures~\ref{fig:reliability-main} and~\ref{fig:reliability-chaosnli-main} present reliability diagrams for the natural-annotation benchmarks (CIFAR-10H and ChaosNLI); Figures~\ref{fig:reliability-isic-main} and~\ref{fig:reliability-derm-main} extend the comparison to the medical imaging benchmarks.  The visual pattern is consistent across all four datasets: Uncalibrated and TS outputs are systematically overconfident (red-shaded residuals), with TS reducing but not eliminating the gap.  Soft-label temperature methods (MCTS/SLTS) substantially correct the overconfidence, and Dirichlet-Soft closes most of the remaining residual.  These diagrams provide a qualitative complement to the quantitative ECE metrics reported in Tables~\ref{tab:cifar} and~\ref{tab:isic-main}.  Complete panels covering all thirteen methods across all eight architecture configurations are provided in Appendix~\ref{app:reliability}.

\begin{figure*}[t]
  \centering
  \includegraphics[width=\linewidth]{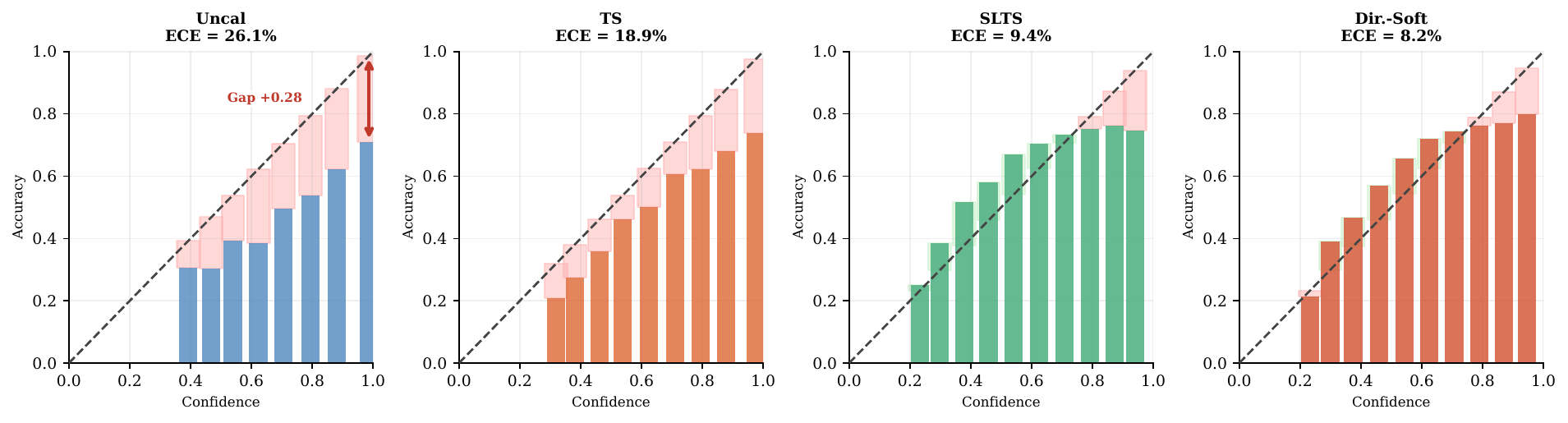}
  \caption{
    \textbf{Reliability diagrams: ISIC 2019 EfficientNet-B4 (ECE$_\text{true}$).}
    Red shading = overconfidence; green = underconfidence.
    Hard-label baselines remain severely overconfident at 18--26\% ECE; ambiguity-aware methods substantially correct this, with IR-Soft achieving the lowest ECE.
  }
  \label{fig:reliability-isic-main}
\end{figure*}

\begin{figure*}[t]
  \centering
  \includegraphics[width=\linewidth]{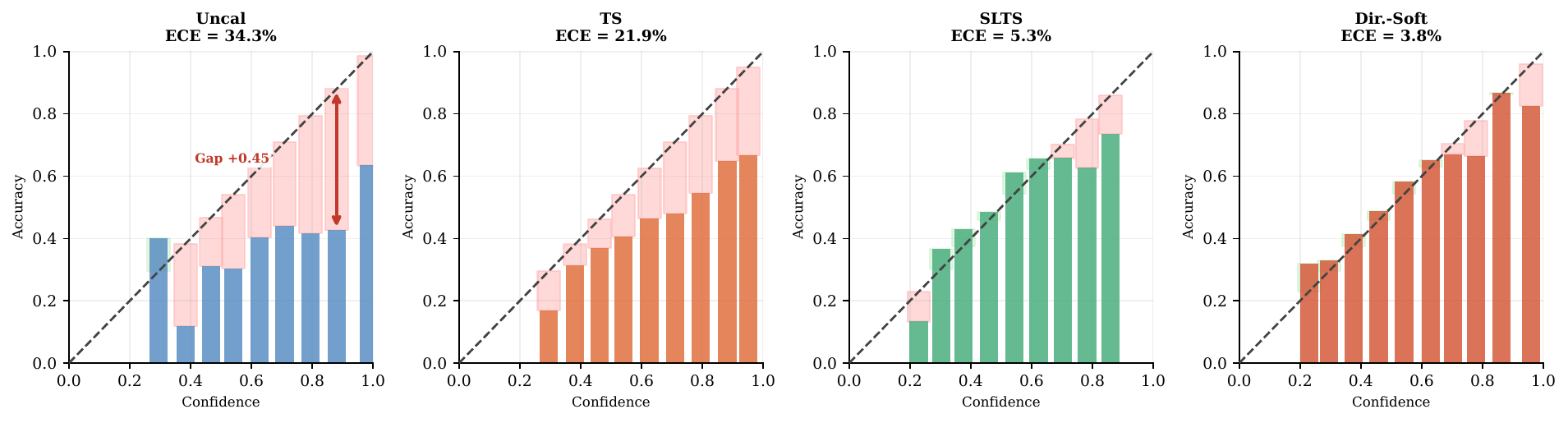}
  \caption{
    \textbf{Reliability diagrams: DermaMNIST ResNet-18 (ECE$_\text{true}$).}
    Same layout as Figure~\ref{fig:reliability-main}.
    TS leaves 22\% ECE; Dirichlet-Soft reduces it to 3.9\% while achieving the best Brier/NLL.
  }
  \label{fig:reliability-derm-main}
\end{figure*}

% ─────────────────────────────────────────────────────────────────────────────
\subsection{Ablation: LS-TS Smoothing Strategies}
\label{sec:ablation-lsts}
% ─────────────────────────────────────────────────────────────────────────────

Table~\ref{tab:lsts-ablation} compares LS-TS against three simpler annotation-free smoothing strategies on CIFAR-10H and ChaosNLI.  All methods use the same single-temperature family; only the pseudo-target construction differs.

\begin{table*}[t]
  \centering
  \caption{
    \textbf{LS-TS ablation: annotation-free smoothing strategies.}
    ECE$_\text{true}$ (\%), Brier score, and NLL on CIFAR-10H and ChaosNLI.
    All methods use the same single-temperature family; only the pseudo-target construction differs.
    Fixed-LS uses $\varepsilon{=}0.1$ regardless of the data; Ent-LS sets $\varepsilon_i{=}H(\hat{p}(x_i))/\log K$ per instance; CC-LS computes a per-class smoothing weight $\varepsilon_k{=}\mathrm{mean}_{i:\,y^*_i=k}(1-\hat{p}_{y^*}(x_i))$.
    Best annotation-free method per column in \textbf{bold}.
  }
  \label{tab:lsts-ablation}
  \small
  \setlength{\tabcolsep}{4.5pt}
  \begin{tabular}{l ccc ccc ccc ccc}
    \toprule
    & \multicolumn{3}{c}{\textbf{CIFAR R50}} & \multicolumn{3}{c}{\textbf{CIFAR ViT}} & \multicolumn{3}{c}{\textbf{Chaos RoBERTa}} & \multicolumn{3}{c}{\textbf{Chaos DeBERTa}} \\
    \cmidrule(lr){2-4}\cmidrule(lr){5-7}\cmidrule(lr){8-10}\cmidrule(lr){11-13}
    Method & ECE & Br & NLL & ECE & Br & NLL & ECE & Br & NLL & ECE & Br & NLL \\
    \midrule
    TS (voted) & 3.90 & .1126 & .346 & 4.01 & .1043 & .320 & 10.55 & .537 & .886 & 11.63 & .549 & .900 \\
    \midrule
    Fixed-LS ($\varepsilon{=}0.1$) & 7.44 & .1207 & .332 & 6.00 & .1058 & .305 & 6.46 & .525 & .866 & 7.53 & .536 & .882 \\
    Ent-LS  & 3.16 & .1116 & .321 & 3.69 & .1037 & .308 & \textbf{3.17} & \textbf{.520} & \textbf{.862} & 8.58 & .539 & .885 \\
    CC-LS   & \textbf{1.48} & \textbf{.1111} & \textbf{.296} & 1.69 & \textbf{.1016} & \textbf{.279} & 4.15 & .522 & .873 & \textbf{2.54} & \textbf{.529} & \textbf{.880} \\
    LS-TS (ours) & 1.49 & \textbf{.1111} & \textbf{.296} & \textbf{1.62} & \textbf{.1016} & \textbf{.279} & 4.16 & .522 & .873 & 2.65 & .529 & .881 \\
    \bottomrule
  \end{tabular}
\end{table*}

Fixed-LS with $\varepsilon{=}0.1$ consistently \emph{under}performs or even \emph{degrades} below TS (e.g., ECE 7.44\% vs.\ 3.90\% on CIFAR-10H ResNet-50), confirming that an arbitrary fixed smoothing weight is insufficient.  Ent-LS (per-instance entropy-based smoothing) is inconsistent across settings: it collapses on DeBERTa-v3 (8.58\%), suggesting prediction entropy is a noisy proxy when models are strongly confident.  CC-LS computes a separate $\varepsilon_k$ for each class $k$ using the mean complement-confidence of examples with voted label $k$; it matches LS-TS within 0.1~pp ECE on three of four settings (CIFAR-10H R50: 1.48\% vs.\ 1.49\%; ViT: 1.69\% vs.\ 1.62\%; ChaosNLI RoBERTa: 4.15\% vs.\ 4.16\%), confirming that the global mean is a good summary of per-class behaviour.  This near-equivalence is expected: because temperature scaling has only a single scalar parameter $T$, per-class differences in $\varepsilon_k$ are averaged out during optimisation, and both CC-LS and LS-TS converge to the same $T^*$ (e.g., both yield $T{=}3.19$ on CIFAR-10H R50).  The slight advantage of CC-LS on DeBERTa-v3 (2.54\% vs.\ 2.65\%) is within single-seed variance.

% ─────────────────────────────────────────────────────────────────────────────
\subsection{Robustness}
\label{sec:robustness}
% ─────────────────────────────────────────────────────────────────────────────

We assess robustness along three dimensions, summarised in Table~\ref{tab:multiseed} and Figure~\ref{fig:robustness-main}.

\noindent\textbf{Multi-seed stability.}
Standard deviations of ECE$_\text{true}$ over five random seeds are uniformly $2{\times}$--$10{\times}$ smaller than the ECE gaps between methods, confirming that the reported rankings are not artefacts of a particular data split or annotation draw.  Full per-architecture bar charts are provided in Appendix~\ref{app:multiseed}.

\noindent\textbf{Calibration set size.}
Ambiguity-aware methods stabilise below 2\% ECE on CIFAR-10H using as few as 5--10\% of the calibration data (${\approx}250$ examples), while TS remains flat regardless of sample size.  Full calibration-set-size curves for all benchmarks are given in Appendix~\ref{app:calset-size}.

\noindent\textbf{Annotation count.}
Ambiguity-aware calibration is robust once $m\geq 2$ annotations per image: increasing from 2 to 51 annotations reduces soft-label ECE$_\text{true}$ (MCTS/SLTS) by less than 0.3~pp.  TS ECE$_\text{true}$, by contrast, \emph{increases} with $m$, because the annotator distribution diverges further from the one-hot voted label as the underlying consensus strengthens.  Full annotation-count results are reported in Appendix~\ref{app:annot-count}.

\begin{table}[t]
  \centering
  \caption{
    \textbf{Multi-seed stability.}
    Mean $\pm$ std of ECE$_\text{true}$ (\%) over 5 seeds.
    For CIFAR-10H/ChaosNLI the seed varies the cal/test split;
    for ISIC/DermaMNIST it varies the annotation sample.
  }
  \label{tab:multiseed}
  \scriptsize
  \setlength{\tabcolsep}{2.5pt}
  \resizebox{\linewidth}{!}{%
  \begin{tabular}{llcccc}
    \toprule
    Dataset & Arch & TS & SLTS & Dir.-Soft & IR-Soft \\
    \midrule
    CIFAR-10H   & R50        & $4.25\pm0.25$ & $1.61\pm0.12$ & $1.37\pm0.09$ & $0.89\pm0.12$ \\
    CIFAR-10H   & ViT-B/16   & $4.39\pm0.20$ & $0.83\pm0.14$ & $0.74\pm0.07$ & $0.88\pm0.16$ \\
    ChaosNLI    & RoBERTa-L  & $11.54\pm0.54$ & $3.85\pm0.34$ & $2.90\pm0.14$ & $2.40\pm0.18$ \\
    ChaosNLI    & DeBERTa-v3 & $12.73\pm0.60$ & $4.41\pm0.54$ & $3.80\pm0.35$ & $2.61\pm0.41$ \\
    ISIC 2019   & ENet-B4    & $18.84\pm0.17$ & $9.80\pm0.22$ & $8.45\pm0.20$ & $1.79\pm0.16$ \\
    ISIC 2019   & ViT-S/16   & $16.93\pm0.16$ & $7.30\pm0.13$ & $6.20\pm0.19$ & $2.05\pm0.13$ \\
    DermaMNIST  & R18        & $22.43\pm0.28$ & $5.08\pm0.25$ & $3.53\pm0.39$ & $2.33\pm0.17$ \\
    DermaMNIST  & ViT-S/16   & $24.81\pm0.08$ & $3.61\pm0.28$ & $3.11\pm0.21$ & $2.09\pm0.14$ \\
    \bottomrule
  \end{tabular}}
\end{table}

\begin{figure}[t]
  \centering
  \includegraphics[width=\linewidth]{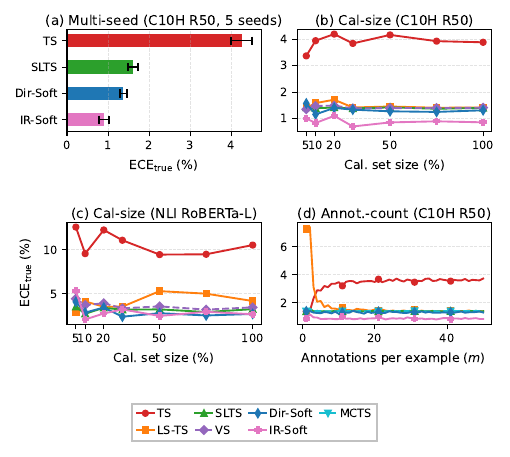}
  \caption{
    \textbf{Robustness of ambiguity-aware calibration.}
    \textbf{(a)}~Multi-seed stability (CIFAR-10H R50, 5 seeds).
    \textbf{(b)--(c)}~ECE$_\text{true}$ vs.\ calibration set size on CIFAR-10H R50 and ChaosNLI RoBERTa-L.
    \textbf{(d)}~ECE$_\text{true}$ vs.\ annotation count (CIFAR-10H R50).
    Rankings are stable across seeds, cal-set sizes, and annotation counts.
    Full results for all benchmarks are in Appendices~\ref{app:multiseed}, \ref{app:calset-size}, and~\ref{app:annot-count}.
  }
  \label{fig:robustness-main}
\end{figure}

% ─────────────────────────────────────────────────────────────────────────────
\section{Discussion}
\label{sec:discussion}
% ─────────────────────────────────────────────────────────────────────────────

\noindent\textbf{Trade-off between voted-label and true-label calibration.}
A natural question is whether optimising for true-label calibration degrades $\ECEvoted$.  This trade-off is expected: voted-label calibration becomes epistemically unreliable when annotators genuinely disagree, because the voted label systematically overstates the majority-class probability.  In practice, this trade-off is rarely problematic, since the scenarios where $\ECEvoted$ matters (deterministic, unambiguous labels) are precisely those where $\ECEvoted\approx\ECEtrue$ and the two objectives coincide.  For deployment contexts that require strict $\ECEvoted$ guarantees, a two-stage approach is possible: apply ambiguity-aware calibration first, then apply standard TS on the outputs to re-optimise $\ECEvoted$.

\noindent\textbf{Practical guidelines.}
The choice of calibration method depends on what annotation information is available at calibration time.  When the full annotator distribution is available, Dirichlet-Soft is recommended as the default: it consistently achieves the best overall calibration quality; SoftPlatt is a competitive alternative for NLI tasks, and IR-Soft should be preferred when the sole objective is minimising ECE.  When individual per-example annotations are stored but have not been aggregated into distributions, MCTS $S{=}1$ provides an effective solution: a single randomly drawn annotation per calibration example yields ECE comparable to the full distribution, as demonstrated consistently across all eight settings.  Finally, when only voted labels are available, LS-TS constructs a data-driven soft pseudo-target from the model's own confidence and reduces ECE by 9--77\% relative to Temperature Scaling without requiring any annotator data.

\noindent\textbf{Limitations.}
Full ambiguity-aware calibration (including Dirichlet-Soft) requires multi-annotator labels at calibration time.  When only voted labels are available, LS-TS remains useful but does not fully match Dirichlet-Soft, particularly on datasets where model confidence diverges from annotation ambiguity (e.g., ISIC~2019 ViT-S/16).  ISIC~2019 and DermaMNIST rely on synthetic annotators derived from clinician-reported agreement.  Crucially, this is a \emph{class-conditional} model: annotation confusion depends only on the consensus class label, not on instance-level visual difficulty.  Real multi-reader medical datasets with per-image reader distributions would provide stronger validation and would capture the instance-level variation our synthetic model does not; concrete candidates include VinDr-CXR \cite{nguyen2022vindr} (3 independent radiologists per chest X-ray, 15k images) and CheXpert \cite{irvin2019chexpert} (5 independent radiologist annotations on the test set), both of which provide true per-image annotator distributions directly amenable to our framework.  Conclusions about absolute ECE magnitudes on these two datasets should be interpreted with this in mind; the relative ranking of methods is expected to be robust to the specific confusion matrix used, as confirmed by the multi-seed annotation-generation ablation (Appendix~\ref{app:multiseed}).

% ─────────────────────────────────────────────────────────────────────────────
\section{Conclusion}
\label{sec:conclusion}
% ─────────────────────────────────────────────────────────────────────────────

This paper has formalised the problem of confidence calibration under ambiguous ground truth and established, both theoretically and empirically, that standard post-hoc calibrators fitted on voted labels are systematically miscalibrated against the underlying annotator distribution.  The central insight, confirmed by four complementary ablations, is that the calibration \emph{target}, not the method's architectural capacity, is the limiting factor.  Dirichlet calibration with voted-label targets consistently underperforms simple Temperature Scaling; per-instance temperature adaptation trained with voted-label supervision yields no improvement; yet the same global-temperature architecture, when supplied with the correct distributional target, substantially closes the gap.

Our proposed methods span three annotation regimes without requiring model retraining.  Dirichlet-Soft, which leverages full annotator distributions, reduces true-label ECE by 55--87\% relative to Temperature Scaling and achieves the best overall calibration quality across settings.  MCTS $S{=}1$ demonstrates that a single randomly drawn annotation per calibration example is sufficient to match full-distribution calibration within 0.6~pp ECE across all benchmarks, a finding with immediate practical implications for annotation protocols, as it suggests that annotating each calibration example once by a different randomly assigned annotator suffices.  LS-TS reduces ECE by 9--77\% using only voted labels.  Together, these results suggest that future calibration benchmarks and deployment standards should move beyond the voted-label paradigm and account for the distributional nature of human annotation.

\bibliographystyle{IEEEtran}
\bibliography{references}

% ─────────────────────────────────────────────────────────────────────────────
% Author Biographies
% To add a photo: replace figs/photo_placeholder.jpg with e.g. figs/photo_linwei.jpg
% ─────────────────────────────────────────────────────────────────────────────

\begin{IEEEbiography}[{\includegraphics[width=1in,height=1.25in,clip,keepaspectratio]{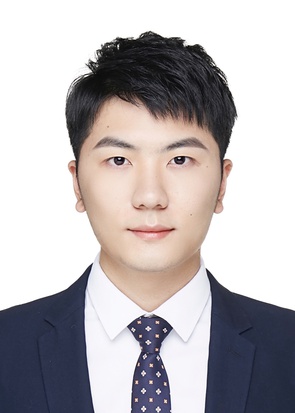}}]{Linwei Tao}
received the B.Eng.\ degree in communication engineering from Huazhong University of Science and Technology, China, in 2017, the M.Sc.\ degree in data science from the University of Sydney, Australia, in 2021, and the M.Phil.\ degree in computer vision from the University of Sydney in 2023.
He is currently pursuing the Ph.D.\ degree in computer vision at the University of Sydney, supervised by A/Prof.\ Chang Xu.
He is also a Research Scientist Intern at Google Research Australia.
His research interests include confidence calibration and uncertainty estimation in deep learning and large vision--language models.
He has published papers at top-tier venues including ICML, CVPR, ICLR, and AAAI.
He serves as a reviewer for NeurIPS, ICML, ICLR, CVPR, AAAI, and IJCAI, and as a journal reviewer for \textit{IEEE Transactions on Multimedia}, \textit{Transactions on Machine Learning Research}, \textit{Data Mining and Knowledge Discovery}, and \textit{IEEE Transactions on Pattern Analysis and Machine Intelligence}.
He is a Student Committee Member of AAAI 2026.
\end{IEEEbiography}
\begin{IEEEbiography}[{\includegraphics[width=1in,height=1.25in,clip,keepaspectratio]{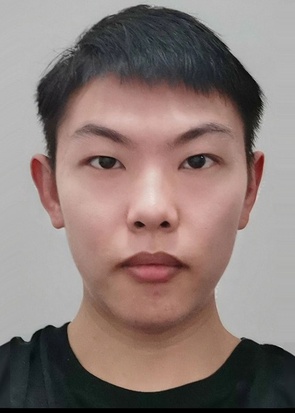}}]{Haoyang Luo}
received the B.S.\ degree (2021) and M.S.\ degree (2024) in computer science from the Harbin Institute of Technology, Shenzhen, China.
He is currently working toward the Ph.D.\ degree in computer science with the City University of Hong Kong, supervised by Dr.\ Minjing Dong.
His research interests lie in multi-modal trustworthy learning, uncertainty quantification, and model calibration.
He has published his research in top-tier conferences and journals, including NeurIPS, ICML, and \textit{IEEE Transactions on Knowledge and Data Engineering}.
\end{IEEEbiography}
\begin{IEEEbiography}[{\includegraphics[width=1in,height=1.25in,clip,keepaspectratio]{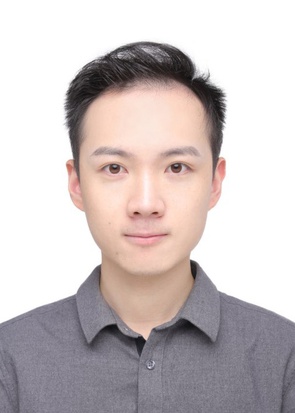}}]{Minjing Dong}
received the B.S.\ degree in software engineering from the Dalian University of Technology, the B.S.\ degree in information technology from the University of Sydney, the M.Phil.\ degree in engineering and information technology from the University of Sydney, and the Ph.D.\ degree in engineering and information technology from the University of Sydney.
He is currently an Assistant Professor with the Department of Computer Science, City University of Hong Kong.
His research interests include adversarial robustness, model calibration, efficient neural networks, human motion analytics, and generative models.
He has published more than 30 papers in top-tier conferences and journals, including NeurIPS, ICML, CVPR, AAAI, ICLR, \textit{IEEE Transactions on Pattern Analysis and Machine Intelligence}, \textit{IEEE Transactions on Neural Networks and Learning Systems}, \textit{IEEE Transactions on Image Processing}, and \textit{IEEE Transactions on Multimedia}.
He received the AAAI 2023 Distinguished Paper Award.
He serves as Area Chair and PC Member of NeurIPS, ICML, CVPR, ICCV, and ECCV.
\end{IEEEbiography}
\begin{IEEEbiography}[{\includegraphics[width=1in,height=1.25in,clip,keepaspectratio]{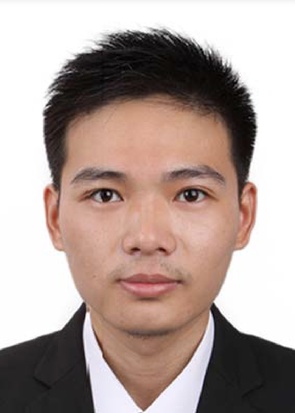}}]{Chang Xu}
(Senior Member, IEEE) received the Ph.D.\ degree from Peking University, China.
He is an ARC Future Fellow and Associate Professor with the School of Computer Science, University of Sydney.
He received the University of Sydney Vice-Chancellor's Award for Outstanding Early Career Research.
His research interests lie in machine learning algorithms and applications in computer vision.
He has published more than 100 papers in prestigious journals and top-tier conferences.
He has received several paper awards, including the Distinguished Paper Award at AAAI 2023 and the Distinguished Paper Award at IJCAI 2018.
He served as Area Chair of NeurIPS, ICML, ICLR, KDD, CVPR, and MM, as well as Senior PC Member of AAAI and IJCAI.
He served as Associate Editor of \textit{IEEE Transactions on Pattern Analysis and Machine Intelligence}, \textit{IEEE Transactions on Multimedia}, and \textit{Transactions on Machine Learning Research}.
He has been named a Top Ten Distinguished Senior PC Member at IJCAI 2017 and an Outstanding Associate Editor of \textit{IEEE Transactions on Multimedia} in 2022.
He is the corresponding author of this paper.
\end{IEEEbiography}

\clearpage

% ─────────────────────────────────────────────────────────────────────────────
\appendices

% ─────────────────────────────────────────────────────────────────────────────

\section{Dataset Background and Sources of Label Ambiguity}
\label{app:dataset-background}

All four benchmarks in this paper are appropriate for studying calibration under ambiguous ground truth, but they exhibit ambiguity in slightly different ways.

\noindent\textbf{CIFAR-10H.}
Peterson et al.~\cite{peterson2019human} introduced CIFAR-10H by collecting a full distribution of human labels for each image in the CIFAR-10 test set, explicitly to capture \emph{human perceptual uncertainty}.  The ambiguous examples are typically low-resolution images whose visual evidence supports multiple plausible classes for human observers.  In this benchmark, the label distribution is therefore directly observed from repeated human annotation, making it a clean testbed for ambiguity-aware calibration.

\noindent\textbf{ChaosNLI.}
Nie et al.~\cite{nie2020what} introduced ChaosNLI to study \emph{collective human opinions} on natural language inference.  Each example receives 100 human labels, and the paper reports that substantial disagreement is common rather than exceptional.  This ambiguity is semantic rather than perceptual: differences in pragmatic assumptions, underspecification, and sentence interpretation lead reasonable annotators to assign different NLI labels to the same premise--hypothesis pair.

\noindent\textbf{ISIC~2019.}
ISIC~2019 is a multiclass dermoscopic lesion diagnosis benchmark assembled from several clinical sources, including HAM10000 \cite{tschandl2018ham10000} and BCN20000 \cite{combalia2019bcn}.  The task itself is ambiguity-prone because several lesion categories share overlapping visual morphology, and the challenge was designed to reflect the more realistic task of differential diagnosis rather than a simpler benign-versus-malignant decision.  Because the public challenge release does not provide per-image reader distributions, we model ambiguity using synthetic annotators calibrated to dermatologist confusion patterns and reader agreement reported in Liu et al.~\cite{liu2020deep}.  This preserves the core property relevant to our study: in dermatology, a single image can support multiple clinically plausible labels.

\noindent\textbf{DermaMNIST.}
DermaMNIST \cite{yang2023medmnist} is a standardized $28\times28$ version of the dermatology benchmark derived from HAM10000 \cite{tschandl2018ham10000}.  HAM10000 consists of dermatoscopic images of common pigmented lesions collected in clinical practice, where diagnosis already involves fine-grained distinctions among visually similar categories such as melanoma, nevi, and benign keratoses.  DermaMNIST inherits this medical ambiguity from HAM10000 and, because it is aggressively downsampled for lightweight benchmarking, removes additional visual detail that can further increase label uncertainty.  As with ISIC~2019, we therefore evaluate calibration against an ambiguity-aware label distribution rather than treating the majority diagnosis as uniquely correct.

\section{Proof of Proposition~\ref{prop:ts-bias}}
\label{app:proof-ts}

Let $\mathcal{D}_\text{amb}=\{(x_i,y_i^*): y_i^*=c,\, x_i\in\mathcal{C}\}$ be the subset of calibration examples from the ambiguous cluster $\mathcal{C}$, where the voted label is always class $c$.  Denote $p_i=\operatorname{softmax}(z_i)_c$.

The TS loss on $\mathcal{D}_\text{amb}$ is $\mathcal{L}_{\text{TS}}(T;\mathcal{D}_\text{amb}) = -|\mathcal{D}_\text{amb}|^{-1}\sum_{i}\log \operatorname{softmax}(z_i/T)_c$.  As $T\to 0^+$, $\operatorname{softmax}(z_i/T)_c\to 1$ (assuming $z_{ic}>z_{ik}$ for all $k\ne c$), so $\mathcal{L}_{\text{TS}}\to 0$.  Hence $\partial\mathcal{L}_{\text{TS}}/\partial T>0$ for all $T>0$: decreasing $T$ always decreases the voted-label loss on the ambiguous cluster.  The globally optimal $T^*_{\text{TS}}$ balances this downward pull from ambiguous examples against the upward pull from unambiguous examples, but $T^*_{\text{TS}}\le T_{\text{no-amb}}^*$ (the optimal temperature without the ambiguous cluster): including ambiguous examples drives the optimum downward.

The cross-entropy against the annotator distribution on $\mathcal{D}_\text{amb}$ is $\mathcal{L}_{\text{soft}}(T;\mathcal{D}_\text{amb}) = -|\mathcal{D}_\text{amb}|^{-1}\sum_i\sum_k\pihat_k(x_i)\log\operatorname{softmax}(z_i/T)_k$, with $\pihat_c(x_i)=q<1$.  The minimum-loss $T^*_{\text{soft}}$ satisfies $\operatorname{softmax}(z_i/T^*_{\text{soft}})_c\approx q$, requiring $T^*_{\text{soft}}>T^*_{\text{TS}}$ whenever $p_i>q$.  Hence $T^*_{\text{TS}} < T^*_{\text{soft}}$ regardless of whether either optimum exceeds 1.  In the special case where the model is already overconfident on the ambiguous cluster ($p_i \gg q$), this gap is large enough that $T^*_{\text{TS}}<1$; in practice both optima exceed 1 (since the model is also underconfident on unambiguous examples), but the inequality $T^*_{\text{TS}} < T^*_{\text{soft}}$ is observed consistently across all experimental settings.  $\square$

\section{Proof of Proposition~\ref{prop:entropy-gap}}
\label{app:proof-entropy}

We formalise the argument given in the proof sketch in the main text.

\noindent\textbf{Setup.}
Let $p=\phat_{\hat{c}}(x)$ denote the model's predicted confidence for the top class $\hat{c}$ and let $q=\pi_{\hat{c}}(x)\in(0,1]$ denote the true annotator probability for that class.  We assume the model is trained against voted labels, so $p$ is approximately constant across examples with the same voted label regardless of their annotation entropy $H(x)$.

\noindent\textbf{Voted-label calibration error contribution.}
For example $x_i$ in confidence bin $B_b$ (where $\overline{p}(B_b)\approx p$), the per-example voted-label calibration error is $|p - \mathbf{1}[\hat{c}=y^*]|$.  Since $\hat{c}=y^*$ for correctly-classified ambiguous examples (the model predicts the majority class), this contribution is approximately $|p-1|$ for examples where the model makes the right hard prediction, independent of $H(x)$.

\noindent\textbf{True-label calibration error contribution.}
The per-example true-label calibration error is $|p - \pi_{\hat{c}}(x)|=|p-q|$.

\noindent\textbf{Excess true-label error relative to voted-label evaluation.}
The per-example excess contribution is:
\begin{align}
  \delta(x) &= |p - q| - |p - \mathbf{1}[\hat{c}=y^*]|.
\end{align}
When the model is correctly calibrated to voted labels ($p\approx\mathbf{1}[\hat{c}=y^*]$ in expectation), the voted contribution is near zero: $|p-1|\approx 0$ for high-accuracy bins and $|p-0|=p$ for error bins.  For unambiguous examples ($H(x)=0$), $q=\pi_{y^*}=1$, so $\delta(x)=|p-1|-|p-1|=0$.

\noindent\textbf{Monotonicity.}
As $H(x)$ increases, $q=\pi_{y^*}(x)$ decreases below 1 (more probability mass shifts to non-majority classes).  Specifically, for a $K$-class problem with uniform annotator distribution at maximum entropy, $q=1/K$.  The true-label contribution $|p-q|$ increases strictly as $q$ decreases from 1 (since $p>q$ for a well-trained model predicting the majority class with confidence exceeding the majority probability).  The voted-label contribution $|p-\mathbf{1}[\hat{c}=y^*]|$ is unaffected by $H(x)$.  Therefore $\delta(x)=|p-q|-|p-\mathbf{1}[\hat{c}=y^*]|$ is non-decreasing in $H(x)$, and strictly increasing whenever $p>q>0$.

\noindent\textbf{Aggregate monotonicity.}
The aggregate excess true-label error relative to voted-label evaluation equals $\sum_b\frac{|B_b|}{n}\sum_{i\in B_b}\delta(x_i)/|B_b|$.  Since $\delta(x_i)$ is non-decreasing in $H(x_i)$ and bins contain a mixture of examples, this discrepancy grows with the mean annotation entropy of the test set.  $\square$

\begin{remark}[Independence assumption]
Assumption~(i) (that $\phat_{\hat{c}}(x)$ is approximately independent of $H(x)$) may be violated in practice if harder examples (high $H(x)$) also have lower model confidence.  When the two are positively correlated, the per-example discrepancy term $\delta(x)$ grows even faster with $H(x)$ than the proof suggests, so Proposition~\ref{prop:entropy-gap} remains valid and the monotonicity is if anything \emph{conservative}.
\end{remark}

\section{Monte Carlo Temperature Scaling (MCTS)}
\label{app:mcts-convergence}

\noindent\textbf{Method.}
When individual annotation records are available rather than pre-aggregated label distributions, \textbf{MCTS} draws $S$ samples $a_{is}\sim\pihat(x_i)$ per example and minimises
\[
  \mathcal{L}_{\text{MCTS}}(T) = -\frac{1}{nS}\sum_{i=1}^n\sum_{s=1}^S \log\operatorname{softmax}(z_i/T)_{a_{is}}.
\]
MCTS subsumes the deterministic soft-label objective (SLTS, Eq.~\ref{eq:slts-loss} with $S\!\to\!\infty$) as a special case, but applies even when only raw annotation samples are stored rather than pre-computed frequency vectors.

\noindent\textbf{A single annotation suffices.}
The key practical implication is that \emph{even $S=1$} (a single randomly drawn annotation per calibration example) already achieves ECE$_\text{true}$ of 1.53\%, matching the deterministic limit ($S\!\to\!\infty$) to within rounding error and reducing TS's 4.29\% by $64\%$ on CIFAR-10H ResNet-50.  This one-annotation efficiency generalises across all benchmarks: across the 8 settings in Tables~\ref{tab:cifar} and~\ref{tab:isic-main}, MCTS $S{=}1$ matches SLTS within 0.6~pp ECE in every case.  On the natural-annotation benchmarks (CIFAR-10H, ChaosNLI), MCTS $S{=}1$ occasionally \emph{outperforms} the deterministic limit slightly (e.g., 2.82\% vs.\ 3.22\% ECE on ChaosNLI RoBERTa-L), consistent with the stochastic single-annotation objective providing mild implicit regularisation.  This means that any dataset where each example has been annotated even once by a randomly selected annotator is immediately amenable to strong ambiguity-aware calibration; pre-aggregating annotations into a distribution provides no measurable additional benefit.  Formally, since $\bE_{\hat{y}\sim\pihat}[\mathrm{CE}(\operatorname{softmax}(z/T),\hat{y})] = \KL(\pihat\|\operatorname{softmax}(z/T)) + H(\pihat)$, the MCTS objective converges to Eq.~\eqref{eq:slts-loss} in expectation as $S\to\infty$, with variance $\propto 1/S$.  Table~\ref{tab:mcts} (main text) confirms the fast per-seed convergence on CIFAR-10H ResNet-50: the ECE gap between $S=1$ and $S\to\infty$ is only $0.02$ pp; in practice $S=20$--$50$ is sufficient to reduce variance to negligible levels.

\section{Ablation: Calibration Set Size}
\label{app:calset-size}

We fix the annotation count at the full observed value and vary the fraction of calibration examples used, from 5\% to 100\% (stratified subsampling by class).  Results are shown for all four benchmarks and eight architectures.

\begin{figure*}[t]
  \centering
  \includegraphics[width=0.49\linewidth]{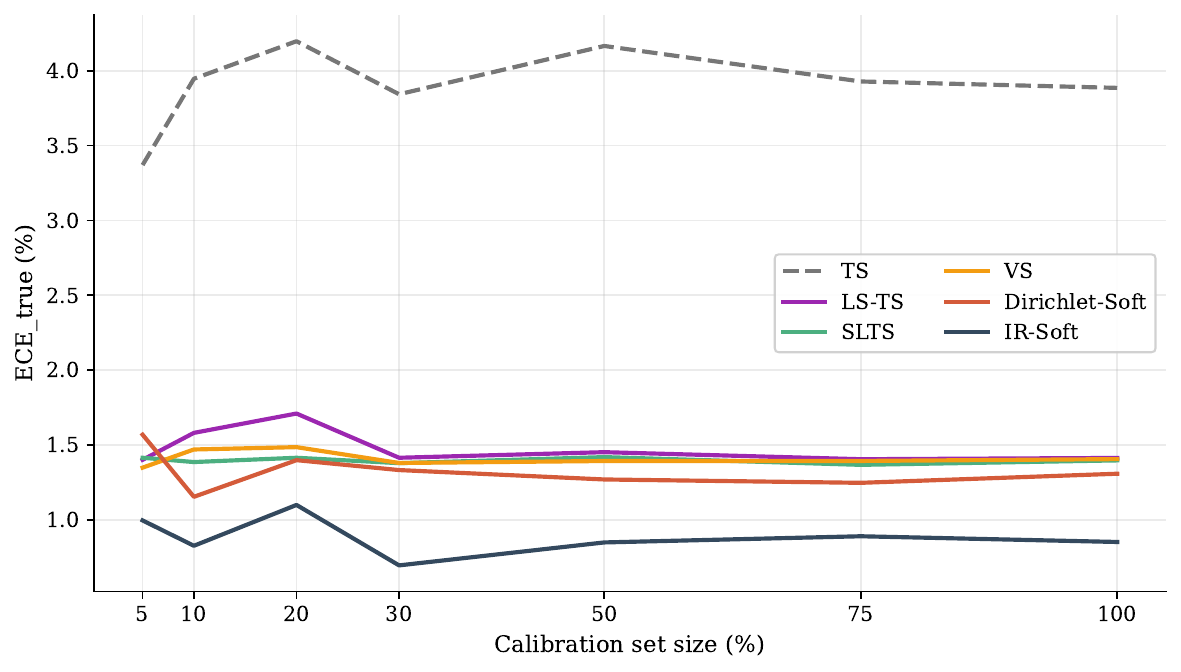}%
  \hfill%
  \includegraphics[width=0.49\linewidth]{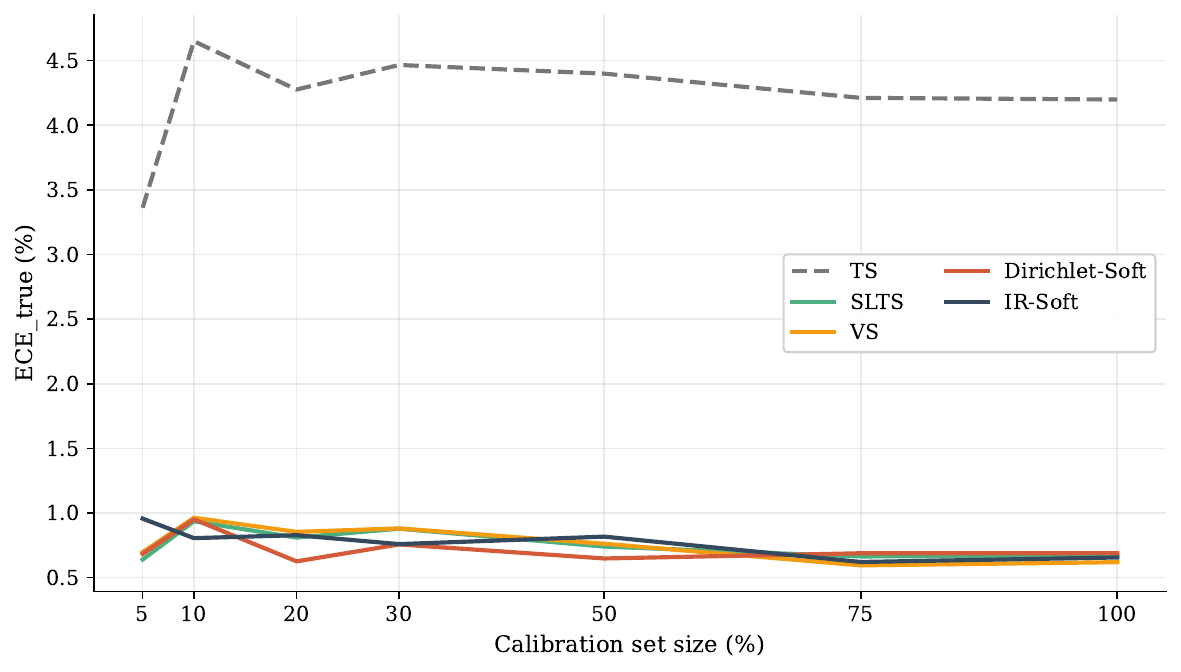}
  \caption{
    \textbf{Calibration-set-size ablation: CIFAR-10H.}
    Left: ResNet-50.  Right: ViT-B/16.
    TS's ECE is flat across all calibration set sizes; ambiguity-aware methods converge within 5--10\%.
  }
  \label{fig:calsize-cifar}
\end{figure*}

\begin{figure*}[t]
  \centering
  \includegraphics[width=0.49\linewidth]{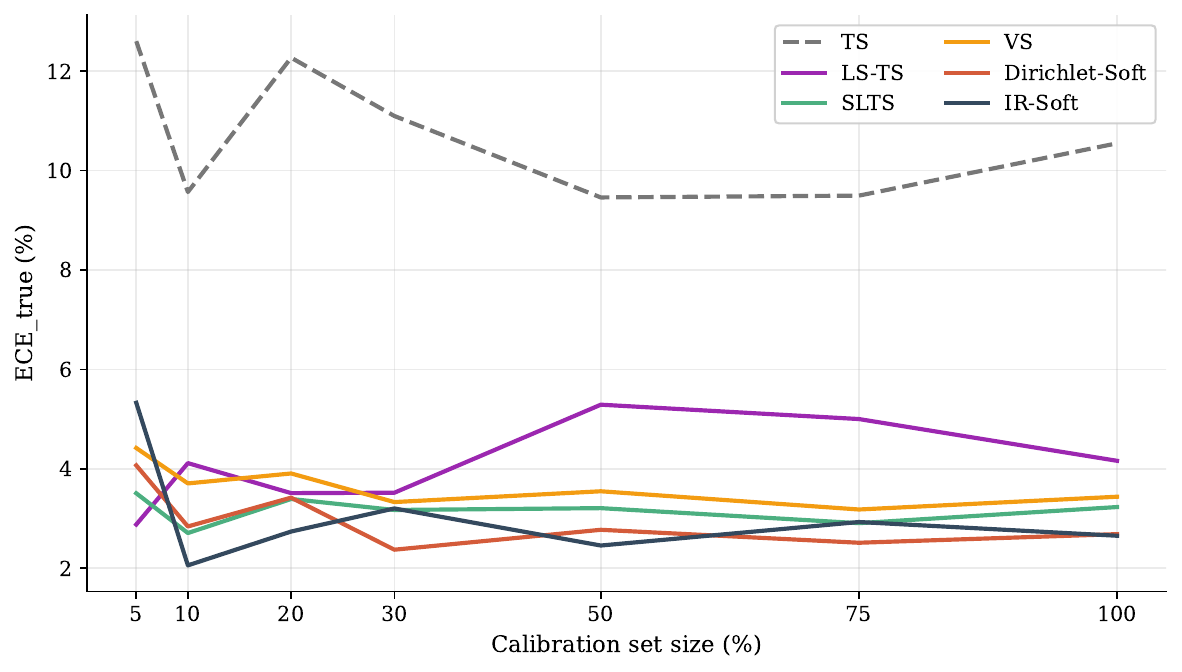}%
  \hfill%
  \includegraphics[width=0.49\linewidth]{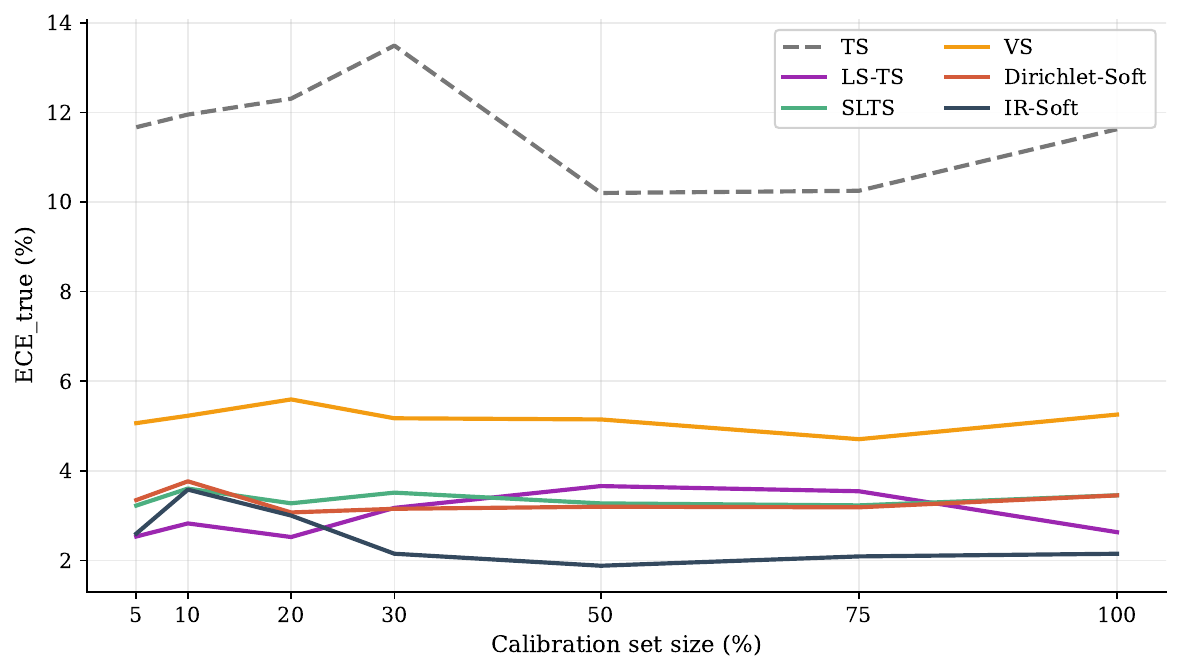}
  \caption{
    \textbf{Calibration-set-size ablation: ChaosNLI.}
    Left: RoBERTa-Large.  Right: DeBERTa-v3-base.
    NLI models show the same pattern: TS is flat near 10--13\% while soft methods stabilise near 3--4\%.
  }
  \label{fig:calsize-chaosnli}
\end{figure*}

\begin{figure*}[t]
  \centering
  \includegraphics[width=0.49\linewidth]{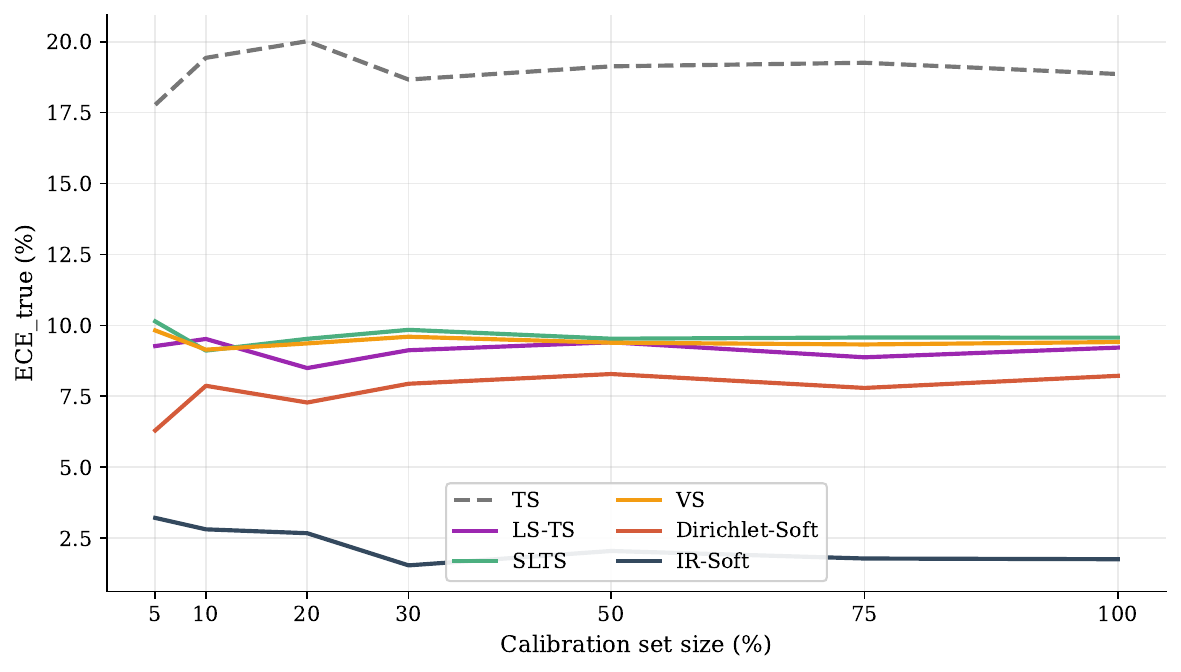}%
  \hfill%
  \includegraphics[width=0.49\linewidth]{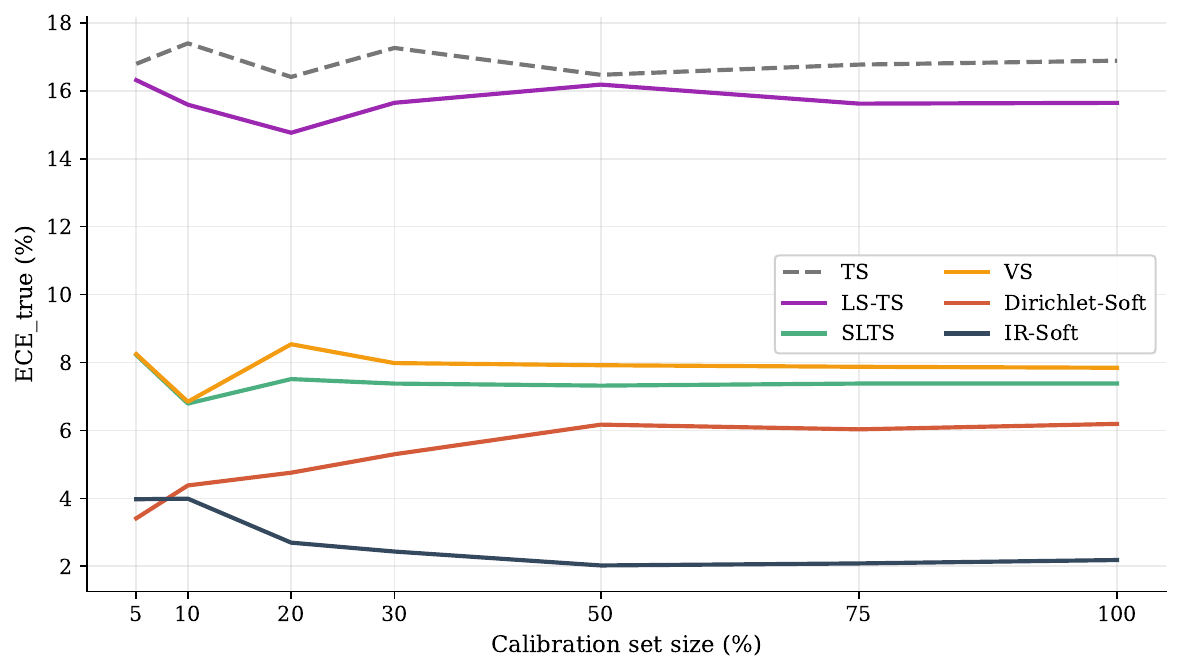}
  \caption{
    \textbf{Calibration-set-size ablation: ISIC 2019.}
    Left: EfficientNet-B4.  Right: ViT-S/16.
    IR-Soft achieves the lowest ECE across all calibration set sizes on both architectures.
  }
  \label{fig:calsize-isic}
\end{figure*}

\begin{figure*}[t]
  \centering
  \includegraphics[width=0.49\linewidth]{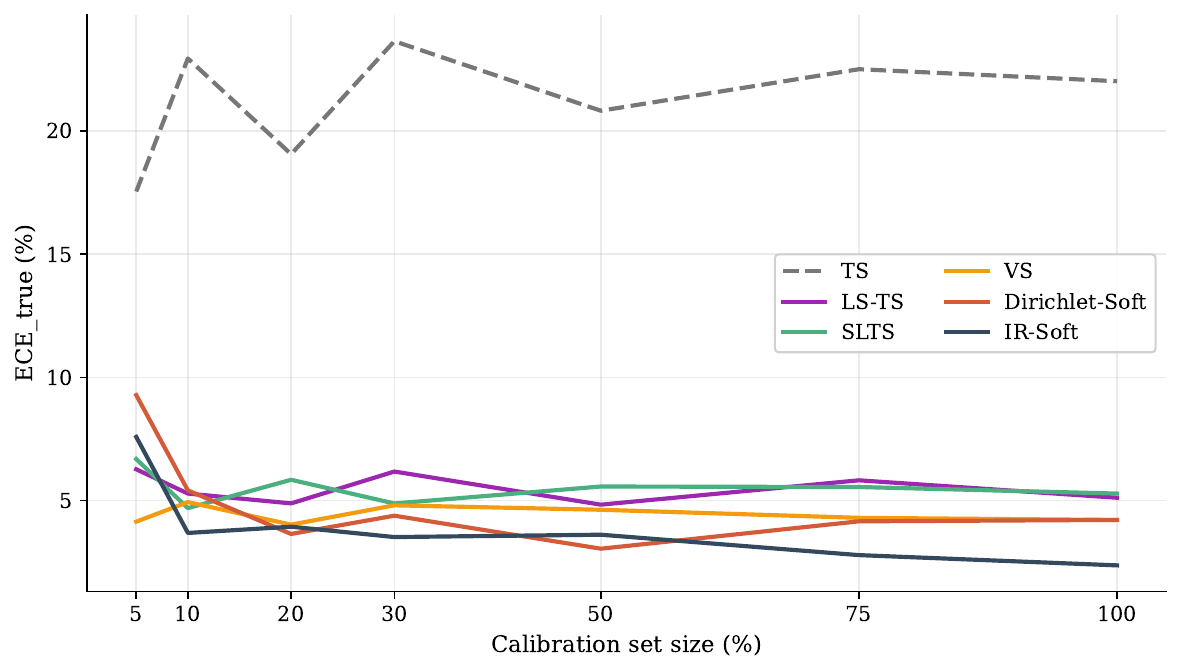}%
  \hfill%
  \includegraphics[width=0.49\linewidth]{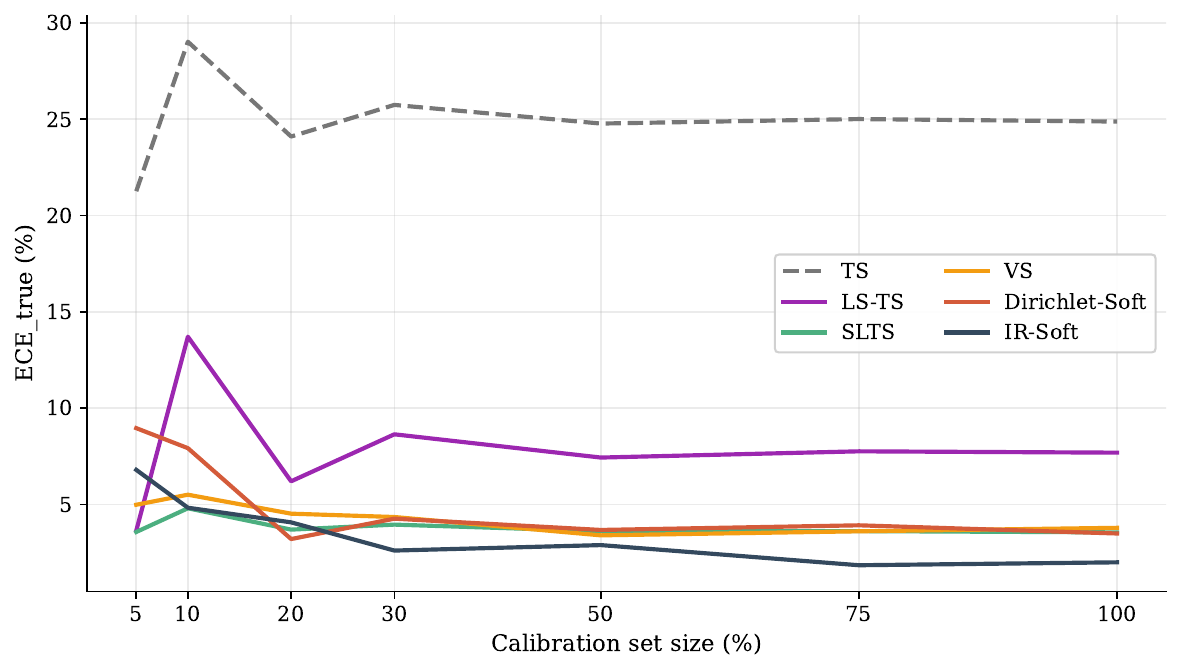}
  \caption{
    \textbf{Calibration-set-size ablation: DermaMNIST.}
    Left: ResNet-18.  Right: ViT-S/16.
    Ambiguity-aware methods converge quickly; TS remains flat regardless of calibration set size.
  }
  \label{fig:calsize-derm}
\end{figure*}

\noindent\textbf{Key findings.}
Across all four benchmarks and eight architectures, TS's ECE does not decrease with more calibration data, confirming that the degradation is a target bias rather than a variance artifact.  In contrast, soft-label methods (MCTS/SLTS) achieve stable low ECE from as few as 5--10\% of the calibration set, and IR-Soft and Dirichlet-Soft converge even faster on most benchmarks.  This suggests that the annotation distribution (not the number of calibration examples) is the limiting factor for ambiguity-aware methods.

\section{Ablation: Number of Annotations per Image}
\label{app:annot-count}

We study how performance depends on the number of annotations $m$ per calibration image.
We fix the calibration set size and the annotation generation seed, and sweep $m$ from $1$ to $50$, resampling the soft targets for each value of $m$.
Results are shown for CIFAR-10H, which provides up to ${\approx}51$ real annotations per image; we subsample from the full pool.

\begin{figure*}[t]
  \centering
  \includegraphics[width=0.49\linewidth]{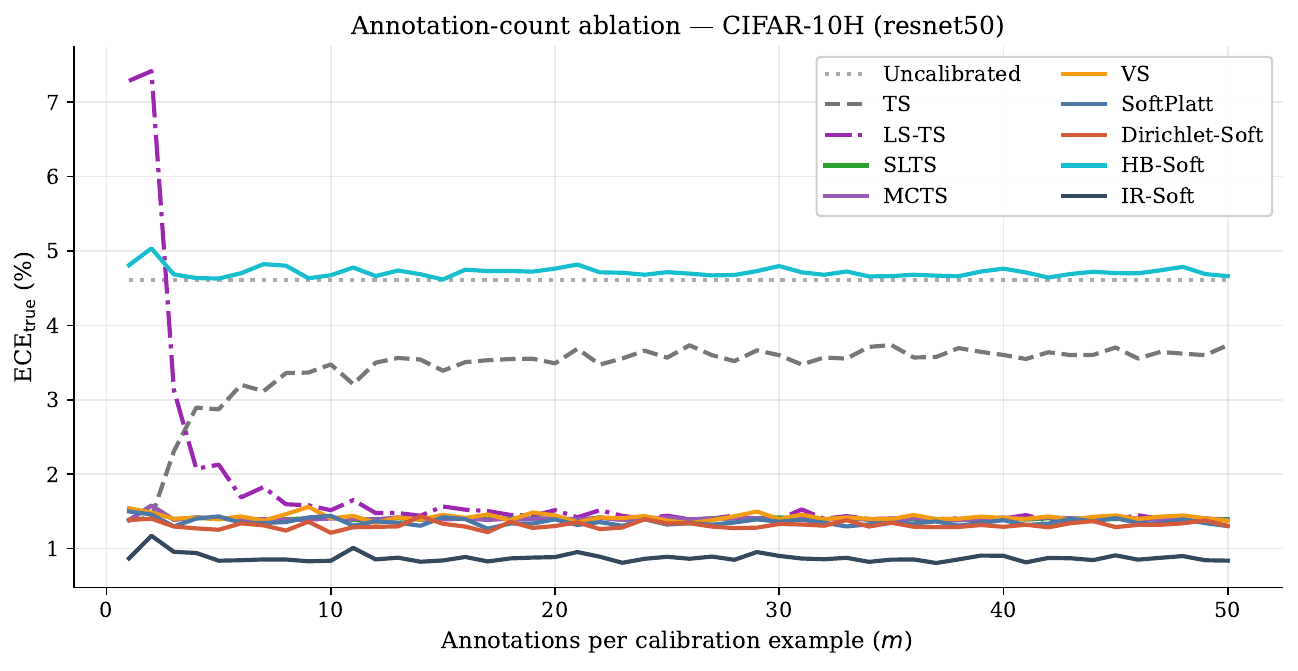}%
  \hfill%
  \includegraphics[width=0.49\linewidth]{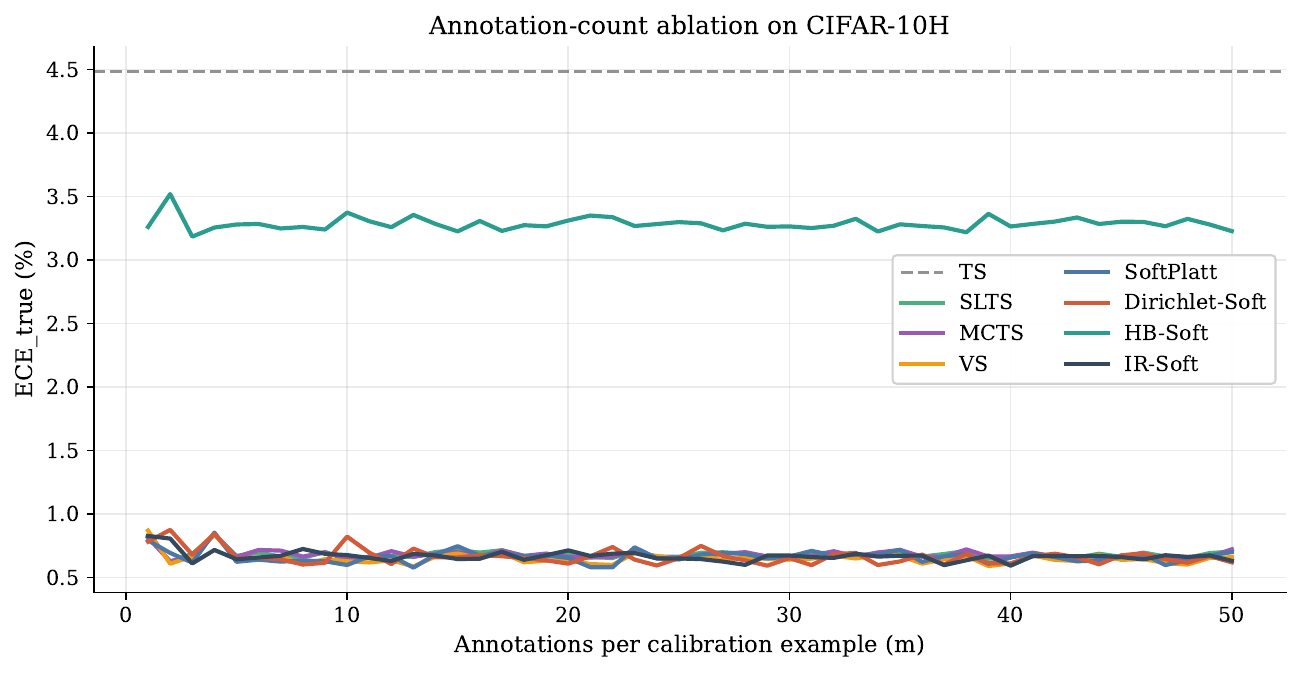}
  \caption{
    \textbf{Annotation-count ablation: CIFAR-10H (both architectures).}
    ECE$_\text{true}$ vs.\ annotations per image $m$ for ResNet-50 (left) and ViT-B/16 (right).
    Ambiguity-aware methods are stable for $m\geq 2$; TS rises with $m$ due to target mismatch; LS-TS converges by $m\approx 4$.
  }
  \label{fig:annot-count-cifar-app}
\end{figure*}

\noindent\textbf{Key findings.}
Ambiguity-aware calibration is robust to annotation count once $m\geq 2$: SLTS, Dirichlet-Soft, and IR-Soft achieve near-minimal ECE with only a handful of annotations per image and do not improve substantially beyond $m=5$.  LS-TS converges quickly (by $m\approx 4$) because the global smoothing parameter $\varepsilon$ stabilises once each image has at least a few annotations.  Standard Temperature Scaling shows the opposite trend: its ECE$_\text{true}$ \emph{increases} with $m$ because a higher temperature is required to match the more diffuse annotation distribution, exposing the fundamental target mismatch.

\section{Statistical Significance: Multi-Seed Analysis}
\label{app:multiseed}

For CIFAR-10H and ChaosNLI we repeat the evaluation with five random seeds (42--46), each producing a different stratified 50/50 cal/test split.  For ISIC~2019 and DermaMNIST the val/test split is predetermined, so we instead vary the annotation generation seed (42--46), testing sensitivity to the particular synthetic annotation sample.  Table~\ref{tab:multiseed} (main text, Robustness section) reports mean $\pm$ std of ECE$_\text{true}$ across seeds for all eight (dataset, arch) combinations.

Across all benchmarks and architectures, standard deviations are small relative to the ECE gaps: TS vs.\ ambiguity-aware methods differ by $2$--$10\times$ the standard deviation, confirming statistical robustness of the results in Tables~\ref{tab:cifar}--\ref{tab:derm}.  Bar charts for all architectures are shown below.

\begin{figure*}[t]
  \centering
  \includegraphics[width=0.49\linewidth]{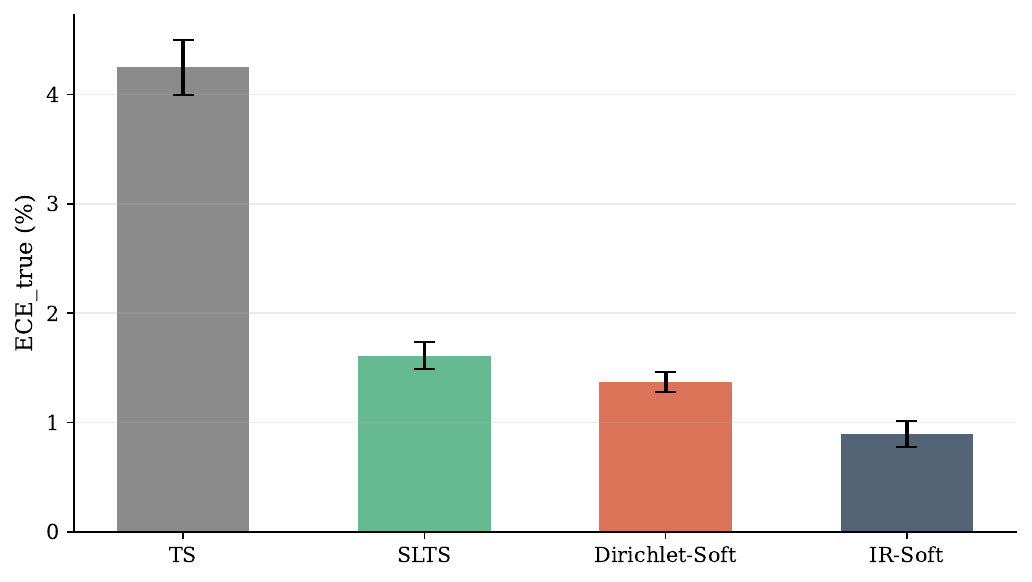}%
  \hfill%
  \includegraphics[width=0.49\linewidth]{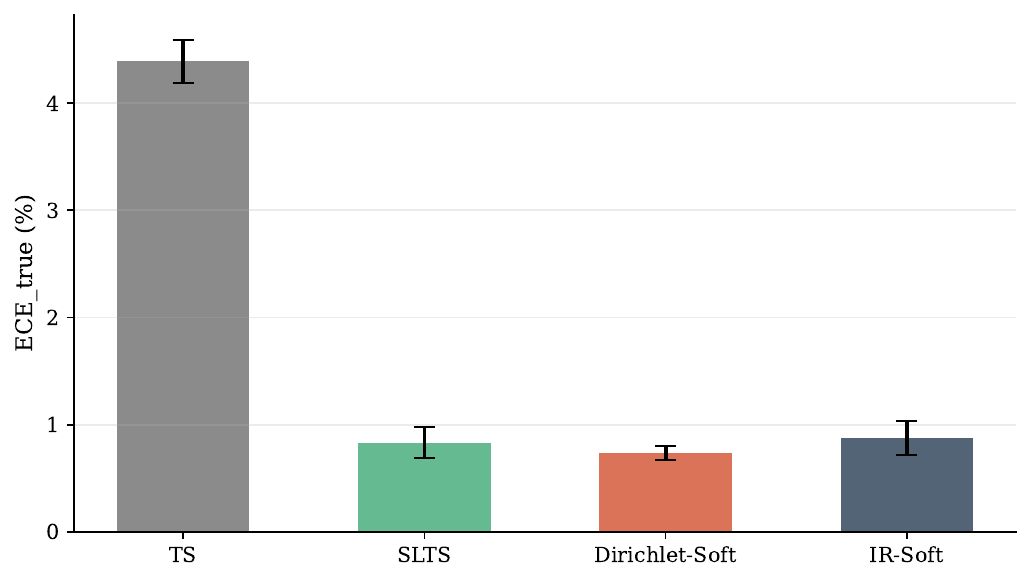}
  \caption{
    \textbf{Multi-seed bar charts: CIFAR-10H.}
    Left: ResNet-50.  Right: ViT-B/16.
  }
  \label{fig:multiseed-cifar}
\end{figure*}

\begin{figure*}[t]
  \centering
  \includegraphics[width=0.49\linewidth]{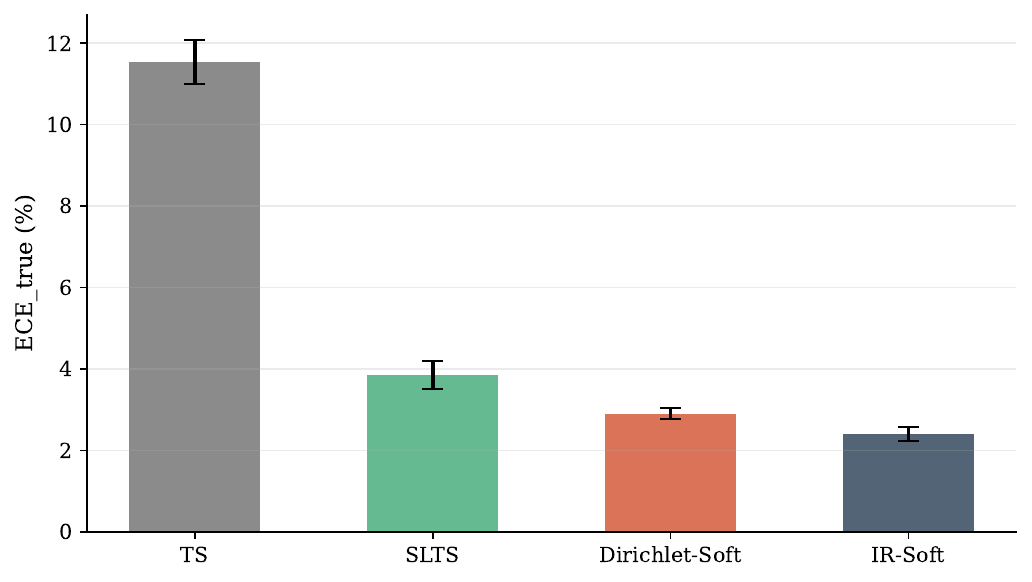}%
  \hfill%
  \includegraphics[width=0.49\linewidth]{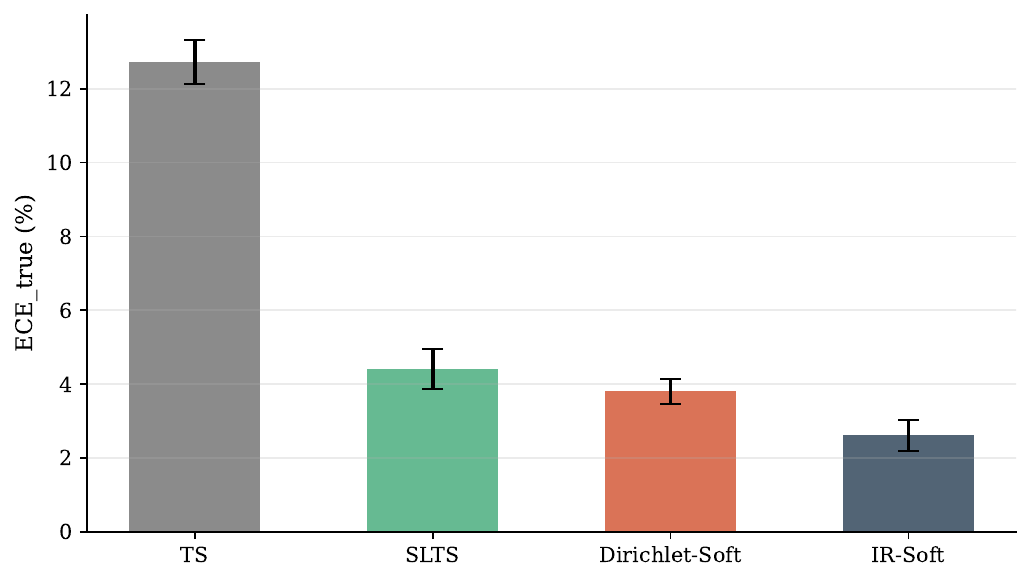}
  \caption{
    \textbf{Multi-seed bar charts: ChaosNLI.}
    Left: RoBERTa-Large.  Right: DeBERTa-v3-base.
  }
  \label{fig:multiseed-chaosnli}
\end{figure*}

\begin{figure*}[t]
  \centering
  \includegraphics[width=0.49\linewidth]{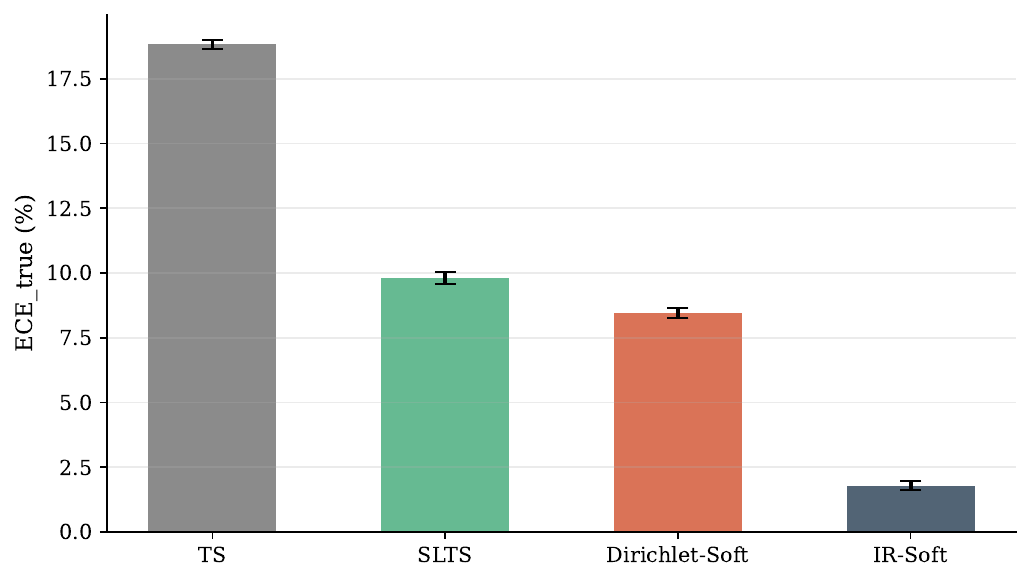}%
  \hfill%
  \includegraphics[width=0.49\linewidth]{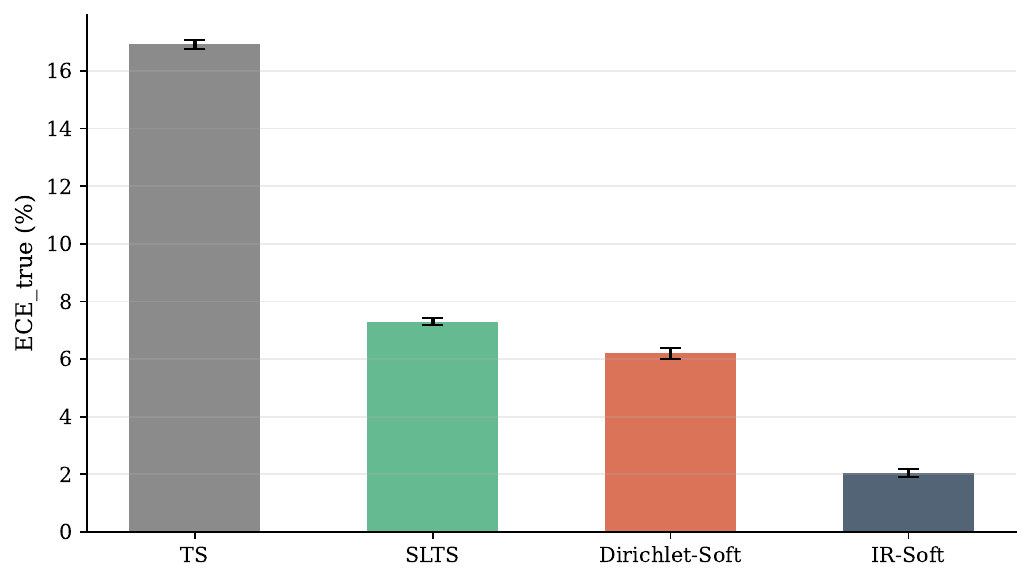}
  \caption{
    \textbf{Multi-seed bar charts: ISIC 2019.}
    Left: EfficientNet-B4.  Right: ViT-S/16.
    Seed varies the annotation generation; val/test split is fixed.
  }
  \label{fig:multiseed-isic}
\end{figure*}

\begin{figure*}[t]
  \centering
  \includegraphics[width=0.49\linewidth]{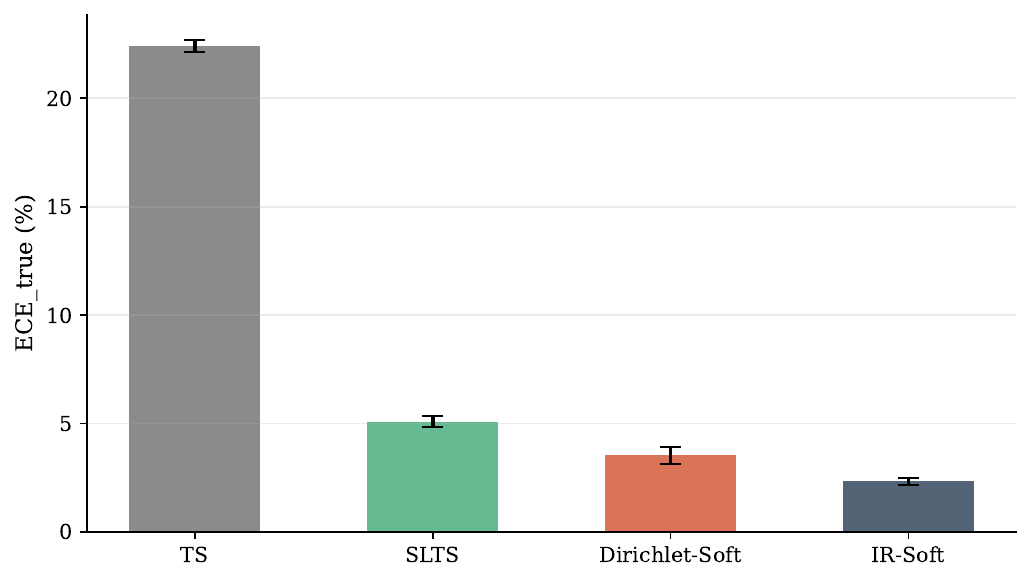}%
  \hfill%
  \includegraphics[width=0.49\linewidth]{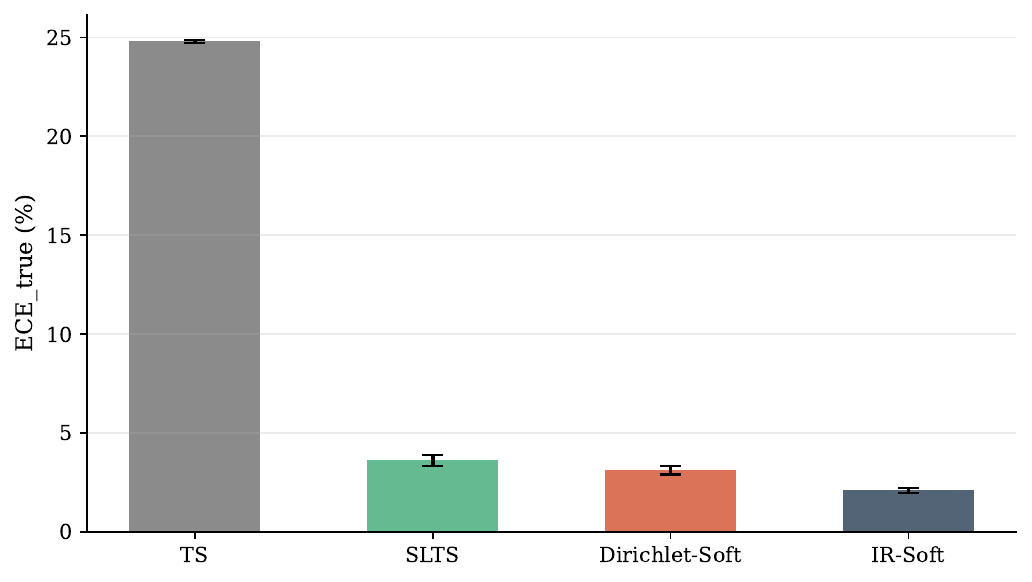}
  \caption{
    \textbf{Multi-seed bar charts: DermaMNIST.}
    Left: ResNet-18.  Right: ViT-S/16.
    Seed varies the annotation generation; val/test split is fixed.
  }
  \label{fig:multiseed-derm}
\end{figure*}

% ─────────────────────────────────────────────────────────────────────────────
\section{Reliability Diagrams}
\label{app:reliability}
% ─────────────────────────────────────────────────────────────────────────────

Four-method summaries for CIFAR-10H (ResNet-50) and ChaosNLI (RoBERTa-Large) are shown in Figures~\ref{fig:reliability-main} and~\ref{fig:reliability-chaosnli-main} in the main text.  Here we provide the four-method summary panels and full panels for all remaining architectures and datasets.  All panels share the same layout: top row = voted-label baselines (Uncal, TS, Platt, Dir.-Hard); middle row = annotation-free (LS-TS) and temperature-based soft methods (SLTS, MCTS, SoftPlatt, VS); bottom row = non-parametric/matrix soft methods (Dir.-Soft, IR-Soft).  Red shading = overconfident; green = underconfident.

\begin{figure*}[t]
  \centering
  \includegraphics[width=\linewidth]{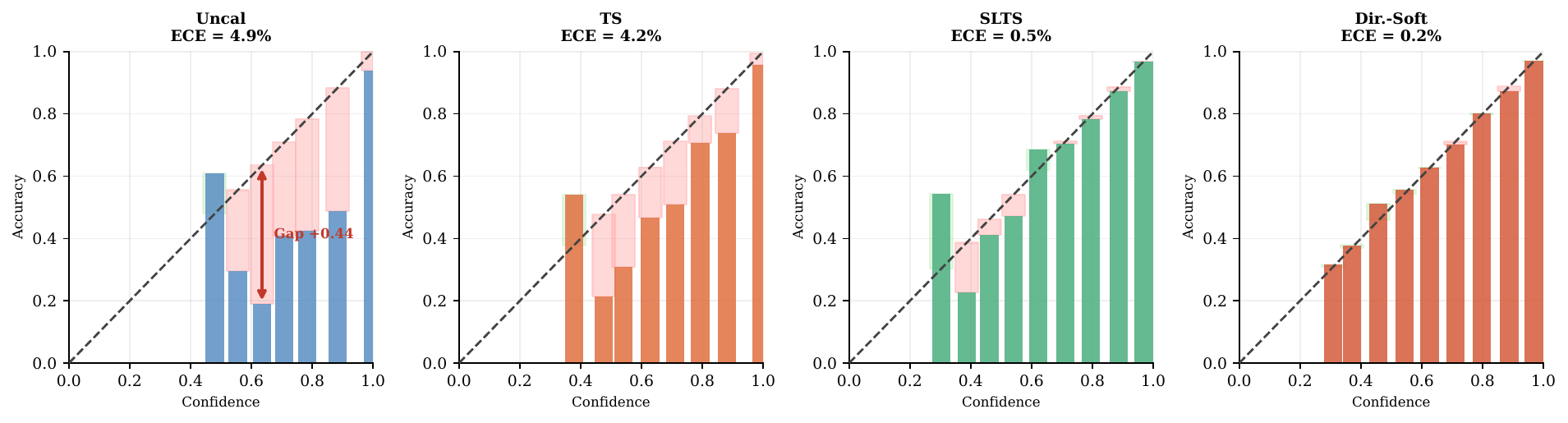}
  \caption{\textbf{Reliability diagrams (summary): CIFAR-10H ViT-B/16.}}
  \label{fig:reliability-cifar-vit-sum}
\end{figure*}
\begin{figure*}[t]
  \centering
  \includegraphics[width=\linewidth]{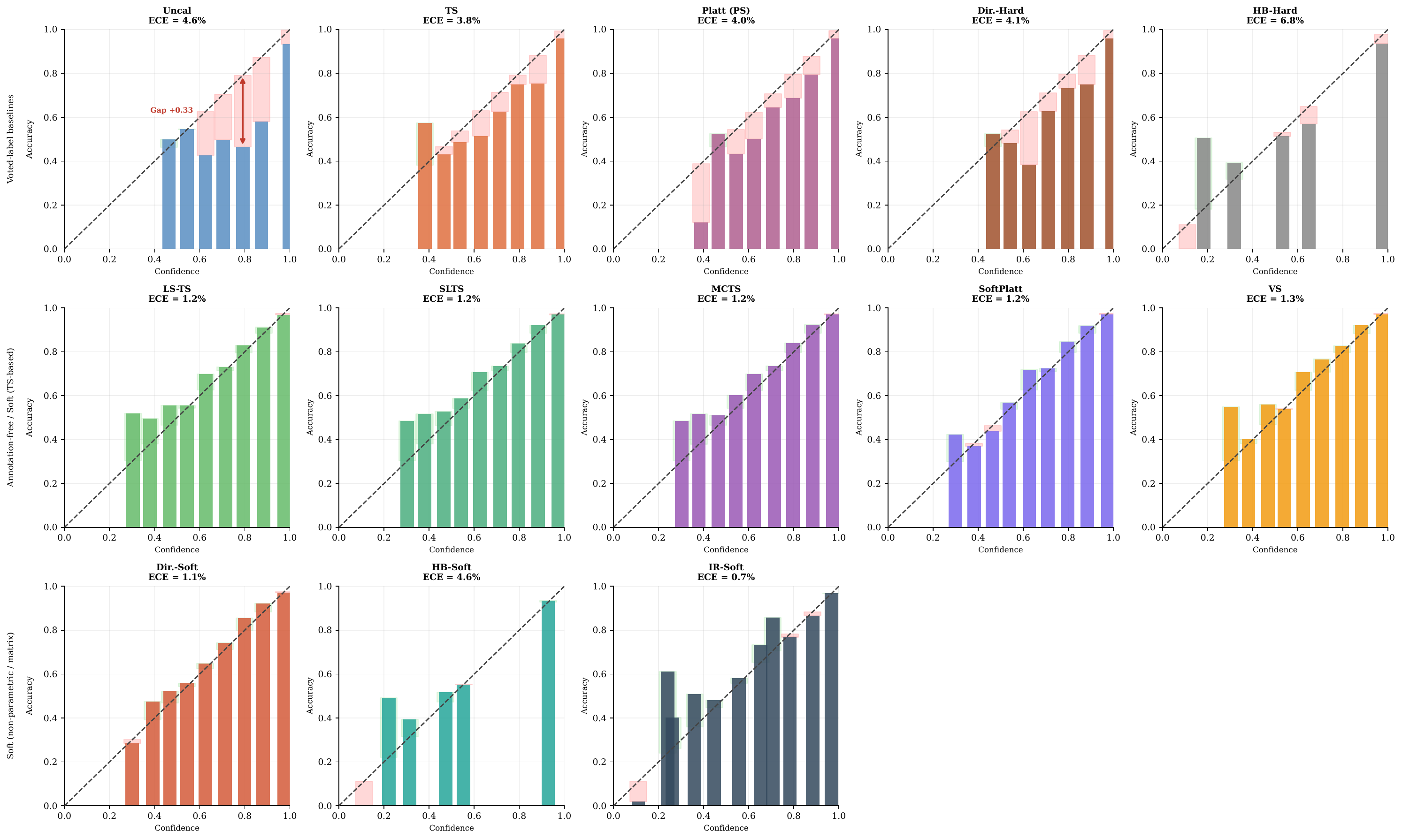}
  \caption{\textbf{Reliability diagrams (all methods): CIFAR-10H ResNet-50 (ECE$_\text{true}$).}}
  \label{fig:reliability-full-cifar}
\end{figure*}
\begin{figure*}[t]
  \centering
  \includegraphics[width=\linewidth]{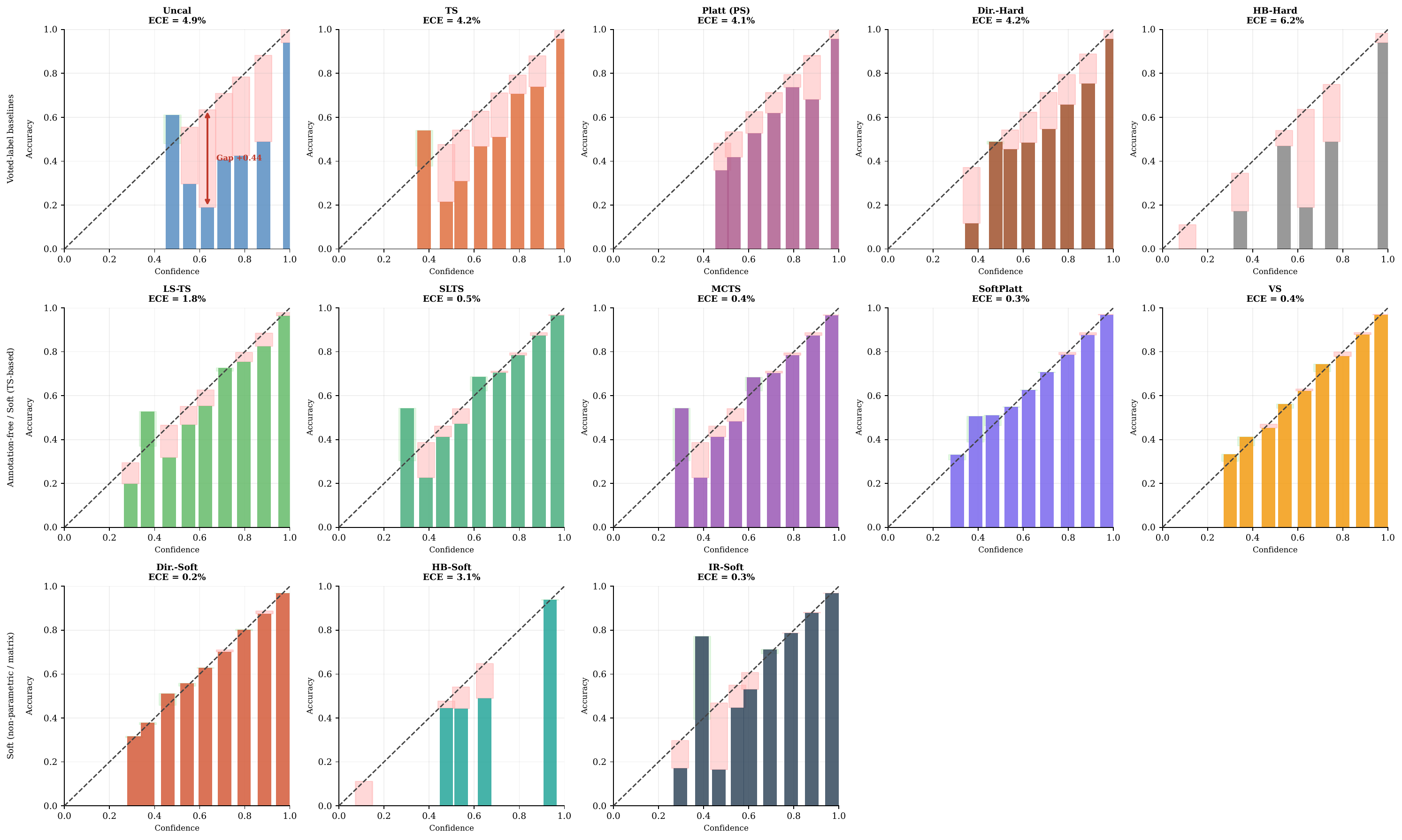}
  \caption{\textbf{Reliability diagrams (all methods): CIFAR-10H ViT-B/16 (ECE$_\text{true}$).}}
  \label{fig:reliability-full-cifar-vit}
\end{figure*}

\begin{figure*}[t]
  \centering
  \includegraphics[width=\linewidth]{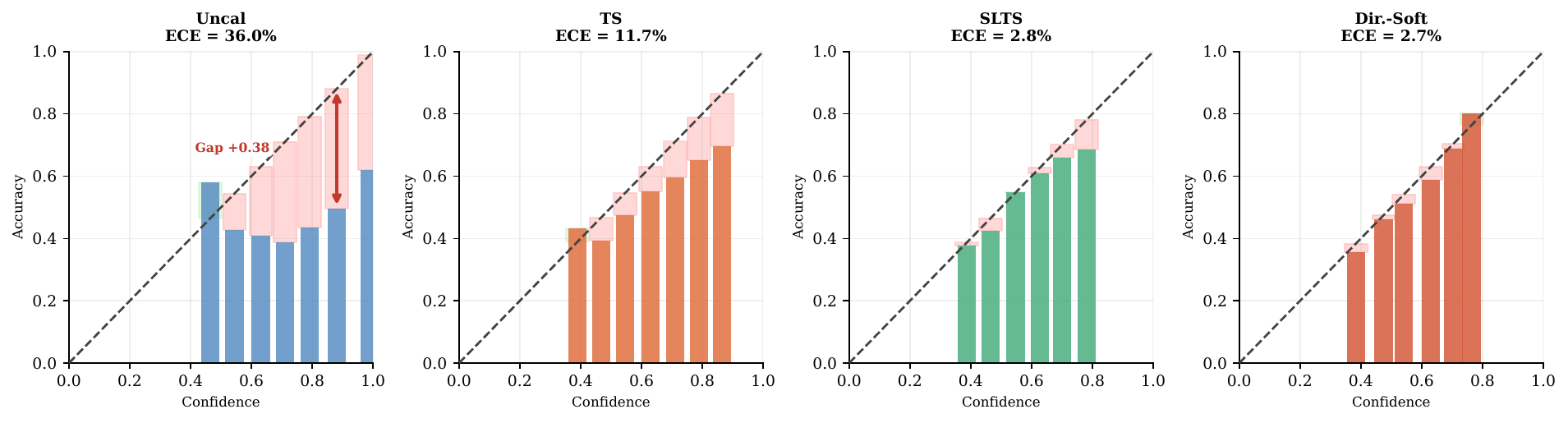}
  \caption{\textbf{Reliability diagrams (summary): ChaosNLI DeBERTa-v3-base.}}
  \label{fig:reliability-chaosnli-deberta-sum}
\end{figure*}
\begin{figure*}[t]
  \centering
  \includegraphics[width=\linewidth]{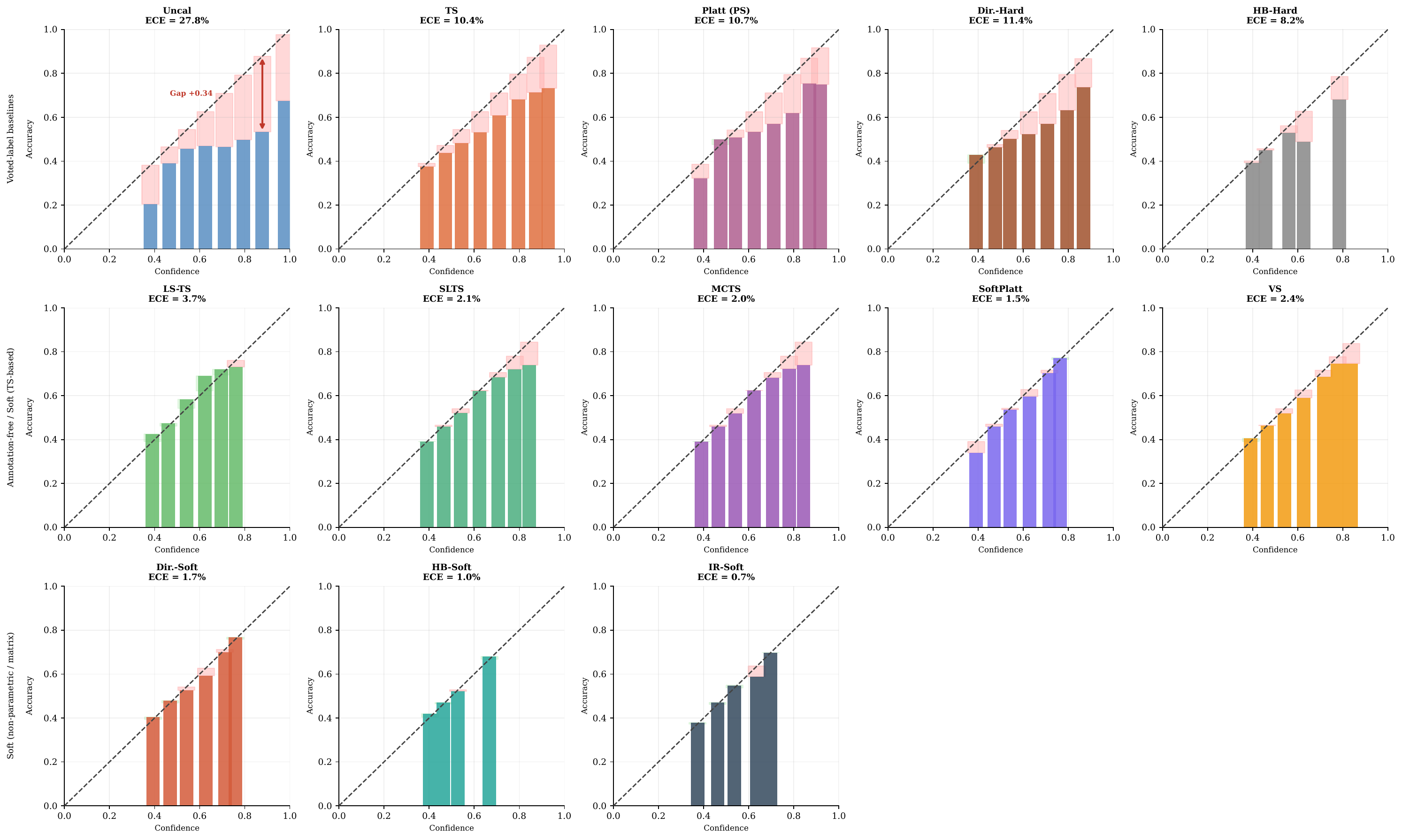}
  \caption{\textbf{Reliability diagrams (all methods): ChaosNLI RoBERTa-Large (ECE$_\text{true}$).}
    NLI models are severely overconfident before calibration; ambiguity-aware methods substantially correct this.}
  \label{fig:reliability-full-chaosnli}
\end{figure*}
\begin{figure*}[t]
  \centering
  \includegraphics[width=\linewidth]{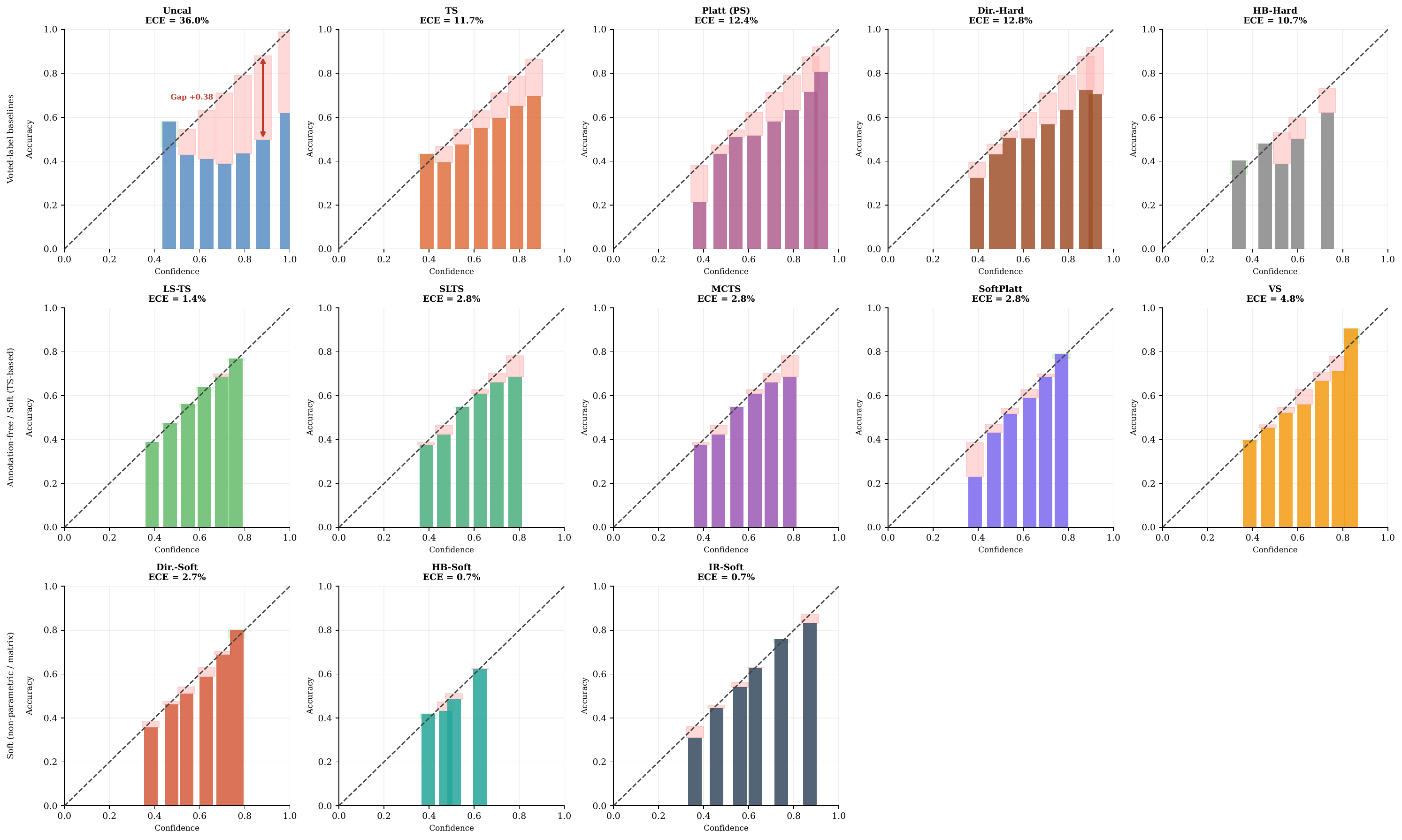}
  \caption{\textbf{Reliability diagrams (all methods): ChaosNLI DeBERTa-v3-base (ECE$_\text{true}$).}}
  \label{fig:reliability-full-chaosnli-deberta}
\end{figure*}

\begin{figure*}[t]
  \centering
  \includegraphics[width=\linewidth]{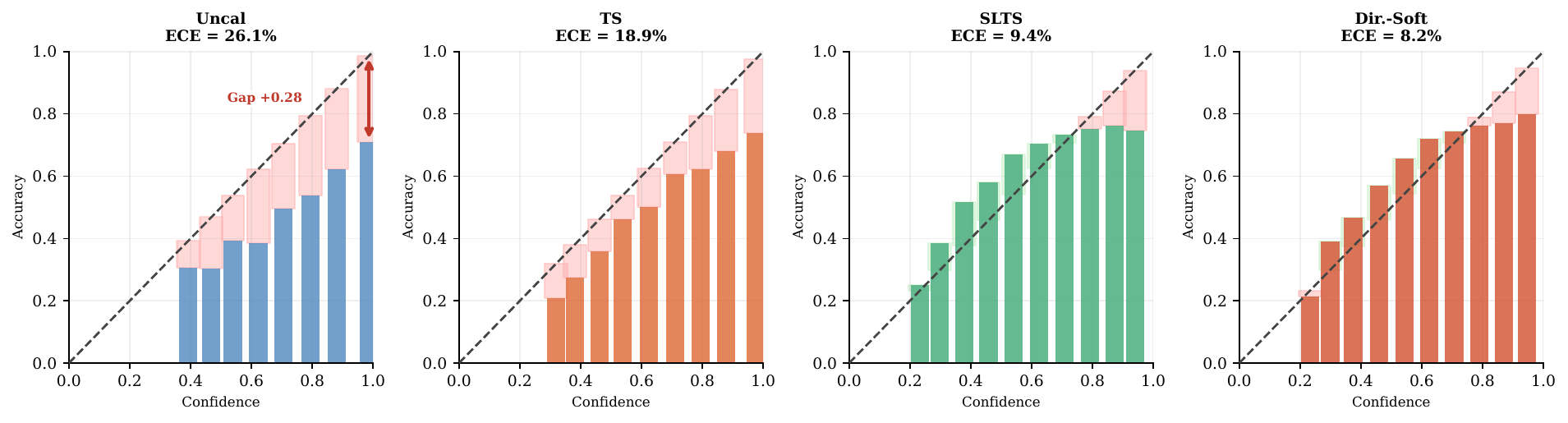}
  \caption{\textbf{Reliability diagrams (summary): ISIC 2019 EfficientNet-B4.}}
  \label{fig:reliability-isic-enet-sum}
\end{figure*}
\begin{figure*}[t]
  \centering
  \includegraphics[width=\linewidth]{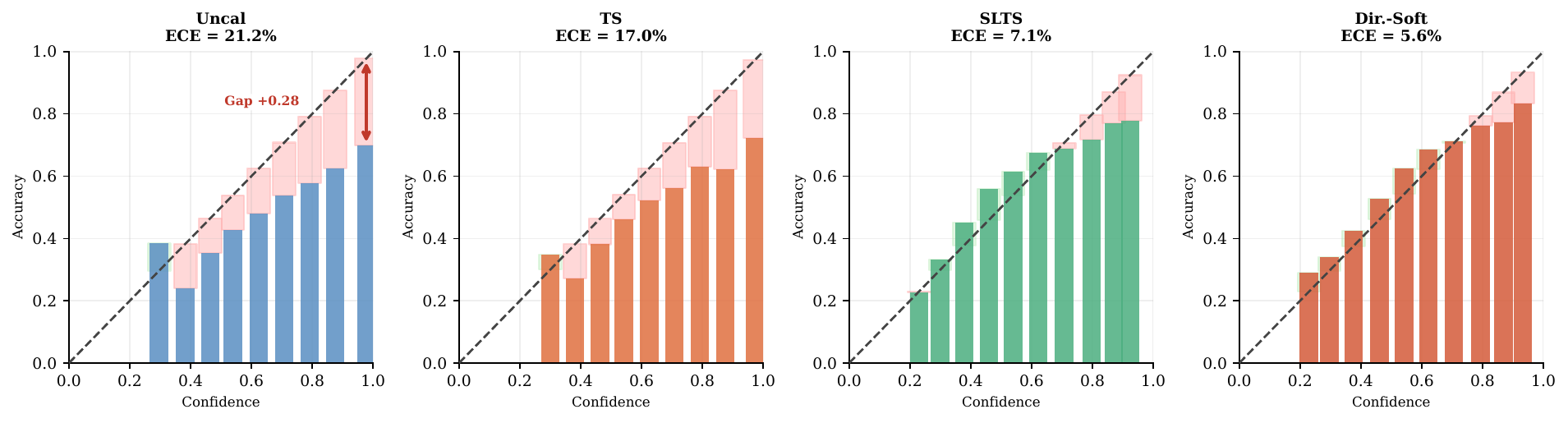}
  \caption{\textbf{Reliability diagrams (summary): ISIC 2019 ViT-S/16.}}
  \label{fig:reliability-isic-vit-sum}
\end{figure*}
\begin{figure*}[t]
  \centering
  \includegraphics[width=\linewidth]{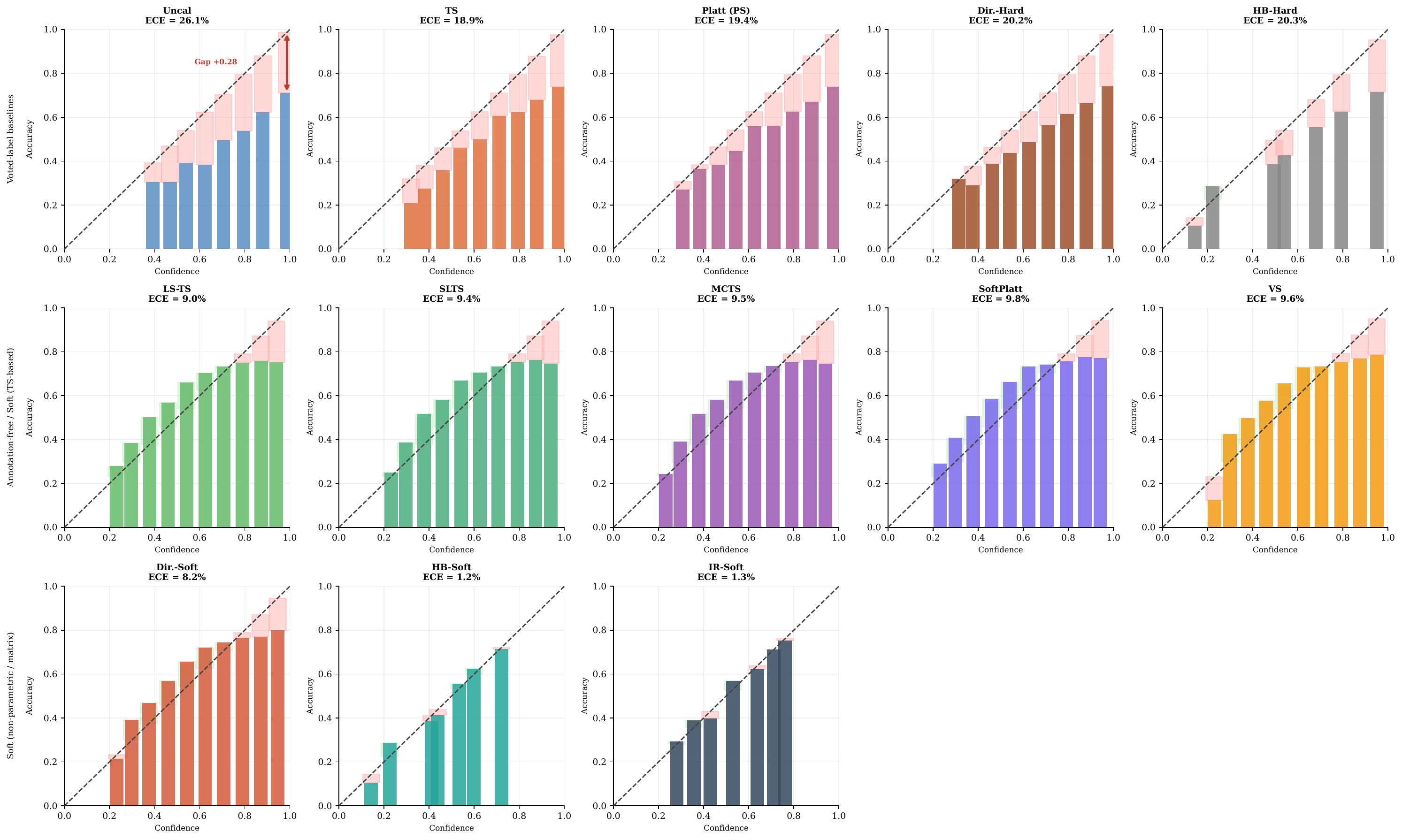}
  \caption{\textbf{Reliability diagrams (all methods): ISIC 2019 EfficientNet-B4 (ECE$_\text{true}$).}}
  \label{fig:reliability-full-isic-enet}
\end{figure*}
\begin{figure*}[t]
  \centering
  \includegraphics[width=\linewidth]{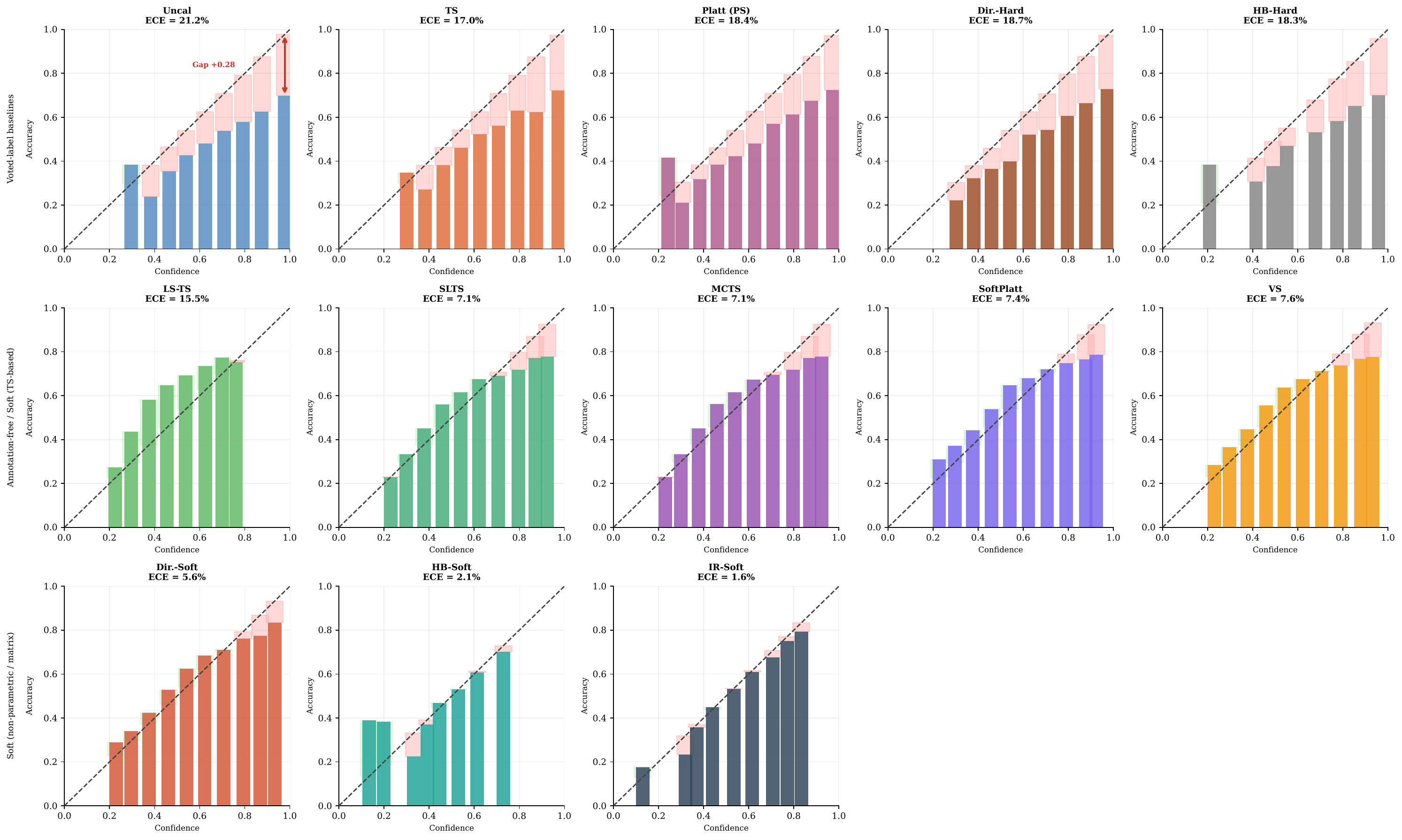}
  \caption{\textbf{Reliability diagrams (all methods): ISIC 2019 ViT-S/16 (ECE$_\text{true}$).}}
  \label{fig:reliability-full-isic-vit}
\end{figure*}

\begin{figure*}[t]
  \centering
  \includegraphics[width=\linewidth]{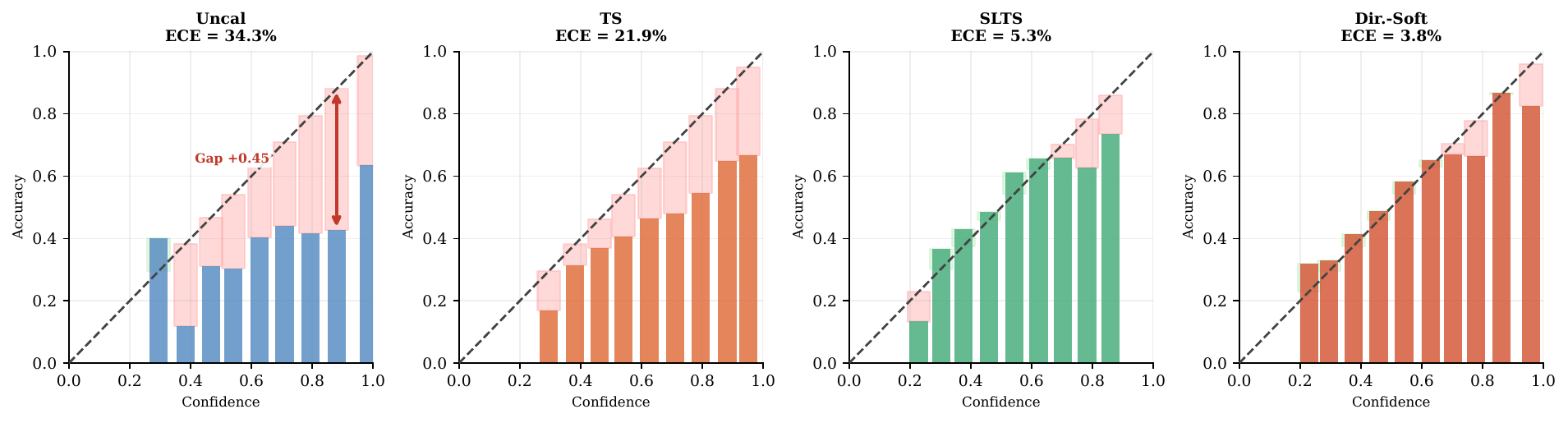}
  \caption{\textbf{Reliability diagrams (summary): DermaMNIST ResNet-18.}}
  \label{fig:reliability-derm-r18-sum}
\end{figure*}
\begin{figure*}[t]
  \centering
  \includegraphics[width=\linewidth]{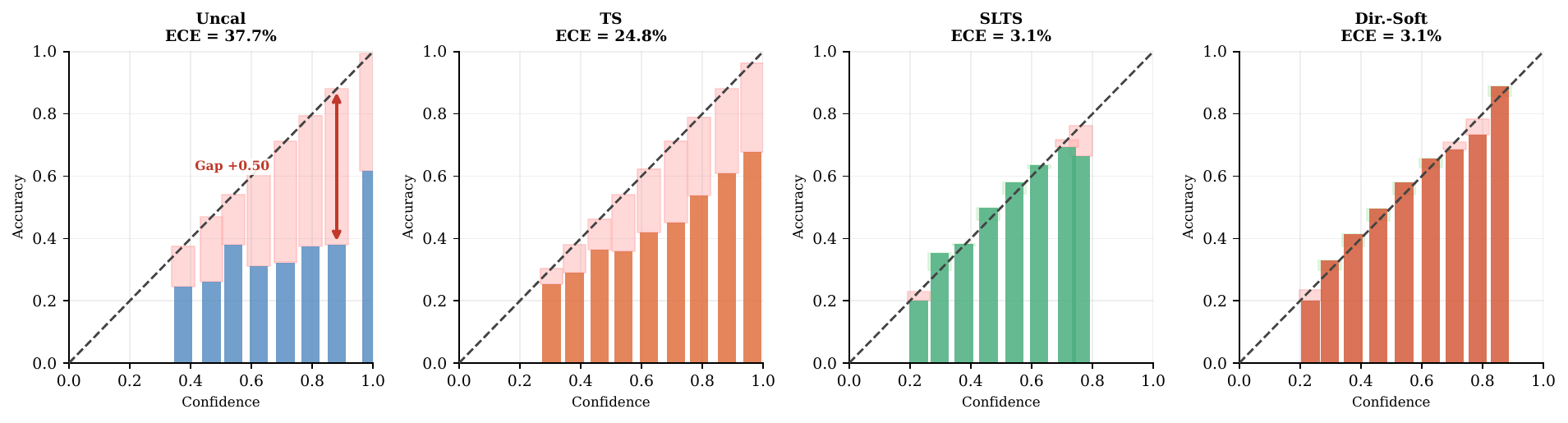}
  \caption{\textbf{Reliability diagrams (summary): DermaMNIST ViT-S/16.}}
  \label{fig:reliability-derm-vit-sum}
\end{figure*}
\begin{figure*}[t]
  \centering
  \includegraphics[width=\linewidth]{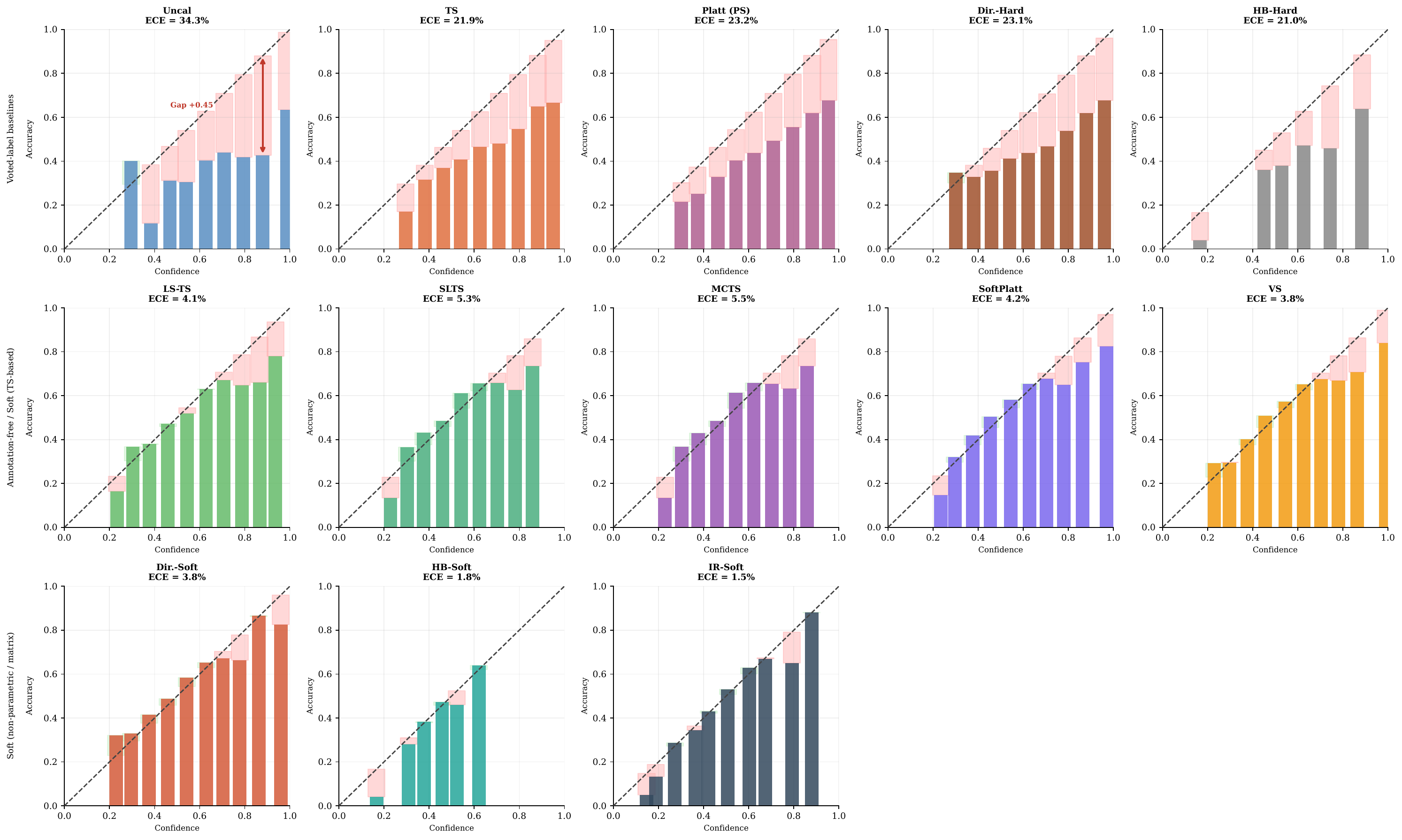}
  \caption{\textbf{Reliability diagrams (all methods): DermaMNIST ResNet-18 (ECE$_\text{true}$).}}
  \label{fig:reliability-full-derm-r18}
\end{figure*}
\begin{figure*}[t]
  \centering
  \includegraphics[width=\linewidth]{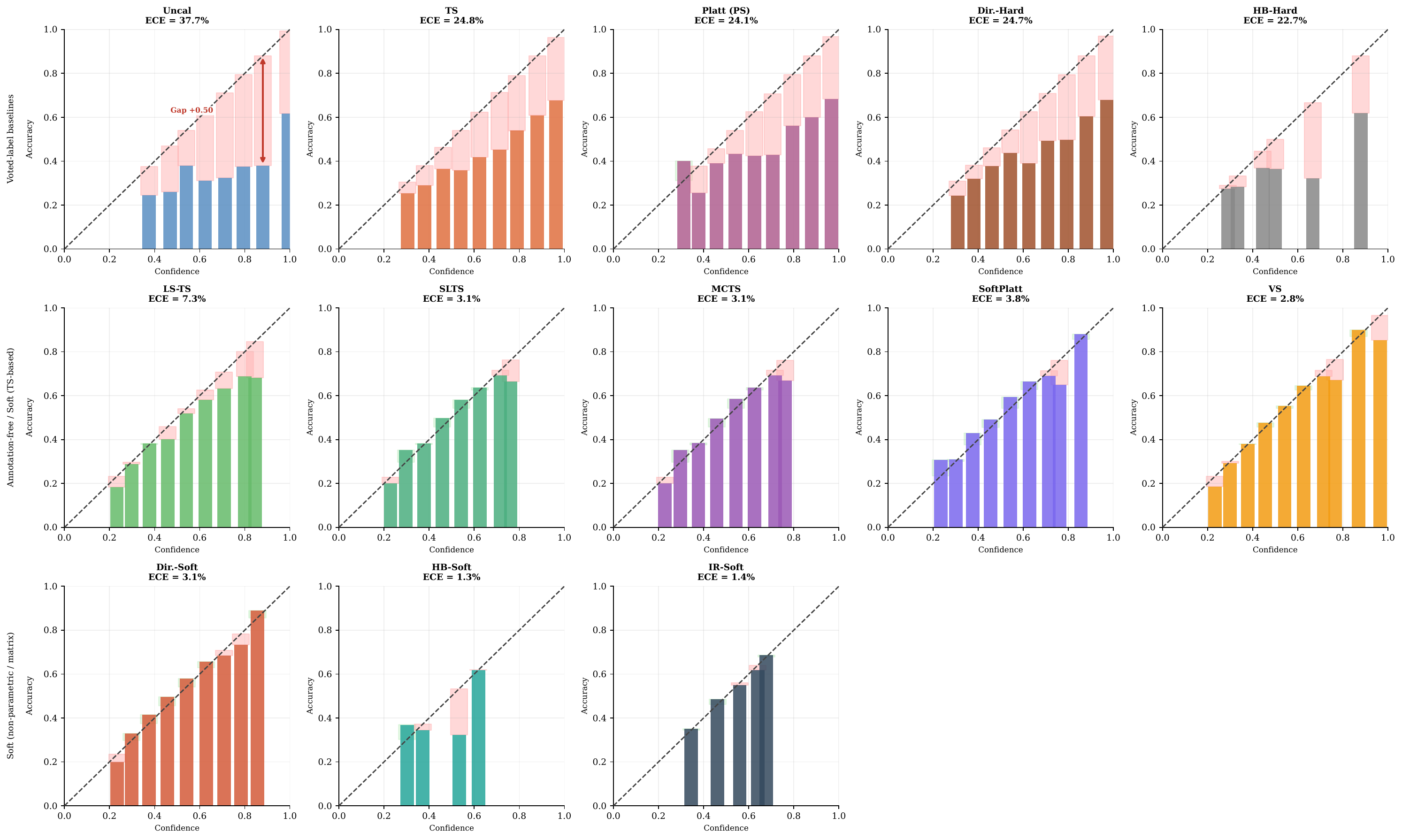}
  \caption{\textbf{Reliability diagrams (all methods): DermaMNIST ViT-S/16 (ECE$_\text{true}$).}}
  \label{fig:reliability-full-derm-vit}
\end{figure*}

\bigskip
The diagrams confirm the quantitative results in Tables~\ref{tab:cifar}--\ref{tab:derm}: TS reduces overconfidence relative to \emph{Uncal} but leaves a large residual gap because it targets voted labels.  Dirichlet-Hard performs similarly to TS and sometimes worse, confirming that the problem lies in the target, not model capacity.  All soft methods substantially close the gap, with IR-Soft and Dir.-Soft achieving the most uniform residuals across all benchmarks.

\section{ISIC 2019 Annotator Confusion Matrix}
\label{app:isic-confusion}

\subsection{Construction}

Because the full per-image reader annotations of Liu et al.~\cite{liu2020deep} are not publicly
available, we construct a synthetic annotator model using a clinically calibrated
$8{\times}8$ confusion matrix $C$, where entry $C_{ij}$ gives the probability that a
board-certified dermatologist labels an image as class~$j$ when the consensus
(majority-vote) diagnosis is class~$i$.  This confusion-matrix annotator model follows
the framework of Dawid and Skene~\cite{dawid1979maximum}, in which each annotator's behaviour is
characterised by a class-conditional label-flipping distribution; it is widely used
in crowdsourcing and multi-annotator learning \cite{rodrigues2018deep,uma2021learning}.

\noindent\textbf{Calibration sources.}
Diagonal entries (per-condition agreement rates) are set to match the per-condition
inter-reader agreement reported in Liu et al.~\cite{liu2020deep} and corroborated by
Haenssle et al.~\cite{haenssle2018man}.  Off-diagonal entries encode clinically established
confusion patterns:
\begin{itemize}[noitemsep, topsep=2pt]
  \item \textbf{MEL/NV} (Melanoma / Melanocytic Nevi): the most consequential and
        most confused pair in dermoscopy, with diagonal rates of 73\%/76\%.
        Off-diagonal entries $C_{\text{MEL},\text{NV}}=0.14$ and
        $C_{\text{NV},\text{MEL}}=0.15$ reflect the well-documented difficulty of
        distinguishing early melanoma from benign nevi \cite{haenssle2018man}.
  \item \textbf{AK/BKL/SCC} (Actinic Keratosis / Benign Keratosis-like Lesions /
        Squamous Cell Carcinoma): a high-confusion cluster with diagonal rates of
        65\%/62\%/67\%.  The AK$\leftrightarrow$SCC confusion ($C_{\text{AK},\text{SCC}}=0.16$,
        $C_{\text{SCC},\text{AK}}=0.18$) is clinically significant because
        untreated AK can progress to SCC.  BKL shares morphological features
        with both conditions.
  \item \textbf{DF/VL} (Dermatofibroma / Vascular Lesion): the easiest to
        distinguish, with diagonal rates of 87\%/91\%.
\end{itemize}
The mean diagonal agreement is
$\bar{C}_{ii} = 75\%$, matching the aggregate inter-reader agreement reported
by Liu et al.~\cite{liu2020deep} exactly.

\noindent\textbf{Simulation procedure.}
For each test image $x_i$ with consensus label $y_i$, we independently draw $m=9$
synthetic annotations from the categorical distribution defined by row $y_i$ of
$C$:
\[
  a^{(k)}_i \sim \mathrm{Categorical}\!\left(C_{y_i,\cdot}\right),
  \quad k = 1,\ldots,m,
\]
and set the empirical label distribution to
$\pi(x_i) = \frac{1}{m}\sum_{k=1}^m \mathbf{e}_{a^{(k)}_i}$.
This directly mirrors the $m=9$ board-certified dermatologist annotation protocol of
Liu et al.~\cite{liu2020deep}, whose reader study we cannot fully replicate due to data
unavailability.  The resulting mean annotation entropy over the ISIC 2019 test set
is $\bar{H}=0.66$, reflecting the genuine diagnostic difficulty of the 8-class task.

\begin{remark}[Class-conditional limitation]
This synthetic annotator model is deliberately class-conditional: every image with the same consensus class $y_i$ is assigned the same row $C_{y_i,\cdot}$ of the confusion matrix, regardless of its individual visual content.  In practice, images within the same consensus class will differ in visual ambiguity, leading to instance-level variation in annotator disagreement that this model does not capture.  Our experiments therefore evaluate calibration under a controlled form of clinically-informed ambiguity rather than true per-instance annotator distributions.
\end{remark}

\subsection{Visualisation}

Figure~\ref{fig:isic-confusion-app} shows the confusion matrix as a heatmap.
The MEL/NV cluster (orange dashed box) and the AK/BKL/SCC high-confusion
cluster (purple dotted highlights) are visually prominent.  DF and VL have
near-perfect diagonal entries and negligible off-diagonal mass.

\begin{figure*}[t]
  \centering
  \includegraphics[width=0.82\linewidth]{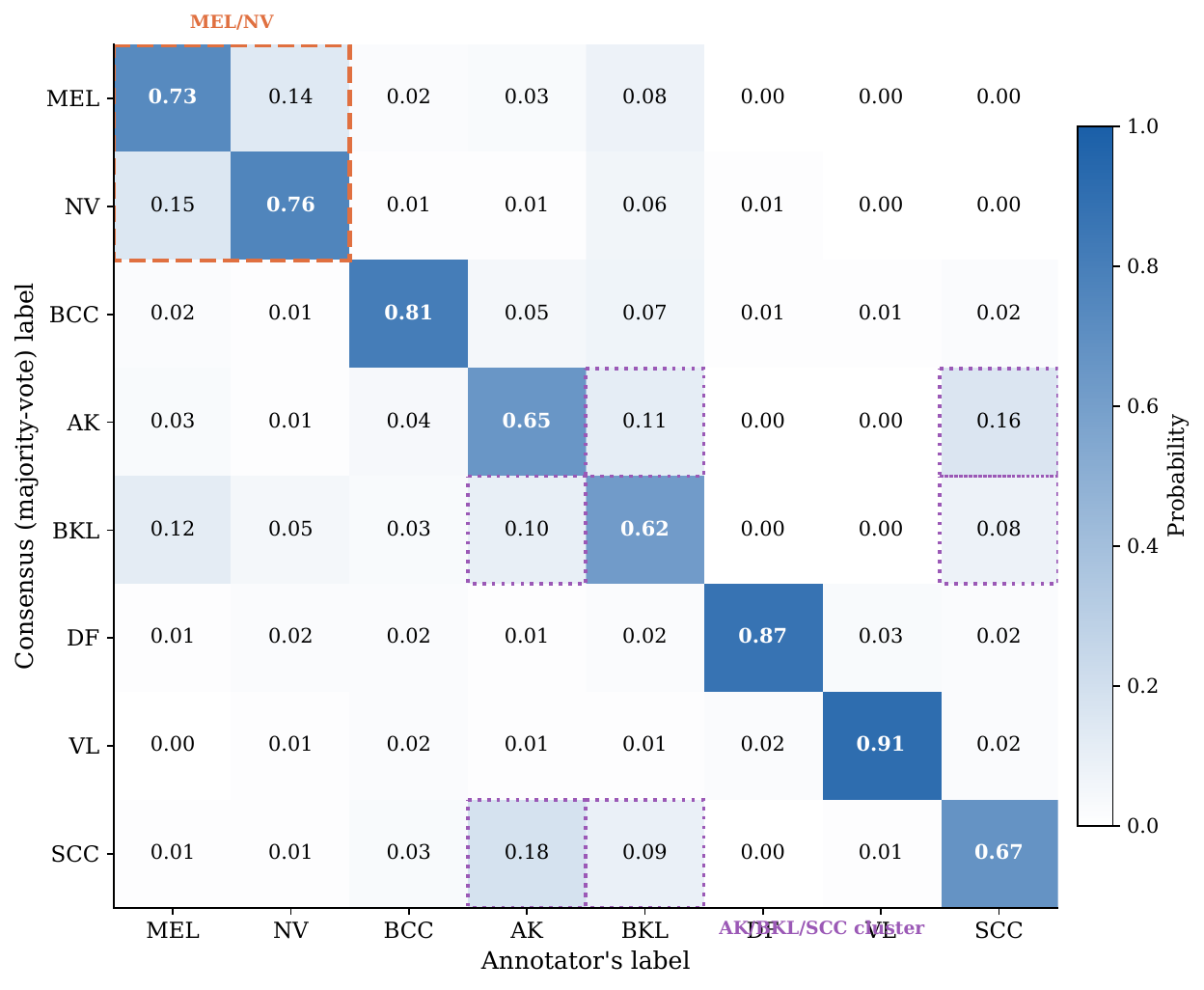}
  \caption{
    \textbf{Inter-reader confusion matrix for ISIC 2019.}
    Entry $C_{ij}$ is the probability that a dermatologist labels an image as
    class~$j$ given the consensus label is~$i$.
    Diagonal entries (bold) are the per-condition agreement rates, calibrated to
    match Liu et al.~\cite{liu2020deep} Table~2.
    Orange dashed box: MEL/NV most-confused pair.
    Purple dotted highlights: AK/BKL/SCC high-confusion cluster.
    Mean diagonal agreement $\bar{C}_{ii}=75\%$.
  }
  \label{fig:isic-confusion-app}
\end{figure*}

\subsection{Full confusion matrix}

Table~\ref{tab:isic-confusion} lists the exact numerical entries of the confusion matrix.

\begin{table}[t]
  \centering
  \caption{
    \textbf{Numerical values of the inter-reader confusion matrix for ISIC 2019.}
    Row = consensus (majority-vote) label; column = annotator's label.
    Each row sums to 1.  Diagonal entries (bold) are per-condition agreement rates,
    calibrated to Liu et al.~\cite{liu2020deep} Table~2; mean diagonal agreement $\bar{C}_{ii}=75\%$.
  }
  \label{tab:isic-confusion}
  \small
  \setlength{\tabcolsep}{4pt}
  \begin{tabular}{lcccccccc}
    \toprule
    & MEL & NV & BCC & AK & BKL & DF & VL & SCC \\
    \midrule
    MEL  & \textbf{0.73} & 0.14 & 0.02 & 0.03 & 0.08 & 0.00 & 0.00 & 0.00 \\
    NV   & 0.15 & \textbf{0.76} & 0.01 & 0.01 & 0.06 & 0.01 & 0.00 & 0.00 \\
    BCC  & 0.02 & 0.01 & \textbf{0.81} & 0.05 & 0.07 & 0.01 & 0.01 & 0.02 \\
    AK   & 0.03 & 0.01 & 0.04 & \textbf{0.65} & 0.11 & 0.00 & 0.00 & 0.16 \\
    BKL  & 0.12 & 0.05 & 0.03 & 0.10 & \textbf{0.62} & 0.00 & 0.00 & 0.08 \\
    DF   & 0.01 & 0.02 & 0.02 & 0.01 & 0.02 & \textbf{0.87} & 0.03 & 0.02 \\
    VL   & 0.00 & 0.01 & 0.02 & 0.01 & 0.01 & 0.02 & \textbf{0.91} & 0.02 \\
    SCC  & 0.01 & 0.01 & 0.03 & 0.18 & 0.09 & 0.00 & 0.01 & \textbf{0.67} \\
    \bottomrule
  \end{tabular}
\end{table}

\section{Empirical Validation of Proposition~\ref{prop:entropy-gap}}
\label{app:entropy-validation}

Proposition~\ref{prop:entropy-gap} states that, under simplifying assumptions, the pointwise true-label calibration error of a voted-label-calibrated model is non-decreasing in the annotation entropy $H(x)$.  Figure~\ref{fig:entropy-validation} provides empirical support: we bin test examples by their normalised annotation entropy $H(x)/\log K$ and plot the mean absolute error $|\hat{p}_{\hat{c}}(x) - \pi_{\hat{c}}(x)|$ for TS across three dataset--architecture combinations.

For CIFAR-10H (both architectures), TS error rises from ${\approx}2\%$ in the near-unambiguous bin ($H/\log K \approx 0.02$) to ${\approx}18\%$ in the high-ambiguity bin ($H/\log K \approx 0.48$), a clear monotone increase consistent with the proposition.  For ChaosNLI (DeBERTa-v3), which operates in a higher-entropy regime ($H/\log K \in [0.2, 0.9]$), TS error is uniformly high ($19$--$25\%$) and rises further at the highest-entropy bins.  The overall trend is consistent across both domains and all three architectures: voted-label calibration error increases with annotation entropy, as predicted by Proposition~\ref{prop:entropy-gap}.

See Figure~\ref{fig:entropy-validation} in the main text (Theory section) for the plot.

\section{LS-TS Ablation: Smoothing Strategy Comparison (Extended)}
\label{app:lsts-ablation}

We compare LS-TS against three simpler annotation-free smoothing baselines that share the same single-temperature family but differ in how the pseudo-target $\tpihat^{\mathrm{LS}}_i$ is constructed:

\begin{itemize}[noitemsep, topsep=2pt]
  \item \textbf{Fixed-LS} ($\varepsilon{=}0.1$): fixed data-independent smoothing weight, the standard label-smoothing value used in training \cite{szegedy2016rethinking}.
  \item \textbf{Ent-LS}: per-instance $\varepsilon_i = H(\hat{p}(x_i)) / \log K$, using normalised prediction entropy as a proxy for annotator disagreement.
  \item \textbf{CC-LS}: per-class $\varepsilon_k = \mathrm{mean}_{i:\,y^*_i=k}(1 - \hat{p}_{y^*}(x_i))$, accommodating class-level variation in model confidence (same idea as Vector Scaling applied to the smoothing weight).
\end{itemize}

All methods are evaluated on CIFAR-10H and ChaosNLI, the two benchmarks with real multi-annotator distributions, using the same experimental protocol as the main paper (seed 42, 50/50 stratified split).  Table~\ref{tab:lsts-ablation} (main text, Methods section) reports the full comparison.

Fixed-LS with $\varepsilon{=}0.1$ consistently \emph{under}performs or even \emph{degrades} below TS (e.g., ECE 7.44\% vs.\ 4.29\% on CIFAR-10H ResNet-50), confirming that an arbitrary fixed smoothing weight is insufficient; the data-driven estimate of $\bar\varepsilon$ is critical.
Ent-LS (per-instance entropy) improves over TS on most settings but is inconsistent: it works well on ChaosNLI RoBERTa-L (3.17\%) but collapses on DeBERTa-v3 (8.58\%), suggesting that prediction entropy is a noisy proxy for annotator ambiguity when models are already strongly confident.
CC-LS (per-class $\varepsilon$) achieves results very close to LS-TS on all four settings (within 0.1 pp ECE), confirming that the global mean is a good summary of per-class behaviour for these benchmarks; the slight advantage of CC-LS on DeBERTa-v3 (2.54\% vs.\ 2.65\%) is within the single-seed variance reported in Appendix~\ref{app:multiseed}.

\section{Adaptive Temperature Scaling with Voted Labels (Extended Analysis)}
\label{app:ats}

A natural question is whether making the calibration map \emph{adaptive} (predicting a per-instance temperature $T(x_i)$ rather than a global $T^*$) can close the gap between voted-label calibration and ambiguity-aware calibration, without requiring annotator data.  We implement \textbf{ATS} \cite{joy2023sample}: a linear layer maps four logit-derived features $\phi(z_i) = [\max_k z_{ik},\; H(\mathrm{softmax}(z_i))/\log K,\; p_{(1)}-p_{(2)},\; p_{(1)}]$ to $\log T_i$ via $T_i = \mathrm{softplus}(w^\top\phi(z_i)+b)+0.1$, trained with voted-label NLL identical to TS.  The four features capture prediction scale, uncertainty, margin, and confidence.  We apply $\ell_2$ regularisation ($\lambda=10^{-3}$) to avoid overfitting on small calibration sets.

\begin{table*}[t]
  \centering
  \small
  \caption{\textbf{ATS vs.\ TS and LS-TS across all 8 settings (annotation-free methods only).}
    ECE$_\text{true}$ (\%) and Brier.
    ATS predicts a per-instance temperature from logit features but is trained with the same voted-label NLL as TS.
    \textbf{Bold} = best annotation-free method per column.}
  \label{tab:ats-app}
  \setlength{\tabcolsep}{4pt}
  \resizebox{\linewidth}{!}{%
  \begin{tabular}{l rr rr rr rr rr rr rr rr}
    \toprule
    & \multicolumn{2}{c}{\shortstack{C10H\\R50}}
    & \multicolumn{2}{c}{\shortstack{C10H\\ViT}}
    & \multicolumn{2}{c}{\shortstack{NLI\\RoBERTa}}
    & \multicolumn{2}{c}{\shortstack{NLI\\DeBERTa}}
    & \multicolumn{2}{c}{\shortstack{ISIC\\ENet}}
    & \multicolumn{2}{c}{\shortstack{ISIC\\ViT}}
    & \multicolumn{2}{c}{\shortstack{Derm\\R18}}
    & \multicolumn{2}{c}{\shortstack{Derm\\ViT}} \\
    \cmidrule(lr){2-3}\cmidrule(lr){4-5}\cmidrule(lr){6-7}\cmidrule(lr){8-9}
    \cmidrule(lr){10-11}\cmidrule(lr){12-13}\cmidrule(lr){14-15}\cmidrule(lr){16-17}
    Method & ECE & Br & ECE & Br & ECE & Br & ECE & Br & ECE & Br & ECE & Br & ECE & Br & ECE & Br \\
    \midrule
    TS
      &  4.29 & .112 &  4.48 & .106 & 10.55 & .537 & 11.63 & .549
      & 18.72 & .559 & 17.04 & .630 & 22.25 & .687 & 24.65 & .683 \\
    ATS
      &  4.40 & .113 &  4.54 & .106 & 11.36 & .538 & 13.09 & .552
      & 18.83 & .560 & 17.53 & .631 & 22.82 & .687 & 25.35 & .686 \\
    \textbf{LS-TS}
      & \textbf{1.57} & \textbf{.109}
      & \textbf{2.37} & \textbf{.103}
      & \textbf{4.16} & \textbf{.522}
      & \textbf{2.65} & \textbf{.529}
      & \textbf{9.34} & \textbf{.528}
      & \textbf{15.49} & \textbf{.619}
      & \textbf{5.05} & \textbf{.630}
      & \textbf{7.36} & \textbf{.615} \\
    \bottomrule
  \end{tabular}}
\end{table*}

\noindent\textbf{Finding.}
ATS fails to improve over TS in all 8 settings and in several cases degrades it (ECE rises from $11.63\%$ to $13.09\%$ on DeBERTa-v3; from $17.04\%$ to $17.53\%$ on ISIC ViT-S/16).  The per-instance temperature learned from logit features adapts to \emph{model uncertainty}, not to \emph{annotator disagreement}; these two quantities are loosely correlated at best (cf.\ Figure~\ref{fig:entropy-validation}).  In contrast, LS-TS, which uses the same global-temperature architecture as TS but simply shifts the calibration target toward a smoothed distribution, reduces ECE by 9--77\% across all 8 settings without any annotator data.  This confirms that the calibration target is the critical design axis: added architectural flexibility without the right supervision signal yields no benefit.

\section{Implementation Details}
\label{app:impl}

\noindent\textbf{CIFAR-10H models.}
\textit{ResNet-50}: ImageNet-1k weights (\texttt{ResNet50\_Weights.IMAGENET1K\_V1}); FC layer replaced with $2048\times 10$ linear; AdamW, lr$=10^{-4}$, weight decay $10^{-4}$, batch 128, cosine annealing, 30 epochs, $224\times224$.
\textit{ViT-B/16}: ImageNet-21k weights; classification head replaced; AdamW, lr$=5\times10^{-5}$, weight decay $10^{-4}$, batch 64, cosine annealing, 20 epochs, $224\times224$.

\noindent\textbf{ChaosNLI models.}
\textit{RoBERTa-Large}: \texttt{roberta-large-mnli} from HuggingFace, pre-fine-tuned on MultiNLI \cite{liu2019roberta}; no additional fine-tuning.
\textit{DeBERTa-v3-base}: \texttt{cross-encoder/nli-deberta-v3-base} from HuggingFace \cite{he2021debertav3}; no additional fine-tuning.
Both models use the combined ChaosNLI SNLI+MNLI subset ($n=3{,}113$), split 50/50 into calibration and test (stratified, seed 42).

\noindent\textbf{ISIC 2019 models.}
\textit{EfficientNet-B4}: ImageNet-1k weights; stratified 70/15/15 split; class-weighted cross-entropy (NV: ${\approx}67\%$ of training images); AdamW, lr$=3\times10^{-4}$, weight decay $10^{-4}$, 2-epoch linear warm-up + cosine decay, 20 epochs, $380\times380$.
\textit{ViT-S/16}: ImageNet-21k weights; same split and loss weighting; AdamW, lr$=10^{-4}$, weight decay $10^{-4}$, 2-epoch warm-up + cosine decay, 20 epochs, $224\times224$.

\noindent\textbf{DermaMNIST models.}
\textit{ResNet-18}: ImageNet-1k weights, class-weighted cross-entropy, AdamW, lr$=10^{-4}$, 30 epochs.
\textit{ViT-S/16}: ImageNet-21k weights, same training protocol, $224\times224$.

\noindent\textbf{Calibration optimisers.}
TS, SLTS: LBFGS, lr$=0.1$, 500 iterations, tolerance $10^{-9}$ (gradient), $10^{-11}$ (loss).  Platt/VS: Adam, lr$=0.01$/0.05, weight decay $10^{-4}$, 2000 steps.

\noindent\textbf{ECE bins.}
$B=15$ equal-width bins in $[0,1]$; bins with fewer than 1 example are excluded.

\noindent\textbf{Compute.}
CIFAR-10H fine-tuning: ${\approx}20$ min on one NVIDIA A100.  All calibration methods: $<60$ s on CPU after logit caching.

\end{document}